\documentclass[twoside,11pt]{article}

\usepackage{jmlr2e}

\usepackage{graphicx}

\usepackage{amsmath}
\usepackage{apalike}
\usepackage{color}
\usepackage{subfigure}
\usepackage{multirow}
\usepackage{enumitem}
\usepackage{graphicx} 
\usepackage{subfigure}
\usepackage{algorithm}
\usepackage{algorithmic}
\usepackage{nicefrac}
\usepackage{float}
\usepackage{booktabs}
\usepackage{times}
\usepackage[normalem]{ulem}

\usepackage{url}
\usepackage{bm}

\newtheorem{assumption}{Assumption}

\usepackage{array}
\newcolumntype{L}[1]{>{\raggedright\let\newline\\\arraybackslash\hspace{0pt}}m{#1}}
\newcolumntype{C}[1]{>{\centering\let\newline  \\\arraybackslash\hspace{0pt}}m{#1}}
\newcolumntype{R}[1]{>{\raggedleft\let\newline \\\arraybackslash\hspace{0pt}}m{#1}}

\newcommand{\NM}[2]{\left\| #1 \right\|_{#2}}
\newcommand{\R}{\mathbb{R}}
\newcommand{\Px}[2]{\text{\normalfont prox}_{#1}(  #2 ) }

\newcommand{\Diag}[1]{ \text{\normalfont \,Diag} \left(  #1 \right) }

\newcommand{\Tr}[1]{\text{Tr}( #1 ) }
\newcommand{\ip}[1]{{\left\langle  #1 \right\rangle}}

\usepackage[mathscr]{euscript}
\newcommand{\ten}[1]{ \mathscr{ #1 } }
\newcommand{\SO}[1]{P_{\vect{\Omega}}\left( #1 \right) }

\newcommand{\vect}[1]{{\boldsymbol{\bm{#1}}}} 

\usepackage{lastpage}
\jmlrheading{23}{2022}{1-\pageref{LastPage}}{12/20; Revised
	2/22}{4/22}{20-1368}{Quanming Yao, Yaqing Wang, Bo Han, James T. Kwok}
\ShortHeadings{Low-rank Tensor Learning with Nonconvex Nuclear Norm Regularization}{Yao, Wang, Han, Kwok}

\begin{document}

\title{Low-rank Tensor Learning with \\
Nonconvex 
Overlapped Nuclear Norm
Regularization}

\author{\name Quanming Yao \email qyaoaa@tsinghua.edu.cn \\
		\addr 
		Department of Electronic Engineering, Tsinghua University
		\AND
		\name Yaqing Wang~\thanks{Corresponding Author.} \email wangyaqing01@baidu.com \\
		\addr Baidu Research
		\AND
		\name Bo Han \email bhanml@comp.hkbu.edu.hk \\
		\addr Department of Computer Science, Hong Kong Baptist University
		\AND   	   
		James T. Kwok \email jamesk@cse.ust.hk \\
		\addr Department of Computer Science and Engineering, 
		Hong Kong University of Science and Technology}

\editor{Prateek Jain}

\maketitle

\begin{abstract}
Nonconvex regularization has been popularly used in low-rank matrix learning.
However, extending  it for low-rank tensor learning
is still computationally expensive.
To address this problem,
we develop an efficient solver for use with a nonconvex extension of the
overlapped nuclear norm regularizer.
Based on the proximal average algorithm,
the proposed algorithm can avoid expensive 
tensor folding/unfolding operations. 
A special ``sparse plus low-rank" structure
is maintained throughout the iterations, and
allows fast computation of the individual proximal steps.
Empirical convergence is further improved
with the use of adaptive momentum.
We  provide convergence guarantees to critical points on 
smooth
losses 
	and also on objectives satisfying
	the Kurdyka-{\L}ojasiewicz condition.
While the optimization problem is nonconvex and nonsmooth,
we show that
its critical points still have good statistical performance
on the tensor completion problem.  
Experiments on various synthetic and real-world data sets show that the proposed algorithm is efficient
in both time and space and more accurate than the existing state-of-the-art.
\end{abstract}


\section{Introduction}
\label{sec:intro}

Tensors can be seen as high-order matrices and are widely used for describing multilinear relationships in the data.
They have been popularly applied in areas such as computer vision, 
data mining and machine learning~\citep{Kolda2009,zhao2016tensor,song2017tensor,papalexakis2017tensors,hong2020generalized,janzamin2020spectral}.
For example,
color images~\citep{liu2013tensor}, 
hyper-spectral images~\citep{signoretto2011tensor,He2019NonLocalMG},
and knowledge graphs~\citep{nickel2015review,lacroix2018canonical}
can be naturally represented as third-order tensors, while
color videos can be seen as 4-order tensors~\citep{candes2011robust,bengua2017efficient}.
In YouTube,
users can follow each other and belong to the same subscribed channels.
By treating channel as the third dimension,
the users' co-subscription network can also be represented as a third-order tensor~\citep{lei2009analysis}.

In many applications, only a few entries  in the tensor are observed.
For example,
each YouTube user usually only interacts
with a few other users
\citep{lei2009analysis,davis2011multi},
and in knowledge graphs,
we can only have a few labeled edges
describing the relations between entities~\citep{nickel2015review,lacroix2018canonical}.
Tensor completion, which
aims at filling in this partially observed tensor,
has attracted a lot of recent interest
\citep{rendle2010pairwise,signoretto2011tensor,bahadori2014fast,cichocki2015tensor}.

In the related task of matrix completion, 
the underlying matrix is often assumed to be low-rank~\citep{candes2009exact},
as
its rows/columns share similar
characteristics.  The nuclear norm, which is the tightest convex envelope of the
matrix rank~\citep{boyd2009convex}, is popularly used as its surrogate 
in low-rank matrix completion~\citep{cai2010singular,mazumder2010spectral}.
In tensor completion, the low-rank assumption also captures relatedness in the 
different tensor dimensions~\citep{tomioka2010estimation,acar2011scalable,song2017tensor,hong2020generalized}.
However, tensors are more complicated
than matrices. Indeed,
even the computation of tensor rank is NP-hard~\citep{hillar-13}.
In recent years, 
many convex relaxations based on the matrix nuclear norm 
have been proposed
for tensors.
Examples include the tensor trace norm~\citep{cheng2016scalable},
overlapped nuclear norm~\citep{tomioka2010estimation,gandy2011tensor},
and latent nuclear norm~\citep{tomioka2010estimation}.
Among these convex relaxations, the overlapped nuclear norm is the most
popular as it (i) can be evaluated exactly by performing SVD on the unfolded matrices~\citep{cheng2016scalable}, 
(ii) has better low-rank approximation~\citep{tomioka2010estimation}, 
and 
(iii) can lead to exact recovery~\citep{tomioka2011statistical,tomioka2013convex,mu2014square}.

The (overlapped) nuclear norm
equally penalizes all singular values.
Intuitively,
larger singular values are more informative and should be less
penalized~\citep{mazumder2010spectral,lu2016nonconvex,yao2018large}.
To alleviate this problem
in low-rank matrix learning,
various adaptive nonconvex regularizers have been recently introduced.
Examples include
the capped-$\ell_1$ norm~\citep{zhang2010analysis},
log-sum-penalty (LSP)~\citep{candes2008enhancing},
truncated nuclear norm (TNN)~\citep{hu2013fast},
smoothed-capped-absolute-deviation (SCAD)~\citep{fan2001variable}
and
minimax concave penalty (MCP)~\citep{zhang2010nearly}.
All these assign smaller penalties to the larger singular values.
This leads to better recovery performance in many applications 
such as image recovery~\citep{lu2016nonconvex,gu2017weighted} and collaborative filtering~\citep{yao2018large},
and lower statistical errors of various 
matrix completion and regression problems~\citep{gui2016towards,Mazumder2020MatrixCW}.

Motivated by 
the success of adaptive nonconvex regularizers in low-rank matrix learning,
there are 
recent
works that apply nonconvex regularization in learning low-rank tensors.
For example, the TNN regularizer is used with 
the overlapped nuclear norm
regularizer
on 
video processing~\citep{Xue2018LowRankTC} and traffic data processing~\citep{Chen2020ANL}.
In this paper,
we propose a general nonconvex variant of the overlapped nuclear norm
regularizer for low-rank tensor completion.
Unlike the standard convex tensor completion problem,
the resulting optimization problem is nonconvex and more difficult to solve.
Previous algorithms in
\citep{Xue2018LowRankTC,Chen2020ANL}
are computationally expensive, and
have neither convergence results nor statistical guarantees. 

To solve this issue,
based on the proximal average algorithm~\citep{bauschke2008proximal},
we 
develop 
an efficient solver 
with much smaller time and space complexities.
The keys to  its success
are on (i) avoiding expensive tensor folding/unfolding,
(ii) maintaining a ``sparse plus low-rank'' structure on the iterates, and
(iii) incorporating 
the adaptive momentum~\citep{li2015accelerated,Li2017ada,yao2017efficient}.
Convergence guarantees to critical points
are provided under the usual smoothness assumption for the loss
and further Kurdyka-{\L}ojasiewicz~\citep{attouch2013convergence} condition
on the whole learning objective.

Besides, we study its
statistical guarantees,
and show
that critical points of the proposed objective can have small statistical errors
under the restricted strong convexity condition~\citep{agarwal2010fast}.
Informally,
for tensor completion with
unknown
noise,
we show that the recovery error can be bounded as
$\| \ten{X}^* - \tilde{\ten{X}} \|_F \le \mathcal{O} ( \lambda \kappa_0 \sum\nolimits_{i = 1}^M \sqrt{k_i} )$
(see Theorem~\ref{thm:stat}),
where $\mathcal{O}$ omits constant terms,
$\ten{X}^*$ (resp. $\tilde{\ten{X}}$) is the ground-truth (resp. recovered) tensor,
$M$ is the tensor order and $k_i$ is the rank of unfolding matrix on the $i$th mode.
When Gaussian additive noise is assumed,
we show that the recovery error also depends 
linearly with the noise level $\sigma$ (see Corollary~\ref{cor:noisy})
and $\sqrt{\log I^{\pi} / \NM{\bf{\Omega}}{1}}$
where  $I^{\pi}$
is the tensor size
and $\NM{\bf{\Omega}}{1}$ is the number of observed elements 
(see Corollary~\ref{cor:number}).

We further extend it for use with
Laplacian regularizer as in spatial-temporal analysis
and 
non-smooth losses
as in robust tensor completion.
Experiments on a variety of synthetic and real-world data sets
(including images, videos, hyper-spectral images, social networks, knowledge graphs and spatial-temporal climate observation records)
show that the proposed algorithm is more efficient
and has much better empirical performance than other low-rank tensor regularization and decomposition methods.

\subsection*{Difference with the Conference Version}
A preliminary conference version of this work 
\citep{yao2019efficient}
was published in ICML-2019.
The main differences 
with this conference version
are as follows.
\begin{itemize}[leftmargin=18px,noitemsep,parsep=0pt,partopsep=0pt]
\item[1).] Only third-order tensor and square loss are considered in 
\citep{yao2019efficient}, while
the proposed algorithm here,
which is enabled by Proposition~\ref{pr:mulv}, 
can work on tensors with arbitrary orders.
The difficulties of extending to higher order tensors
are also discussed after Proposition~\ref{pr:mulv}.

\item[2).] Statistical guarantee of the proposed model 
for the tensor completion problem
is added in Section~\ref{sec:stat},
which shows that
tensors that are not too spiky can be recovered.
We also show
how the recovery performance can depend on
noise level,
tensor ranks,
and the number of observations.

\item[3).] In Section~\ref{sec:extend}, we enable the proposed method work with robust loss function (which is nonconvex and nonsmooth) and Laplacian regularizer. These enable the proposed algorithm to be applied to a wider range of application scenarios such as
knowledge graph completion, 
spatial-temporal analysis 
and robust video recovery.

\item[4).]
Extensive experiments are added.  
Specifically,
quality of the obtained critical points is studied in Section~\ref{sec:quacri};
application to knowledge graphs 
in Section~\ref{sec:exp:kg},
application to robust video completion in Section~\ref{sec:exp:rpca},
and 
application to spatial-temporal data
in Section~\ref{sec:exp:sptemp}.
 
\end{itemize}

\subsection*{Notation}
Vectors
are denoted by lowercase boldface,
matrices
by uppercase boldface,
and
tensors
by Euler.

For a matrix $\vect{A}  \in  \R^{m  \times  n}$
(without loss of generality, we assume that $m  \ge  n$),
$\sigma_i(\vect{A})$ denotes its $i$th singular, 
its nuclear norm is $\| \vect{A} \|_*  = 
\sum_i  \sigma_i$;
$\| \vect{A} \|_{\infty}$ returns its maximum singular.

For tensors,
we follow the
notation in~\citep{Kolda2009}.
For a  $M$-order tensor
$\ten{X}  \in  \R^{I^1 \times \dots \times I^M}$
(without loss of generality, we assume $ I^1 \ge \dots \ge I^M$),
its $(i_1, \dots , i_M)$th entry
is $\ten{X}_{i_1 \dots i_M}$.
One can {\em unfold}
$\ten{X}$
along its $d$th mode to
obtain the matrix $\vect{X}_{\left\langle  d \right\rangle}  \in  \R^{I^d \times ( \frac{ I^{\pi} }{ I^d } )}$
with $I^{\pi} = \prod_{i = 1}^M I^i$,
whose $(i_d,j)$  entry is $\ten{X}_{i_1 \dots i_M}$
with $ j  =  1  +  \sum_{l=1,l\neq d}^M (i_l  -  1)\prod_{m=1, m\neq d}^{l-1}I^m$.
One can also {\em fold}  a matrix
$\vect{X}$
back to a tensor
$\ten{X}  =  \vect{X}^{\ip{d}}$,
with $\ten{X}_{ i_1 \dots i_M }  =  \vect{X}_{i_d j}$,
and $j$ as defined above.
A slice in a tensor $\ten{X}$ is a matrix $\bm{x}$ obtained by fixing all but two $\ten{X}$'s indices.
The inner product of two $M$-order
tensors $\ten{X}$ and $\ten{Y}$ is $\ip{\ten{X} , \ten{Y}}  = 
\sum_{i_1=1}^{I^1} ... \sum_{i_M = 1}^{I^M} \ten{X}_{i_1 ... i_M} \ten{Y}_{i_1 ... i_M}$,
the Frobenius norm is $\|\ten{X}\|_F  =  \sqrt{\langle \ten{X}, \ten{X} \rangle}$,
$\| \ten{X} \|_{\max}$ returns the value of the element in $\ten{X}$ with the maximum absolute value.

For a proper and lower-semi-continuous function $f$,
$\partial f$ denotes its Frechet subdifferential
\citep{attouch2013convergence}.

Finally,
$\SO{\cdot}$ is the observer operator,
i.e., 
given a binary tensor $\vect{\Omega} \in \{ 0, 1 \}^{I_1 \times ... \times I_M}$
and an arbitrary tensor $\ten{X} \in \R^{I_1 \times ... \times I_M}$,
$[ \SO{ \ten{X} } ]_{i_1 ... i_M} = \ten{X}_{i_1 ... i_M}$
if $\vect{\Omega}_{i_1 ... i_M} = 1$ and 
$[ \SO{ \ten{X} } ]_{i_1 ... i_M} = 0$ otherwise.

\section{Related Works}
\label{sec:relworks}

\subsection{Low-Rank Matrix Learning}
\label{sec:rel:nonreg}

Learning 
of a low-rank matrix 
$\vect{X}\in\R^{m\times n}$
can be formulated as the following optimization problem:
\begin{align}
	\min\nolimits_{\vect{X}}
	f(\vect{X}) + \lambda r( \vect{X} ),
	\label{eq:lrmat}
\end{align}
where
$r$ is a low-rank regularizer, 
$\lambda \ge 0$ is a hyperparameter, and
$f$ is a 
loss function that  is
$\rho$-Lipschitz smooth\footnote{In other words, 
	$\| \nabla f(\vect{X})  -  \nabla f(\vect{Y}) \|_F \le  \rho \NM{\vect{X}  -  \vect{Y}}{F}$ for any $\vect{X}, \vect{Y}$.}.
Existing methods
for low-rank matrix learning generally
fall into three
types:
(i) nuclear norm minimization; (ii)
nonconvex regularization; 
and (iii) matrix factorization.

\subsubsection{Nuclear Norm Minimization}

A common choice
for $r$ is 
the nuclear norm regularizer.
Using the 
proximal algorithm~\citep{parikh2013proximal}
on (\ref{eq:lrmat}),
the iterate at iteration $t$ is given by
$\vect{X}_{t+1} = \Px{\frac{\lambda}{\tau} \|\cdot\|_*}{ \bm{Z}_t}$,
where 
\begin{equation} \label{eq:zt21}
\bm{Z}_t = \vect{X}_t - \frac{1}{\tau} \nabla f(\vect{X}_t).
\end{equation} 
Here,
$\tau > \rho$ controls the stepsize ($1/\tau$), and
\begin{align}
\Px{\frac{\lambda}{\tau} \NM{\cdot}{*}}{\bm{Z}}
= \arg\min\nolimits_{\vect{X}} \frac{1}{2} \NM{\vect{X} - \bm{Z}}{F}^2  
+ \frac{\lambda}{\tau} \NM{\vect{X}}{*}
\label{eq:proxsvt}
\end{align}
is the proximal step.
The following Lemma shows that 
$\Px{\frac{\lambda}{\tau} \NM{\cdot}{*}}{\bm{Z}}$
can be obtained by shrinking the singular values 
of $\bm{Z}$, 
which 
encourages
$\vect{X}_t$ to be low-rank.

\begin{lemma}\citep{cai2010singular}
$\Px{\lambda\NM{\cdot}{*}}{\bm{Z}} = \bm{U}
(\bm{\Sigma} - \lambda \bm{I})_+ \bm{V}^{\top}$,
where 
$\bm{U} \bm{\Sigma} \bm{V}^{\top}$ is the SVD of $\bm{Z}$,
and
$\left[ (\vect{X})_+ \right]_{ij} = \max(\vect{X}_{ij}, 0)$.
\label{lem:svt}
\end{lemma}


A special class of low-rank matrix learning problems is
matrix completion, which attempts to find a low-rank matrix that agrees with the observations in data $\bm{O}$:
\begin{align}
\min\nolimits_{\vect{X}}
\frac{1}{2}\NM{ \SO{\vect{X} - \bm{O}} }{F}^2
+ \lambda \NM{\vect{X}}{*}.
\label{eq:matcomp}
\end{align}
Here, positions of the observed elements 
in $\bm{O}$
are indicated by $1$'s
in the binary matrix $\vect{\Omega}$.
Setting
$f(\vect{X}) = \frac{1}{2}\NM{\SO{\vect{X} - \bm{O}}}{F}^2$
in 
(\ref{eq:lrmat}),
$\vect{Z}_t$ in (\ref{eq:zt21}) becomes:
\begin{align}
\label{eq:zt1}
\vect{Z}_t
=
\vect{X}_{t} - \frac{1}{\tau} \SO{ \vect{X}_{t} - \vect{O} }.
\end{align}
Note that $\vect{X}_t$ is low-rank and $\frac{1}{\tau} \SO{ \vect{X}_{t} - \vect{O}
}$ is sparse.
$\vect{Z}_t$ thus
has a ``sparse plus low-rank'' structure.
This allows
the SVD  computation
in Lemma~\ref{lem:svt} 
to be 
much more 
efficient~\citep{mazumder2010spectral}.
Specifically, 
on using the power method
to compute 
$\vect{Z}_t$'s
SVD, most effort is spent on 
multiplications of the forms $\vect{Z}_t \bm{b}$ 
and $\bm{a}^{\top} \vect{Z}_t$ 
(where 
$\bm{a} \in \R^{n}$
and 
$\bm{b} \in \R^{m}$).
Let 
$\vect{X}_{t}$ 
in (\ref{eq:zt1})
be low-rank factorized as
$\bm{U}_{t} \bm{V}_{t}^{\top}$, where 
$\bm{U}_{t}\in\R^{m\times k_t}$ and $\bm{V}_{t}\in\R^{n\times k_t}$
with rank
$k_{t}$.
Computing
\begin{align}
\vect{Z}_t \bm{b}
= 
\bm{U}_{t} \left( \bm{V}_{t}^{\top} \bm{b} \right) 
- \frac{1}{\tau} \SO{ \vect{Y}_{t} - \bm{O} } \bm{b}
\label{eq:zt1v}
\end{align}
takes $O( (m + n) k_t + \| \vect{\Omega} \|_1)$ time.
Usually, $k_{t} \ll n$ and $\| \vect{\Omega} \|_1 \ll m n$.
Thus, this is much faster than directly multiplying $\vect{Z}_t$ and $\bm{b}$,
which takes $O( m n )$ time.
The same holds for computing $\bm{a}^{\top} \vect{Z}_t$.
The proximal step in \eqref{eq:proxsvt} 
takes a total of $O( (m + n) k_t k_{t + 1} + \| \vect{\Omega} \|_1 k_{t + 1} )$ time,
while a direct computation
without utilizing the ``sparse plus low-rank'' structure
takes $O( m n k_{t + 1} )$ time.
Besides, as
only 
$\SO{\vect{X}_{t}}$ 
and 
the factorized form of $\vect{X}_{t}$ 
need to be kept,
the space complexity is reduced from $O( m n )$ to
$O( (m + n) k_t + \| \vect{\Omega} \|_1 )$.


\subsubsection{Nonconvex Low-Rank Regularizer}

Instead of using a convex $r$ in \eqref{eq:lrmat}, the following
nonconvex regularizer 
has been commonly used~\citep{gui2016towards,lu2016nonconvex,gu2017weighted,yao2018large}:
\begin{align}
\phi(  \vect{X} ) = \sum\nolimits_{i = 1}^{n} \kappa( \sigma_i( \vect{X} ) ),
\label{eq:lrphi}
\end{align}
where $\kappa$ is nonconvex and possibly nonsmooth.
We assume the following on $\kappa$.

\begin{assumption}
\label{ass:kappa}
$\kappa(\alpha)$ is a concave and
non-decreasing
function on $\alpha \ge 0$,
with $\kappa(0) = 0$ and 
$\lim_{\alpha \rightarrow 0^+} \kappa'(\alpha) = \kappa_0$ 
for a positive constant $\kappa_0$.
\end{assumption}

Table~\ref{tab:regdef} shows
the 
$\kappa$'s
corresponding 
to the popular nonconvex regularizers of
capped-$\ell_1$ penalty
\citep{zhang2010analysis},  
log-sum-penalty (LSP)
\citep{candes2008enhancing},
truncated nuclear norm (TNN)~\citep{hu2013fast},
smoothed-capped-absolute-deviation (SCAD)~\citep{fan2001variable},
and minimax concave penalty (MCP)~\citep{zhang2010nearly}.
These nonconvex regularizers have similar statistical guarantees~\citep{gui2016towards}, and 
perform empirically better than the convex nuclear norm regularizer
\citep{lu2016nonconvex,yao2018large}.
The proximal algorithm can 
also be used,
and converges to a critical point~\citep{attouch2013convergence}.
Analogous to Lemma~\ref{lem:svt},
the underlying proximal step 
\begin{align}
\Px{\frac{\lambda}{\tau} \phi}{\bm{Z}}
= \arg\min\nolimits_{\vect{X}} \frac{1}{2} \NM{\vect{X} - \bm{Z}}{F}^2  
+ \frac{\lambda}{\tau} \phi(\vect{X})
\label{eq:proxgsvt}
\end{align}
can be obtained as follows.
\begin{lemma}{\normalfont~\citep{lu2016nonconvex}}
\label{lem:gsvt}
$\Px{\lambda\phi}{\bm{Z}} = \bm{U} \Diag{  y_1, \dots, y_{n} }
\bm{V}^{\top}$, where
$\bm{U} \bm{\Sigma} \bm{V}^{\top}$ is the SVD of $\bm{Z}$, and
$y_i = \arg \min\nolimits_{y \ge 0} \frac{1}{2} (y - \sigma_i(\bm{Z}))^2 + \lambda \kappa(y)$.
\end{lemma}

\begin{table}[ht]
	\centering
	\caption{Common examples of $\kappa(\sigma_i( \vect{X} ))$.
		Here, $\theta$ is a constant.
		For the capped-$\ell_1$, 
		LSP and 
		MCP, $\theta > 0$;
		for SCAD, $\theta > 2$;
		and for 
		TNN, $\theta$ is a positive integer.}
	\label{tab:regdef}
	\vspace{-10px}
	\small
	\begin{tabular}{c | c } \toprule
		& $\kappa(\sigma_i( \vect{X} ))$
		\\ \midrule
		nuclear norm~\citep{candes2009exact} & $\sigma_i(\vect{X})$  \\ \midrule
		capped-$\ell_1$~\citep{zhang2010analysis} & $\min(\sigma_i( \vect{X} ), \theta )$  \\ \midrule
		LSP~\citep{candes2008enhancing} & $\log (\frac{\sigma_i( \vect{X} )}{\theta}  + 1)$  \\ \midrule
		TNN~\citep{hu2013fast}
		& $\begin{cases}
		\sigma_i( \vect{X} ) & \text{if}\; i  > \theta \\
		0         & \text{otherwise}
		\end{cases}$
		\\ \midrule
		SCAD~\citep{fan2001variable} &
		$\begin{cases}
		\sigma_i( \vect{X} )                                                                      & 
		\text{if}\; \sigma_i( \vect{X} ) \le 1                       \\
		\frac{(2 \theta \sigma_i( \vect{X} ) -\sigma_i( \vect{X} )^2 - 1)}{2(\theta - 1)}       & 
		\text{if}\;  1  <  \sigma_i( \vect{X} )  \le  \theta \\
		\frac{(\theta + 1) ^2}{2}                                                                   &  \text{otherwise}
		\end{cases}$
		\\ \midrule
		MCP~\citep{zhang2010nearly} &
		$\begin{cases}
		\sigma_i( \vect{X} ) - \frac{\alpha^2}{2 \theta} &
		\text{if}\; \sigma_i( \vect{X} )  \le  \theta \\
		\frac{\theta ^2}{2}             & \text{otherwise}
		\end{cases}$
		\\ \bottomrule
	\end{tabular}
\end{table}

\subsubsection{Matrix Factorization}
\label{sec:matfact}

Note that the aforementioned regularizers require access to individual singular values.
As computing the singular values of a $m\times n$ matrix (with
$m\ge n$) via SVD takes $O(mn^2)$
time, this can be costly for a large matrix.  
Even when rank-$k$ truncated SVD is used, the computation cost is  still
$O(mnk)$.
To reduce the computational burden, factored low-rank regularizers are
proposed~\citep{srebro2005maximum,mazumder2010spectral}.
Specifically,
equation \eqref{eq:lrmat} is rewritten into a factored form as   
\begin{equation}
\label{eq:mc_fact}
\min\nolimits_{\bm{W},\bm{H}}
f(\bm{W}\bm{H}^\top) + \lambda \cdot h(\bm{W},\bm{H}),
\end{equation}
where $\vect{X}$ is factorized as
	$\bm{W}\bm{H}^\top$ 
with
$\bm{W} \in \R^{m \times k}$ and $\bm{H} \in \R^{n \times k}$, 
$h$ is a regularizer on
$\vect{W}$ and  $\vect{H}$,
and 
$\lambda \geq 0$ is a hyperparameter. 
When $\lambda = 0$, this reduces to matrix factorization~\citep{vandereycken2013low,boumal2015low,tu2016low,wang2017unified}.
After 
factorization,
gradient descent or alternative minimization
are usually used for optimization.
When 
certain conditions (such as
proper initialization, 
restricted strong convexity~(RSC)~\citep{negahban2012restricted},
or restricted isometry property~(RIP)~\citep{candes2005decoding}) are met,
statistical guarantees can be obtained~\citep{zheng2015convergent,tu2016low,wang2017unified}.

Note that
in Table~\ref{tab:regdef},
the nuclear norm is the only regularizer $r(\vect{X})$ 
that has an equivalent factored form $h(\bm{W},\bm{H})$. 
For a matrix $\vect{X}$ with ground-truth rank no larger than $k$, 
it has been shown that  the
nuclear norm can be
rewritten in a factored form 
as 
\citep{srebro2005maximum} 
\begin{align*}
\NM{\vect{X}}{*} 
= 
\min\nolimits_{\vect{X} = \bm{W}\bm{H}^\top} \frac{1}{2} 
\left( \NM{\bm{W}}{F}^2 + \NM{\bm{H}}{F}^2 \right).
\end{align*}
However,
the other nonconvex regularizers
need to penalize individual singular values,  and so 
cannot be written in factored form. 


\subsection{Low-Rank Tensor Learning} 
\label{sec:ten:reg}



A $M$-order tensor
$\ten{X}$
has rank one if
it can be written as the outer product of $M$ vectors,
i.e.,
$\ten{X} = \bm{x}^1 \circ \bm{x}^2 \circ \dots\circ \bm{x}^M$
where $\circ$ denotes the outer product
(i.e., 
$\ten{X}_{i_1, \ldots, i_M} = \bm{x}^1_{i_1} \cdot \bm{x}^2_{i_2} \cdot
\dots\cdot \bm{x}^M_{i_M}$).
In general,
the rank of a tensor $\ten{X}$
is the smallest number of rank-one tensors
that generate $\ten{X}$ as their sum~\citep{Kolda2009}.


To impose a low-rank structure on  tensors,
factorization methods
(such as the Tucker / CP~\citep{Kolda2009,hong2020generalized}, 
tensor-train~\citep{oseledets2011tensor}
and tensor ring~\citep{zhao2016tensor} decompositions)
have been used for
low-rank tensor
learning.
These methods
assume that the tensor can be decomposed into low-rank factor matrices
\citep{Kolda2009}, which are 
then learned by alternating least squares 
or coordinate descent~\citep{acar2011scalable,xu2013tmac,balazevic2019tucker}. 
\cite{kressner2014low}~proposed  to utilize the Riemannian
structure on the manifold of tensors with fixed multilinear rank, and then perform nonlinear conjugate gradient descent.
It can be speeded up by preconditioning 
\citep{kasai2016low}.
However, these models suffer from the problem of local minimum, and
have no theoretical 
guarantee on the convergence rate.
Moreover, its per-iteration cost depends on the product of all the mode 
ranks, and so can be expensive.
Thus,
they may lead to worse approximations and inferior performance
\citep{tomioka2011statistical,liu2013tensor,guo2017efficient}.

Due to the above limitations,
the nuclear norm, which has been commonly used in low-rank matrix learning,
has been extended to the learning of low-rank tensors
\citep{tomioka2010estimation,signoretto2011tensor,gu2014robust,yuan2016tensor,zhang2017exact}.
The most commonly used low-rank tensor regularizer
is the following (convex) overlapped nuclear norm:

\begin{definition} \label{def:overlap}
\citep{tomioka2010estimation}
\label{def:tensor_norm}
The overlapped nuclear norm
of a
$M$-order tensor $\ten{X}$ is
$\NM{\ten{X}}{\text{overlap}}$
$=$
$\sum_{i=1}^M \lambda_i \NM{\ten{X}_{\ip{i}}}{*}$,
where $\{\lambda_i  \ge  0\}$ are hyperparameters.
\end{definition}


Note that
the nuclear norm is a convex envelop
of the matrix rank~\citep{candes2009exact}.
Similarly,
$\NM{\ten{X}}{\text{overlap}}$ 
is a convex envelop of the tensor rank~\citep{tomioka2010estimation,tomioka2011statistical}.
Empirically,
$\NM{\ten{X}}{\text{overlap}}$ has better 
performance
than 
the other nuclear norm variants 
in many tensor applications
such as image inpainting~\citep{liu2013tensor} and
multi-relational link prediction~\citep{guo2017efficient}.
On the theoretical side,
let $\ten{X}^*$ be the ground-truth tensor, and $\ten{X}$ be the tensor obtained by solving the overlapped nuclear norm regularized problem.
The statistical error
between 
$\ten{X}^*$ and $\ten{X}$ 
has been established 
in tensor decomposition~\citep{tomioka2011statistical}
and robust tensor decomposition problems~\citep{gu2014robust}.
Specifically, 
under the restricted strong convexity (RSC) condition~\citep{negahban2012restricted},
$\| \ten{X}^* - \ten{X} \|_F$
can be bounded by
$O(\sigma \sum\nolimits_{i = 1}^M \sqrt{k_i})$,
where
$\sigma$ is the noise level 
and $k_i$ is the rank of $\ten{X}^*_{\ip{i}}$.
Furthermore,
we can see that 
when $\sigma = 0$ (no noise),
exactly recovery can be guaranteed.

\subsection{Proximal Average (PA) Algorithm}
\label{sec:pa}

Let $\mathcal{H}$ be
a Hilbert space of
$\ten{X}$, which can be a scalar/vector/matrix/tensor variable.
Consider the following optimization problem:
\begin{align}
\min\nolimits_{\ten{X} \in \mathcal{H}} 
F(\ten{X})
=
f(\ten{X}) 
+ \sum\nolimits_{i = 1}^K \lambda_i \, g_i(\ten{X}),
\label{eq:compopt}
\end{align}
where $f$ is the loss
and each $g_i$ is a regularizer with hyper-parameter $\{ \lambda_i \}$.
Often, 
$g(\ten{X}) = \sum_{i = 1}^K$ 
$\lambda_i$ 
$g_i(\ten{X})$ is complicated,
and 
its proximal step
does not have a simple solution.
Hence, the proximal algorithm cannot be efficiently used.
However, it is possible that the proximal step for 
each individual $g_i$ can be easy obtained.
For example,
let $g_1(\vect{X}) = \NM{\vect{X}}{1}$
and $g_2(\vect{X}) = \NM{\vect{X}}{*}$.
The closed-form solution on the proximal step for $g_1$ (resp. $g_2$)
is 
given by the soft-thresholding operator~\citep{efron2004least} 
(resp. singular value thresholding operator~\citep{cai2010singular}).
However,
the closed-form solution
does not exist for the proximal step with $g_1 + g_2$.


In this case, the
proximal average (PA) 
algorithm 
\citep{bauschke2008proximal}  can be used instead. 
The PA algorithm generates $\ten{X}_{t}$'s as
\begin{align}
\ten{X}_{t}
& = \sum\nolimits_{i = 1}^K
\ten{Y}_{t}^i,
\label{eq:proxavg1}
\\
\ten{Z}_t 
& =\ten{X}_t - \frac{1}{\tau} \nabla f(\ten{X}_t), 
\label{eq:proxavg2} 
\\
\ten{Y}_{t + 1}^i
& = \Px{ \frac{\lambda_i g_i}{\tau} }{\ten{Z}_t},
\quad
i = 1, \dots, K.
\label{eq:proxavg3} 
\end{align}
As each individual proximal step in \eqref{eq:proxavg3} is easy, 
the
PA algorithm can be significantly faster than the proximal algorithm~\citep{yu2013better,zhong2014gradient,yu2015minimizing,shen2017adaptive}.
When both $f$ and $g$ are convex, the
PA algorithm converges to an optimal solution of \eqref{eq:compopt}
with a proper choice of stepsize $\tau$~\citep{yu2013better,shen2017adaptive}.
Recently, the
PA algorithm has also been extended 
to nonconvex $f$ and $g_i$'s 
\citep{zhong2014gradient,yu2015minimizing}.
Moreover, $\tau$ can be adaptively changed to obtain an empirically faster convergence~\citep{shen2017adaptive}.


\section{Proposed Algorithm}
\label{sec:proalg}

Analogous to 
the low-rank matrix completion problem in
\eqref{eq:lrmat}, we consider
the following low-rank tensor completion problem with a nonconvex extension of
the overlapped nuclear norm:
\begin{align}
\label{eq:pro}
\min_{\ten{X}}
\sum\nolimits_{\mathbf{\Omega}_{i_1 ... i_M} = 1}
\ell\left( \ten{X}_{i_1 \dots i_M}, \ten{O}_{i_1 \dots i_M} \right) 
+ \sum\nolimits_{i = 1}^D \lambda_i \, \phi( \ten{X}_{\ip{i}} ).
\end{align}
Here,
the
observed elements are in $\ten{O}_{i_1 ... i_M}$,
$\ten{X}$ is the tensor to be recovered,
$\ell(\cdot,\cdot)$ is a smooth loss function, and
$\phi$ 
is a nonconvex regularizer in the form in 
(\ref{eq:lrphi}).
Unlike 
the overlapped nuclear norm in Definition~\ref{def:overlap},
here we only sum over
$D  \le  M$ modes.
This is useful 
when some modes are already small
(e.g., the number of bands in color images), and
so do not need to be low-rank regularized.
When $D=M$ and $\kappa(\alpha) = \alpha$
in (\ref{eq:lrphi}), problem
\eqref{eq:pro} reduces to  
(convex) overlapped nuclear norm regularization. 
When $D  =  1$ and $\ell$
is the square loss, 
\eqref{eq:pro} reduces to the matrix completion problem:
\begin{align*}
\min\nolimits_{\vect{X} \in \R^{I^1 \times (\frac{I^{\pi}}{I^1}) } } 
\frac{1}{2} \NM{ \SO{\vect{X} - \ten{O}_{\ip{1}}} }{F}^2 + \lambda_1 \phi( \vect{X} ),
\end{align*}
which can be solved by the proximal algorithm as in~\citep{lu2016nonconvex,yao2018large}.  
In the sequel, we only consider $D  >1$.

\subsection{Issues with Existing Solvers}
\label{sec:exsolvers}

First, consider the case 
where $\kappa$ 
in (\ref{eq:lrphi})
is convex.
While $D$ may not be equal to  $M$, 
it can be easily shown that 
existing optimization solvers  in
\citep{tomioka2010estimation,boyd2011distributed,liu2013tensor}
can still be used.
However,
when $\kappa$ is nonconvex,
the fast low-rank tensor completion (FaLRTC) solver~\citep{liu2013tensor}
cannot be applied,
as the dual of \eqref{eq:pro}
cannot be 
derived.
\citeauthor{tomioka2010estimation} 
(\citeyear{tomioka2010estimation})
used the alternating direction of multiple multipliers (ADMM)~\citep{boyd2011distributed} solver
for the overlapped nuclear norm.
Recently, it
is used in
\citep{Chen2020ANL}
to solve
a special case of \eqref{eq:pro},
in which $\phi$ is the
truncated nuclear norm (TNN) regularizer 
(see Table~\ref{tab:regdef}).
Specifically,
\eqref{eq:pro} 
is first reformulated
as
\begin{align*}
\min\nolimits_{\ten{X}}
\sum\nolimits_{\mathbf{\Omega}_{i_1 ... i_M} = 1}
\ell\left( \ten{X}_{i_1 \dots i_M}, \ten{O}_{i_1 \dots i_M} \right) 
+ \sum\nolimits_{i = 1}^D \lambda_i \, \phi( \vect{X}_i )
\;\text{s.t.}\;
\vect{X}_i = \ten{X}_{\ip{i}},
\;
i = 1, \dots, D.
\end{align*}
Using ADMM,
it 
then
generates iterates as
\begin{align}
\ten{X}_{t + 1}
& = \arg\min\nolimits_{\ten{X}} 
\!\!\!\!\!\!
\sum_{\mathbf{\Omega}_{i_1 ... i_M} = 1}
\!\!\!\!\!\!
\ell\left( \ten{X}_{i_1 \dots i_M}, \ten{O}_{i_1 \dots i_M} \right) 
\! + \! \frac{\zeta}{2} \sum\nolimits_{i = 1}^D 
\big\| (\vect{X}_i)_t \! - \! \ten{X}_{\ip{i}} \! + \! \frac{1}{\zeta} (\bm{M}_i)_{t} \big\|_F^2,
\label{eq:admm1}
\\
(\vect{X}_i)_{t + 1}
& = \Px{\frac{ \lambda_i }{ \zeta }}{(\vect{X}_i)_{t + 1} + \frac{1}{\zeta} (\bm{M}_i)_{t}},
\quad
i = 1, \dots, D,
\label{eq:admm2}
\\
(\bm{M}_i)_{t + 1}
& = (\bm{M}_i)_{t} + \frac{1}{\zeta}
\left( (\vect{X}_i)_{t + 1} - (\ten{X}_{\ip{i}})_{t + 1} \right),
\quad
i = 1, \dots, D,
\label{eq:admm3}
\end{align}
where $\bm{M}_i$'s are the dual variables, and $\zeta > 0$.
The proximal step in \eqref{eq:admm2} can be computed from Lemma~\ref{lem:gsvt}.
Convergence of this ADMM procedure is guaranteed in~\citep{hong2016convergence,wang2019global}.
However, 
it does not utilize the sparsity induced by $\vect{\Omega}$.
Moreover, 
as the tensor $\ten{X}$ needs to be folded and unfolded repeatedly,
the iterative procedure is expensive, taking
$O(I^{\pi})$ space and $O( I^{\pi} \sum_{i = 1}^D I^i )$ time per iteration.

On the other hand, the proximal algorithm 
(Section~\ref{sec:rel:nonreg}) cannot be easily used, as
the proximal step for $\sum_{i = 1}^D \lambda_i \phi( \ten{X}_{\ip{i}} )$ is 
not simple in general.


\subsection{Structure-aware Proximal Average Iterations} 
\label{ssec:proxavg}

Note that 
$\phi$ in \eqref{eq:lrphi} 
admits a difference-of-convex decomposition~\citep{hartman1959functions,tao2005dc},
i.e., 
$\phi$ can be decomposed as $\phi = \phi_1 - \phi_2$ where $\phi_1$ and $\phi_2$ are convex~\citep{yao2018large}.
The proximal average (PA) algorithm  (Section~\ref{sec:pa})
has been 
recently
extended for nonconvex $f$ and $g_i$'s,
where each $g_i$
admits a difference-of-convex decomposition
\citep{zhong2014gradient}.
Hence,
as \eqref{eq:pro} is in the form in \eqref{eq:compopt},
one can
generate
the PA 
iterates 
as:
\begin{eqnarray}
\ten{X}_{t} & = & \sum\nolimits_{i = 1}^D \ten{Y}^i_{t},
\label{eq:tsplr_1}
\\
\ten{Z}_t
& = &
\ten{X}_t - \frac{1}{\tau} \varpi(\ten{X}_t),
\label{eq:tsplr_2}
\\
\ten{Y}^i_{t + 1} & = & \big[ \Px{\frac{\lambda_i \phi}{\tau} }{
[\ten{Z}_t]_{\ip{i}} } \big]^{\ip{i}},  
\quad
i=1,\dots,D. \label{eq:tsplr_3}
\end{eqnarray}
where $\xi(\ten{X}_t)$ is a sparse tensor with
\begin{align}
\label{eq:l_gradient}
\big[ \varpi(\ten{X}_t) \big]_{i_1 \dots i_M}
=
\begin{cases}
0 
& (i_1, \dots, i_M) \not\in \mathbf{\Omega}
\\
\ell'\left( [\ten{X}_t]_{i_1 \dots i_M}, \ten{O}_{i_1 \dots i_M} \right)  
& (i_1, \dots, i_M) \in \mathbf{\Omega}
\end{cases}.
\end{align}
In \eqref{eq:tsplr_3},
each individual proximal step 
can be computed 
using Lemma~\ref{lem:gsvt}.
However, tensor folding and unfolding are still required.
A direct application of the PA algorithm
is as expensive as using ADMM (see Table~\ref{tab:overview}).

In the following, 
we show 
that by utilizing the ``sparse plus low-rank'' structure,
the PA iterations can be computed efficiently  without tensor
folding/unfolding.
In the earlier conference version~\citep{yao2019efficient},
we only considered the case $M = 3$.
Here, 
we extend this to $M \ge 3$ by noting 
that
the coordinate format of sparse tensors can naturally handle tensors with arbitrary orders (Section~\ref{ssec:effxz})
and
the proximal step can be performed without tensor folding/unfolding (Section~\ref{ssec:effy}).

\subsubsection{Efficient Computations of $\mathscr{ X }_{t}$ and $\ten{Z}_t$
in (\ref{eq:tsplr_1}),
(\ref{eq:tsplr_2})}
\label{ssec:effxz}

First, rewrite \eqref{eq:tsplr_3} as
$\ten{Y}^i_{t + 1} = 
[\vect{Y}^i_{t + 1}]^{\ip{i}}$,   where
$\vect{Y}^i_{t + 1} = \Px{ \frac{\lambda_i \phi }{ \tau } }{ \vect{Z}^i_t }$
and 
$\vect{Z}^i_t = \left[ \ten{Z}_t \right]_{\ip{i}}$.
Recall that $\vect{Y}_t^i$ obtained from the proximal step
is low-rank. 
Let its rank be $k^i_t$. 
Hence,
$\vect{Y}_t^i$ can be represented as
$\bm{U}_{t}^i  (\bm{V}_{t}^i)^{\top}$, 
where $\bm{U}_{t}^i \in \R^{I^i \times k_t^i}$
and $\bm{V}_{t}^i \in \R^{(\frac{I^{\pi}}{I^i}) \times k_t^i}$.
In each PA iteration, we avoid constructing the dense $\ten{Y}_t^i$ by storing 
the above low-rank factorized form of $\vect{Y}_t^i$ instead. Similarly,
we also avoid constructing $\ten{X}_t$ in \eqref{eq:tsplr_1} by storing it implicitly as 
\begin{align}
\ten{X}_t = \sum\nolimits_{i = 1}^D \big( \bm{U}_{t}^i ( \bm{V}_{t}^i )^{\top} \big)^{\ip{i}}.
\label{eq:lowrf}
\end{align}
$\ten{Z}_t$ in \eqref{eq:tsplr_2} can then be rewritten as
\begin{align}
\ten{Z}_t = \sum\nolimits_{i = 1}^D \left( \bm{U}_{t}^i  (\bm{V}_{t}^i)^{\top} \right)^{\ip{i}} - \frac{1}{\tau} \xi( \ten{X}_t ).
\label{eq:zt}
\end{align}
The sparse tensor
$\xi(\ten{X}_t)$ 
in (\ref{eq:l_gradient})
can be 
constructed
efficiently 
by using the coordinate format\footnote{For a sparse  $M$-order tensor,
its $p$th nonzero element 
is	represented
in the coordinate format
as $(i_p^1, \dots, i_p^M, v_p)$,
where $i_p^1, \dots, i_p^M$ are indices on each mode and $v_p$ is the value.}
\citep{bader2007efficient}. 
Using \eqref{eq:lowrf},
each 
$[ \xi(\ten{X}_t) ]_{i_1 \dots i_M}$
can be computed by finding the corresponding rows in 
$\{\bm{U}_{t}^i, \bm{V}_{t}^i\}$ as shown in 
Algorithm~\ref{alg:compp3}.
This takes $O( \sum_{i = 1}^D k_t^i )$ time.

\begin{algorithm}[ht]
\caption{Computing the $p$th element $v_p$ with index $(i_p^1 \dots i_p^M)$ in $\xi(\ten{X}_t)$.}
\small
\begin{algorithmic}[1]
	\REQUIRE factorizations $\{\bm{U}_{t}^1  (\bm{V}_{t}^1)^{\top},\dots, \bm{U}_{t}^D  (\bm{V}_{t}^D)^{\top}\}$ of $\vect{Y}_t^1$, $\dots$, $\vect{Y}_t^D$,
	and observed elements in $\SO{\ten{O}}$;
	\STATE $v_p \leftarrow 0$;
	\FOR{$d = 1, \dots, D$}
		\STATE $\bm{u}^{\top} \leftarrow$ $i^d_p$th row of $\bm{U}_t^d$;
		\STATE $\bm{v}^{\top} \leftarrow$ $(\sum_{k \neq d}^M i^k_p I^{\pi} + i^d_p)$th row of $\bm{V}_t^d$;
		\STATE $v_p \leftarrow v_p + \bm{u}^{\top} \bm{v}$;
	\ENDFOR
	\STATE $v_p \leftarrow \ell'( v_p,\ten{O}_{i_p^1 \dots i_p^M} )$;
	\RETURN $v_p$. 
\end{algorithmic}
\label{alg:compp3}
\end{algorithm}

\subsubsection{Efficient Computation of $\ten{Y}^i_{t + 1}$
in (\ref{eq:tsplr_3})}
\label{ssec:effy}

Recall that the proximal step   
in (\ref{eq:tsplr_3})
requires SVD,
which involves matrix multiplications in the form
$\bm{a}^{\top} ( \ten{Z}_t )_{\ip{i}}$ (where $\bm{a}  \in  \R^{I^i}$)
and
$( \ten{Z}_t )_{\ip{i}} \bm{b}$ (where
$\bm{b}  \in  \R^{\frac{I^{\pi}}{I^i}}$).
Using
the ``sparse plus low-rank'' structure in
(\ref{eq:zt}), these can be computed as 
\begin{align}
\bm{a}^{\top}  ( \ten{Z}_t )_{\ip{i}} 
 =   
\big( \bm{a}^{\top}  \bm{U}_{t}^i \big) (\bm{V}_{t}^i )^{\top}
 + 
\sum\nolimits_{j \neq i} 
\bm{a}^{\top} 
\big[ ( \bm{U}_{t}^j (\bm{V}_{t}^j)^{\top} )^{\ip{j}} \big]_{\ip{i}}
- \frac{1}{\tau} \bm{a}^{\top} [ \xi(\ten{X}_t) ]_{\ip{i}},
\label{eq:tenztu}
\end{align}
and
\begin{align}
( \ten{Z}_t )_{\ip{i}} \bm{b} 
= \bm{U}_{t}^i \big[  (\bm{V}_{t}^i)^{\top} \bm{b} \big] 
+ \sum\nolimits_{j \neq i}
\big[ ( \bm{U}_{t}^j (\bm{V}_{t}^j)^{\top} )^{\ip{j}} \big]_{\ip{i}}
\bm{b}
- \frac{1}{\tau} [ \xi(\ten{X}_t) ]_{\ip{i}} \bm{b}.
\label{eq:tenztv}
\end{align}
The first terms 
in \eqref{eq:tenztu} and \eqref{eq:tenztv}
can be easily computed
in $O((\frac{I^{\pi}}{I^i}  +  I^i) k_t^i)$ space and time.
The last terms 
($\bm{a}^{\top} \left[ \xi(\ten{X}_t) \right]_{\ip{i}}$
and $\left[ \xi(\ten{X}_t) \right]_{\ip{i}} \bm{b}$)
are sparse, and
can be computed in $O( \| \vect{\Omega} \|_1 )$ space and time by
using sparse tensor packages
such as the Tensor Toolbox~\citep{bader2007efficient}.
However, a direct computation of  the 
$\bm{a}^{\top} [ ( \bm{U}_t^j (\bm{V}_t^j)^{\top} )^{\ip{j}} ]_{\ip{i}}$
and 
$[  ( \bm{U}_t^j (\bm{V}_t^j)^{\top} )^{\ip{j}}  ]_{\ip{i}} \bm{b}$
terms
involves tensor folding/unfolding, and is expensive.
By examining how elements are ordered by folding/unfolding,
the following shows that
these multiplications
can indeed be computed efficiently without 
explicit folding/unfolding.

\begin{proposition} \label{pr:mulv}
Let $\bm{U} \in \R^{I^{j} \times k}$, $\bm{V} \in \R^{(\frac{I^{\pi}}{I^j}) \times k}$,
and $\bm{u}_p$ (resp.$\bm{v}_p$) be the $p$th column of $\bm{U}$ (resp.$\bm{V}$).
For any $i \neq j$,
$\bm{a} \in \R^{I^{i}}$
and $\bm{b} \in \R^{\frac{I^{\pi}}{I^i}}$,
we have
\begin{align}
\bm{a}^{\top}
\big[  
( \bm{U} \bm{V}^{\top} )^{\ip{j}} 
\big]_{\ip{i}}
& = 
\sum\nolimits_{p = 1}^k
\bm{u}_p^{\top}
\otimes
\big[ 
\bm{a}^{\top} \text{mat}(\bm{v}_p; I^i, \bar{I}^{ij}) 
\big] 
,
\label{eq:mulv1}
\\
\big[  
( \bm{U} \bm{V}^{\top} )^{\ip{j}} 
\big]_{\ip{i}}
\bm{b}
& = 
\sum\nolimits_{p = 1}^k
\text{mat}\big( \bm{v}_p; I^i, \bar{I}^{ij} \big) 
\text{mat}\big( \bm{b}; \bar{I}^{ij}, I^j \big) 
\bm{u}_p
,
\label{eq:mulv2}
\end{align}
where $\otimes$ is the Kronecker product,
$\bar{I}^{ij} = I^{\pi} / (I^i I^j)$,
and
$\text{mat}(\bm{x}; a, b)$ reshapes a vector 
$\bm{x}$
of length $ab$ into a matrix of size $a  \times b$.
\end{proposition}

In the earlier conference version~\citep{yao2019efficient}, Proposition~3.2 there
(not the proposed algorithm) limits the usage to $M = 3$.
Without 
Proposition~\ref{pr:mulv},
the
algorithm can suffer from expensive computation cost, and  thus has no efficiency
advantage over the simple use of the PA algorithm.
	Specifically,
	when mapping the vector $\bm{v}_p$ back to a matrix,
	we do not need to take a special treatment on the size of matrix.
	The reason is that,
	$\bm{v}_p$ has $I_i I_j$ elements
	and we just need to map it back to a matrix of size $I_i \times I_j$.   
	Thus, we do not have parameters 
	for \textit{mat} operation in the conference version.
	However,
	when $M > 3$,
	$\bm{v}_p$ has $I^{\pi}/I^i$ elements,
	we need to check 
	whether ideas used in the conference version 
	can be done in a similar way.
	As a result,
	we have two more parameters
	for the \textit{mat} operation here,
	which customize reshaping matrix to a proper size.

\begin{remark}
For a second-order tensor (i.e., matrix case with $M = 2$),
Proposition~\ref{pr:mulv} becomes
\begin{align*}
\bm{a}^{\top} \big[ ( \bm{U} \bm{V}^{\top} )^{\ip{1}} \big]_{\ip{2}} =
\sum\nolimits_{p = 1}^k (\bm{a}^{\top} \bm{v}_p) \bm{u}_p^{\top}
\quad\text{and}\quad
\big[ (\bm{U} \bm{V}^{\top} )^{\ip{1}} \big]_{\ip{2}} \bm{b} 
= \sum\nolimits_{p = 1}^k
\bm{v}_p (\bm{b}^{\top} \bm{u}_p).
\end{align*}
With the usual square loss (i.e., $\sum\nolimits_{\mathbf{\Omega}_{i_1 ... i_M} = 1}
\ell\left( \ten{X}_{i_1 \dots i_M}, \ten{O}_{i_1 \dots i_M} \right)$ in  
(\ref{eq:pro}) equals
$\frac{1}{2} \NM{ \SO{ \ten{X} - \ten{O} } }{F}^2$),
\eqref{eq:tenztv} then reduces to \eqref{eq:zt1v}
when $D = 1$.
When $D = 2$,
$\sum\nolimits_{i = 1}^D \lambda_i \, \phi( \ten{X}_{\ip{i}} )$ 
in (\ref{eq:pro}) becomes
	$\lambda_1 \phi( \vect{X} ) + \lambda_2 \phi( \vect{X}^{\top})=
	(\lambda_1 + \lambda_2) \phi( \vect{X})$, and is the same as the corresponding
	regularizer 
when $D = 1$.
Hence, the reduction from \eqref{eq:tenztv} to \eqref{eq:zt1v} still holds.
\end{remark}

\subsubsection{Time and Space Complexities}
\label{sec:compl}

A direct computation of
$\bm{a}^{\top} [ ( \bm{U}_t^j (\bm{V}_t^j)^{\top} )^{\ip{j}} ]_{\ip{i}}$ 
takes $O(k_t^i I^{\pi})$ time and $O(I^{\pi})$ space.
By using the computation in Proposition~\ref{pr:mulv},
these 
are reduced 
to
$O( (\frac{ 1 }{ I^i } + \frac{ 1 }{ I^j }) k_t^i I^{\pi} )$ time
and  
$O( ( \frac{ 1 }{ I^j } + \frac{ 1 }{ I^i } ) k_t^i I^{\pi} )$
space.
This is also the case for
$[  ( \bm{U}_t^j (\bm{V}_t^j)^{\top} )^{\ip{j}}  ]_{\ip{i}} \bm{b}$.
Details are in the following.

\begin{table}[H]
	\centering
	\small
	\begin{tabular}{c | c | c | c}
		\toprule
		                  & operation                                 & time                                                           & space                                                              \\ \midrule
		    reshaping     & $\text{mat}(\bm{v}_p; I^j, \bar{I}^{ij})$ & $O(\frac{ I^{\pi} }{ I^i })$                               & $O(\frac{ I^{\pi} }{ I^i })$                                   \\ \midrule
		 multiplication   & $\bm{a}^{\top} (\cdot)$                   & $O(\frac{ I^{\pi} }{ I^i })$                               & $O(\frac{ I^{\pi} }{ I^i })$                                   \\ \midrule
		Kronecker product & $\bm{u}_p^{\top} \otimes (\cdot)$             & $O( \frac{ I^{\pi} }{ I^j } )$                             & $O(\frac{ I^{\pi} }{ I^j })$                                   \\ \midrule
		    summation     & $\sum_{p = 1}^{k_t^i} (\cdot)$                  & $O( \frac{ k_t^i I^{\pi} }{ I^j } )$                           & $O( ( \frac{ 1 }{ I^j } + \frac{ 1 }{ I^i } ) k_t^i I^{\pi} )$ \\ \midrule
		\multicolumn{2}{c|}{total for \eqref{eq:mulv1}} & $O( (\frac{ 1 }{ I^j } + \frac{ k_t^i }{ I^j }) I^{\pi} )$ & $O( ( \frac{ 1 }{ I^j } + \frac{ 1 }{ I^i } ) k_t^i I^{\pi} )$ \\ \bottomrule
	\end{tabular}
\end{table}

\begin{table}[H]
	\centering
	\small
	\begin{tabular}{c | c | c | c}
		\toprule
		               & operation                                                      & time                                                           & space                                                              \\ \midrule
		  reshaping    & $\text{mat}\left( \bm{b}; \bar{I}^{ij}, I^i \right)$           & $O(\frac{ I^{\pi} }{ I^j })$                               & $O(\frac{ I^{\pi} }{ I^j })$                                   \\ \midrule
		multiplication & $(\cdot) \bm{u}_p$                                             & $O(\frac{ I^{\pi} }{ I^i })$                               & $O(\frac{ I^{\pi} }{ I^j })$                                   \\ \midrule
		  reshaping    & $\text{mat}\left( \bm{v}_p; I^j, \bar{I}^{ij} \right) $        & $O( \frac{ I^{\pi} }{ I^i } )$                             & $O( \frac{ I^{\pi} }{ I^i } )$                                 \\ \midrule
		multiplication & $\text{mat}\left( \bm{v}_p; I^j, \bar{I}^{ij} \right) (\cdot)$ & $O( \frac{ I^{\pi} }{ I^j } )$                             & $O( \frac{ I^{\pi} }{ I^i } )$                                 \\ \midrule
		  summation    & $\sum_{p = 1}^{k_t^i} (\cdot)$                                       & $O( k_t^i I^{i} )$                                                 & $O( k_t^i I^{i} )$                                                     \\ \midrule
		         \multicolumn{2}{c|}{total for \eqref{eq:mulv2}}          & $O( (\frac{ 1 }{ I^j } + \frac{ k_t^i }{ I^j }) I^{\pi} )$ & $O( ( \frac{ 1 }{ I^j } + \frac{ 1 }{ I^i } ) k_t^i I^{\pi} )$ \\ \bottomrule
	\end{tabular}
\end{table}

Combining the above,
and noting that we have to keep the factorized form $\bm{U}_t^i
(\bm{V}_t^i)^{\top}$ of $\vect{Y}_t^i$,
computing all 
the proximal steps in \eqref{eq:tsplr_3} takes
\begin{align}
O\big( 
\sum\nolimits_{i = 1}^D \sum\nolimits_{j \neq i} (\frac{1}{I^i}  +  \frac{1}{I^j}) k_t^i
I^{\pi}  +  \| \vect{\Omega} \|_1 \big)
\label{eq:space}
\end{align}
space
and
\begin{align}
O\big( 
\sum\nolimits_{i = 1}^D 
\sum\nolimits_{j \neq i} (\frac{1}{I^i}  +  \frac{1}{I^j}) k_t^i k_{t +
	1}^i I^{\pi}  +  \| \vect{\Omega} \|_1 (k_{t}^i  +  k_{t + 1}^i) 
\big)
\label{eq:time}
\end{align}
time.
Empirically, 
as will be seen in the experimental results in Section~\ref{sec:exp:rank},
$k_t^i$, $k_{t + 1}^i \ll I^i$.
Hence,
(\ref{eq:space}) and 
(\ref{eq:time}) are much smaller
than the complexities with a direct usage of PA and ADMM in Section~\ref{sec:exsolvers}
(Table~\ref{tab:overview}).

\begin{table}[ht]
\centering
\renewcommand{\arraystretch}{1.2}
\caption{Comparison of the proposed NORT with PA and ADMM for \eqref{eq:pro} in Section~\ref{sec:exsolvers}.}
\small
\begin{tabular}{c | c | c | c}
	\toprule
	  \multirow{2}{*}{algorithm}    &              \multicolumn{2}{c|}{complexity}              & adaptive \\
	                                & time per iteration                 & space                & momentum \\ \midrule
	 PA~\citep{zhong2014gradient}   & $O( I^{\pi} \sum_{i = 1}^D I^i  )$ & $O(I^{\pi})$         & $\times$ \\ \midrule
	   ADMM~\citep{Chen2020ANL}     & $O( I^{\pi} \sum_{i = 1}^D I^i  )$ & $O(I^{\pi})$         & $\times$ \\ \midrule
	NORT~(Algorithm~\ref{alg:nort}) & see \eqref{eq:time}                & see \eqref{eq:space} & $\surd$  \\ \bottomrule
\end{tabular}
\label{tab:overview}
\end{table}


\subsection{Use of Adaptive Momentum}
\label{ssec:fastconv}

The PA algorithm uses only first-order information,
and can be slow to converge empirically~\citep{parikh2013proximal}.
To address
this problem,
we adopt adaptive momentum,
which uses historical iterates to speed up convergence.
This has been popularly used in stochastic gradient descent~\citep{duchi2011adaptive,kingma2014adam},
proximal algorithms~\citep{li2015accelerated,yao2017efficient,Li2017ada},
cubic regularization~\citep{wang2020cubic},
and zero-order black-box optimization~\citep{chen2019zo}.
Here, 
we adopt the adaptive scheme in~\citep{Li2017ada}.

The resultant procedure,
which will be called \underline{NO}nconvex \underline{R}egularized \underline{T}ensor (NORT),
is  shown in Algorithm~\ref{alg:nort}.
When the extrapolation step $\bar{\ten{X}}_t$ achieves a lower function value (step~4),
the momentum $\gamma_t$ is increased to further exploit the opportunity of acceleration;
otherwise, $\gamma_t$ is decayed (step~7).
When step~5 is performed,
$\ten{V}_t = \ten{X}_t + \gamma_t ( \ten{X}_t - \ten{X}_{t - 1} )$.
$\ten{Z}_t$ in step~9 becomes
\begin{align}
\ten{Z}_t
=
(1 + \gamma_t) \sum\nolimits_{i = 1}^D 
\big( \bm{U}_{t}^i  (\bm{V}_{t}^i)^{\top} \big)^{\ip{i}}
-  \gamma_t \sum\nolimits_{i = 1}^D 
\big( \bm{U}_{t - 1}^i  (\bm{V}_{t - 1}^i)^{\top} \big)^{\ip{i}}
- \frac{1}{\tau} \xi( \ten{V}_t ),
\label{eq:newz}
\end{align}
which still has the ``sparse plus low-rank'' structure.
When step~7 is performed, $\ten{V}_t = \ten{X}_t$, and 
obviously the resultant $\ten{Z}_t$ is ``sparse plus low-rank''.
Thus, the more efficient reformulations in 
Proposition~\ref{pr:mulv}
can be applied in computing the proximal steps at step~11.
Note that
the rank of $\vect{X}_{t + 1}^i$ 
in step~11
is 
determined implicitly by the proximal step.
As $\ten{X}_t$ and
$\ten{Z}_t$ 
are
implicitly represented in factorized forms,
$\ten{V}_t$ and $\bar{\ten{X}}_t$ (in step~3)
do not need to be explicitly constructed.
As a result,
the resultant time and space complexities are the same 
as those in Section~\ref{sec:compl}.

\begin{algorithm}[ht]
\caption{\underline{NO}nconvex \underline{R}egularized \underline{T}ensor (NORT) Algorithm.}
\small
\begin{algorithmic}[1]
	\STATE \textbf{Initialize} $\tau > \rho + D \kappa_0$, $\gamma_1, p \in (0, 1]$, $\ten{X}_0  =  \ten{X}_1  =  0$
	and $t = 1$;
	\WHILE{not converged}
	\STATE $\bar{\ten{X}}_t \leftarrow \ten{X}_t + \gamma_t ( \ten{X}_t - \ten{X}_{t - 1} )$;
	\IF{$F(\bar{\ten{X}}_t) \le F(\ten{X}_t)$}
	\STATE $\ten{V}_t \leftarrow \bar{\ten{X}}_t$ and $\gamma_{t + 1} \leftarrow \min(\frac{\gamma_t}{p}, 1)$;
	\ELSE
	\STATE $\ten{V}_t \leftarrow \ten{X}_t$ and $\gamma_{t + 1} \leftarrow p \gamma_t$;
	\ENDIF
	\STATE $\ten{Z}_t \leftarrow \ten{V}_t - \frac{1}{\tau} \xi (\ten{V}_t)$; 
	// compute $\xi (\ten{V}_t)$ using Algorithm~\ref{alg:compp3}
	\label{alg:px:v}
	\FOR{$i = 1, \dots, D$}
	\STATE $\vect{X}^i_{t + 1} \leftarrow \Px{\frac{\lambda_i \phi}{\tau} } { ( \ten{Z}_t )_{\ip{i}} }$;
	// keep as 
	$\bm{U}_{t}^i ( \bm{V}_{t}^i )^{\top}$;
	\ENDFOR
	\label{alg:px:x}
	\\
	{\it // implicitly construct $\ten{X}_{t + 1} \leftarrow \sum_{i = 1}^D \big( \bm{U}_{t + 1}^i ( \bm{V}_{t + 1}^i )^{\top} \big)^{\ip{i}}$;}
	\STATE $t = t + 1$
	\ENDWHILE
	\RETURN $\ten{X}_{t}$.
\end{algorithmic}
\label{alg:nort}
\end{algorithm}


\subsection{Convergence Properties}
\label{sec:convana}

In this section,
we analyze the convergence properties of the proposed algorithm.
As can be seen from \eqref{eq:pro},
we have $f(\ten{X}) = \sum\nolimits_{\mathbf{\Omega}_{i_1 \ldots i_M} = 1}
\ell\left( \ten{X}_{i_1 \dots i_M}, \ten{O}_{i_1 \dots i_M} \right) $ here.
Moreover, throughout this section,
we assume that the loss $f$ is (Lipschitz-)smooth.

Note that existing proofs 
for PA algorithm
\citep{yu2013better,zhong2014gradient,yu2015minimizing}
cannot be directly used, as
adaptive momentum has not been used with the PA algorithm
on nonconvex problems (see Table~\ref{tab:overview}),  
and also that they do not involve tensor folding/unfolding operations.  
Our proof strategy will
still follow the three main steps
in proving convergence of PA:
\begin{enumerate}[leftmargin=*]
\item Show 
that the proximal average step with $g_i$'s in \eqref{eq:proxavg3}
corresponds to a regularizer;

\item Show that this regularizer, when combined with 
the loss $f$ in \eqref{eq:compopt},
serves as a good approximation of
the original objective $F$. 

\item Show that the proposed algorithm finds critical points of this approximate optimization problem.
\end{enumerate}
First,
the following Proposition shows that 
the average step in 
\eqref{eq:tsplr_1}
and
proximal steps in
\eqref{eq:tsplr_3}
together
correspond to a new regularizer $\bar{g}_{\tau}$.
\begin{proposition}
\label{pr:reg}
For any $\tau  >  0$,
$\sum_{i = 1}^D [ \Px{ \frac{1} { \tau } \lambda_i \phi }{ [ \ten{Z} ]_{\ip{i}} } ]^{ \ip{i} }= 
\Px{ \frac{1}{ \tau }  \bar{g}_{\tau} }{\ten{Z}}$,
where
\begin{align*}
\bar{g}_{\tau}(\ten{X})
= 
\tau
\big[ 
\min\nolimits_{ \{ \vect{X}_{d} \} :
\sum\nolimits_{d = 1}^D \vect{X}_{d}^{\ip{d}} = \ten{X}} 
\sum\nolimits_{d = 1}^D 
\big( \frac{1}{2} \NM{\vect{X}_d}{F}^2 
+  \frac{\lambda_d}{\tau} \phi (\vect{X}_d) \big)
- \frac{D}{2}\NM{ \ten{X} }{F}^2
\big].
\end{align*}
\end{proposition}

Analogous to (\ref{eq:compopt}),
let the objective  corresponding to regularizer
$\bar{g}_{\tau}$
be
\begin{equation} \label{eq:Ftau}
F_{\tau}(\ten{X}) 
= f(\ten{X}) + \bar{g}_{\tau}(\ten{X}).
\end{equation} 
The following bounds the difference between the optimal values 
($F^{\min}$ and 
$F_{\tau}^{\min}$, respectively)
of the objectives $F$ in (\ref{eq:compopt}) and $F_{\tau}$.
It thus shows that $F_{\tau}$ serves as an approximation to $F$,
which is controlled by $\tau$.

\begin{proposition}
	\label{pr:bnd}
	$0  \le  F^{\min}  -  F_{\tau}^{\min} \le \frac{\kappa_0^2}{ 2 \tau } \sum_{i = 1}^D \lambda_i^2$,
	where $\kappa_0$ is defined in Assumption~\ref{ass:kappa}.
\end{proposition}

Before showing
the convergence
of the proposed algorithm,
the following Proposition
first shows the condition of being
critical points of $F_{\tau}(\ten{X})$.

\begin{proposition}
\label{pr:criti}
If there
exists $\tau > 0$
such that
$\tilde{\ten{X}} = \Px{ \frac{ \bar{g}_{\tau} }{ \tau } }{ \tilde{\ten{X}} - \nabla f( \tilde{\ten{X}} ) / \tau }$,
then $\tilde{\ten{X}}$ is a critical point of $F_{\tau}(\tilde{\ten{X}})$.
\end{proposition}

Finally,
we show how convergence
to critical points can be ensured by the proposed algorithm
under smooth assumption of loss $f$ (Section~\ref{sec:conv:smooth})
and Kurdyka-{\L}ojasiewicz 
condition 
for the approximated objective $F_{\tau}$
(Section~\ref{sec:conv:kl}).

\subsubsection{With Smoothness Assumption on Loss $f$}
\label{sec:conv:smooth}

The following shows that
Algorithm~\ref{alg:nort}
converges
to a critical point
(Theorem~\ref{thm:conv}).

\begin{theorem} \label{thm:conv}
	The sequence $\{ \ten{X}_t \}$ generated from Algorithm~\ref{alg:nort} has at least one limit point,
	and all limits points are critical points of $F_{\tau}(\ten{X})$.
\end{theorem}

\begin{proof}[Sketch, details are in Appendix~\ref{app:proof:thm:conv}.]
\it
The main idea is as follows.
First,
we show  that
(i)
if step~5 is performed,
$F_{\tau}(\ten{X}_{t + 1})
\le  F_{\tau}(\ten{X}_t)  -  \frac{\eta}{2}\NM{\ten{X}_{t + 1}  -  \bar{\ten{X}}_t}{F}^2$;
(ii)
if step~7 is performed, we have
$F_{\tau}(\ten{X}_{t + 1})
\le F_{\tau}(\ten{X}_t)
- \frac{\eta}{2}\NM{\ten{X}_{t + 1} - \ten{X}_t}{F}^2$.
Combining the above two conditions,
we obtain
\begin{align*}
	\frac{2}{\eta} ( F_{\tau}(\ten{X}_1) - F_{\tau}(\ten{X}_{T + 1}) )
	\ge 
	\sum\nolimits_{j \in \chi_1(T)} 
	\NM{\ten{X}_{t + 1} - \bar{\ten{X}}_t}{F}^2
	+  \sum\nolimits_{j \in \chi_2(T)} \NM{\ten{X}_{t + 1} - \ten{X}_t}{F}^2,
\end{align*}
where $\chi_1(T)$ and $\chi_2(T)$ are partitions of $\{ 1, \dots, T \}$
such that when $j \in \chi_1(T)$ step~5 is performed,
and when $j \in \chi_2(T)$ step~7 is performed.
Finally,
when $T \rightarrow \infty$,
we discuss three cases:
(i) $\chi_1(\infty)$ is finite, $\chi_2(\infty)$ is infinite;
(ii) $\chi_1(\infty)$ is infinite, $\chi_2(\infty)$ is finite; and
(iii) both $\chi_1(\infty)$ and $\chi_2(\infty)$ are infinite.
Let $\tilde{\ten{X}}$ be a limit point of $\{ \ten{X}_t \}$,
and $\{ \ten{X}_{j_t} \}$ be a subsequence that converges to $\tilde{\ten{X}}$.
In all three cases,
we show that
\begin{align*}
\lim\limits_{j_t \rightarrow \infty} \NM{\ten{X}_{j_t + 1} - \ten{X}_{j_t}}{F}^2
=
\| \Px{\frac{\bar{g}_{\tau}}{\tau}}{\tilde{\ten{X}} - \frac{1}{\tau} \nabla f(\tilde{\ten{X}})} - \tilde{\ten{X}} \|_F^2 = 0.
\end{align*}
Thus,
we must have 
$\tilde{\ten{X}}$ is also a critical point 
based on Proposition~\ref{pr:criti}.
It is easy to see that
we have not made any specifications
on  the limit points.
Thus, all limit points are also critical points.
\end{proof}

Recall that 
$\ten{X}_{t + 1}$ is generated from $\ten{V}_t$ 
in steps~\ref{alg:px:v}-\ref{alg:px:x}
and 
$\ten{X}_{t + 1} = \ten{V}_t$ indicates
convergence to a critical point (Proposition~\ref{pr:criti}).
Thus, we can measure convergence of Algorithm~\ref{alg:nort} by 
$\NM{\ten{X}_{t + 1}  -  \ten{V}_t}{F}$.
Corollary~\ref{cor:rate}
shows that a rate of $O(1 / T)$ can be obtained,
which is also the best possible rate for
first-order methods on general nonconvex problems~\citep{nesterov2013introductory,ghadimi2016accelerated}.

\begin{corollary}\label{cor:rate}
	$\min\nolimits_{t = 1, \dots, T}\frac{1}{2}\NM{\ten{X}_{t + 1}  -  \ten{V}_t}{F}^2 \le
	\frac{1}{\eta T }\big[  F_{\tau}(\ten{X}_1)  -  F^{\min}_{\tau} \big]$,
	where $\eta = \tau  -  \rho  -  DL$.
\end{corollary}

\begin{remark} \label{rmk:tau}
	A larger $\tau$ leads to a better approximation to the original problem $F$
	(Proposition~\ref{pr:bnd}).
	However, 
	it also make the stepsize $1/\tau$ smaller 
	(step~11 in Algorithm~\ref{alg:nort})
	and thus slower convergence (Corollary~\ref{cor:rate}).
\end{remark}

\subsubsection{With Kurdyka-{\L}ojasiewicz  Condition on 
Approximated Objective $F_{\tau}$}
\label{sec:conv:kl}


In Section~\ref{sec:conv:smooth},
we showed the
convergence results when $f$ is smooth and 
$g$ is of the form in \eqref{eq:lrphi}.
In this section,
we consider using
the Kurdyka-{\L}ojasiewicz (KL) condition~\citep{attouch2013convergence,bolte2014proximal}
on $F_{\tau}$,
which has been popularly used in nonconvex optimization, particularly
in gradient descent~\citep{attouch2013convergence} and proximal gradient algorithms~\citep{bolte2014proximal,li2015accelerated,Li2017ada}.
For example, 
the class of semi-algebraic functions satisfy the KL condition. 
More examples can be found in~\citep{bolte2010characterizations,bolte2014proximal}.

\begin{definition}\label{def:kl}
	A function $h$: $\mathbb{R}^n \rightarrow (-\infty, \infty]$ has the {\em uniformized KL property}
	if for every compact set $\mathcal{S} \in \text{dom}(h)$
	on which $h$ is a constant,
	there exist $\epsilon$, $c > 0$ such that for 
	all $\bm{u} \in \mathcal{S}$
	and all 
	$	\bar{\bm{u}} \in \{ \bm{u} : \min\nolimits_{\bm{v} \in \mathcal{S} } 
	\NM{\bm{u}  -  \bm{v}}{2}  \le  \epsilon \} 
	\cap
	\{ \bm{u} : f(\bar{\bm{u}})  <  f(\bm{u})  <  f(\bar{\bm{u}})  +  c \}$,
	one has 
	$\psi'
	\left( f(\bm{u}) - f(\bm{\bar{u}}) \right)
	\min\nolimits_{\bm{v} \in \partial f(\bm{u})}
	\NM{ \bm{v} }{2}  >  1$,
	where 
	$\psi(\alpha) = \frac{C \alpha^{x}}{x}$ for some $C > 0$,
	$\alpha \in [0, c)$ and $x \in (0, 1]$.
\end{definition}

Since the KL property~\citep{attouch2013convergence,bolte2014proximal}
does not require $h$ to be smooth or convex,
it thus allows convergence analysis under the nonconvex 
and nonsmooth setting.
However,
such a property cannot replace the smoothness assumption in Section~\ref{sec:conv:smooth},
as there are example functions which are smooth but fail to meet the KL condition
(Section~4.3 of~\citep{bolte2010characterizations}).

The following Theorem extends 
Algorithm~\ref{alg:nort}
to be used with the uniformized KL property.


\begin{theorem}\label{thm:klrate}
	Assume that $F_{\tau}$ in (\ref{eq:Ftau}) has the uniformized KL property, and
	let 
	$r_t = F_{\tau}(\ten{X}_t) - F_{\tau}^{\min}$.
	For a sufficiently large $t_0$, 
	\begin{itemize}[noitemsep,topsep=0pt,parsep=0pt,partopsep=0pt,leftmargin=15pt]
		\item[a)] If $x$ in Definition~\ref{def:kl} equals $1$,
		then $r_t = 0$ for all $t \ge t_0$; 
		
		\item[b)] If $x \in [\frac{1}{2}, 1)$, 
		$r_t 
		\le  
		( \frac{d_1 C^2}{1  +  d_1 C^2} )^{t  -  t_0} r_{t_0}$ where
		$d_1  =  2 (\tau  +  \rho)^2/\eta$;
		
		\item[c)] If
		$x  \in  (0, \frac{1}{2})$, 
		$r_t  \le  ( \frac{C}{(t - t_0)d_2(1 - 2 x)} )^{1/(1 - 2 x)}$
		where 
		$d_2 = \min
		\big( \frac{1}{2 d_1 C}, \frac{C}{1 - 2 x} (2^{\frac{2 x - 1}{2 x - 2}} - 1) r_{t_0} \big)$.
	\end{itemize}
\end{theorem}

\begin{proof}[Sketch, details are in Appendix~\ref{app:thm:klrate}.]
\it
The proof idea generally follows that for~\citep{bolte2014proximal} 
with a special treatment for $\ten{V}_t$ here.
First,
we show
\begin{align*}
\lim\limits_{t \rightarrow \infty}
\min\nolimits_{\ten{U}_t \in \partial F_{\tau}(\ten{X}_t)  } 
\NM{\ten{U}_t}{F} 
 \le  
\lim\limits_{t \rightarrow \infty}
(\tau  +  \rho) \NM{\ten{X}_{t + 1}  -  \ten{V}_t}{F}
 =  0.
\end{align*}
Next,
using the KL condition,
we have
\begin{align*}
1
\le 
\psi'\left( F_{\tau} (\ten{X}_{t + 1}) - F_{\tau}^{\min} \right) 
(\tau  +  \rho) \NM{\ten{X}_{t + 1}  -  \ten{V}_t}{F}.
\end{align*}
Then,
let $r_t = F_{\tau} (\ten{X}_t) - F_{\tau}^{\min}$.
From its definition,
we have
\begin{align*}
r_t - r_{t + 1} \ge F_{\tau} (\ten{V}_t) - F_{\tau} (\ten{X}_{t + 1}).
\end{align*}
Combining the above three inequalities,
we obtain
\begin{align*}
1 
\le 
\frac{2 (\tau + \rho)^2}{\eta}  \left[ \psi'( r_{t + 1} ) \right]^2 (r_t - r_{t + 1}).
\end{align*}
Since 
$\phi(\alpha) = \frac{C \alpha^{x}}{x} $, then $\phi'(\alpha) = C \alpha^{x - 1}$.
The above inequality becomes
$1 \le d_1 C^2 r_{t + 1}^{2 x - 2} (r_{t} - r_{t + 1})$,
where $d_1 = \frac{2 (\tau + \rho)^2}{\eta}$.
It is shown in~\citep{bolte2014proximal,li2015accelerated,Li2017ada} that
for the sequence $\{ r_t \}$ satisfying the above inequality, we have convergence to zero
with the 
different rates stated in the Theorem. 
\end{proof}

In Corollary~\ref{cor:rate} and Theorem~\ref{thm:klrate},
the convergence rates 
do not depend on $p$, and thus do not demonstrate the  effect of
momentum.
Empirically, 
the proposed  algorithm does have faster convergence
when momentum is used, and
will be shown in Section~\ref{sec:exps}.
This also agrees with previous studies in 
\citep{duchi2011adaptive,kingma2014adam,li2015accelerated,Li2017ada,yao2017efficient}.

\subsection{Statistical Guarantees}
\label{sec:stat}

Existing
statistical analysis on nonconvex regularization
has been 
studied in the context of sparse and low-rank matrix learning.
For example, the
SCAD~\citep{fan2001variable}, MCP~\citep{zhang2010nearly} and
capped-$\ell_1$~\citep{zhang2010analysis} penalties
have shown to be better than the convex $\ell_1$-regularizer
on sparse learning problems;
and SCAD, MCP and LSP have
shown to be better than the convex nuclear norm in
matrix completion~\citep{gui2016towards}.
However, these results cannot be extended 
to the tensor completion problem 
here as the nonconvex overlapped nuclear norm regularizer in 
\eqref{eq:compopt}
is not separable.
Statistical analysis on tensor completion has been studied with 
CP and Tucker decompositions~\citep{mu2014square}, 
tensor ring decomposition~\citep{huang2020provable},
convex overlapped nuclear norm
\citep{tomioka2011statistical},
and tensor nuclear norm~\citep{yuan2016tensor,cheng2016scalable}.
They show that tensor completion is possible 
under the incoherence condition when the number of observations is sufficiently large.
In comparison, in this section, 
we will (i) use 
the restricted strong convexity condition~\citep{agarwal2010fast,negahban2012unified})
instead of the incoherence condition, 
and (ii) study nonconvex overlapped nuclear norm regularization.

\subsubsection{Controlling the Spikiness and Rank}
\label{sec:obvmod}

In the following, 
we assume that elements in $\vect{\Omega}$ are drawn i.i.d. from the uniform distribution.
However,
when the sample size 
$\NM{\mathbf{\Omega}}{1} \ll I^{\pi}$,
tensor 
completion
is not always possible.
Take the special case of matrix completion as an example.
If $\ten{X}$
is an
almost-zero matrix
with only one element being $1$,
it cannot be 
recovered unless the nonzero element is observed.
However,
when 
$\ten{X}$
gets larger, 
there is a vanishing probability of observing the nonzero element,
and
so $\SO{\ten{X}} = \mathbf{0}$ with high probability~\citep{candes2009exact,negahban2012restricted}.

To exclude tensors
that are too ``spiky''
and allow tensor completion,
we 
introduce
\begin{align}
m_\text{spike}(\ten{X})
= \sqrt{I^{\pi}} \NM{\ten{X}}{\max} / \NM{\ten{X}}{F},
\label{eq:spiky}
\end{align}
which is an extension of the measure 
$\sqrt{I^1 I^2} \NM{\vect{X}}{\max} / \NM{\vect{X}}{F}$ 
in~\citep{negahban2012restricted,gu2014robust} for matrices.
Note that $m_\text{spike}(\ten{X})$ is invariant to the scale of $\ten{X}$ and $1 \le
m_\text{spike}(\ten{X}) \le \sqrt{I^{\pi}}$.
Moreover,
$m_\text{spike}(\ten{X}) = 1$ when all
elements in 
$\ten{X}$
have the same value
(least spiky);
and $m_\text{spike}(\ten{X}) = \sqrt{I^{\pi}}$
when 
$\ten{X}$ has
only one nonzero element 
(spikiest).
Similarly,
to measure how close is $\ten{X}$ to low-rank,
we use
\begin{align}
m_\text{rank}(\ten{X})
= \sum\nolimits_{i = 1}^D \alpha_i \NM{\ten{X}_{\ip{i}}}{*} / \NM{\ten{X}}{F},
\label{eq:mrank}
\end{align}
where $\alpha_i = \lambda_i / \sum_{d = 1}^D \lambda_d$'s 
are pre-defined constants depending on the penalty strength.
This is also extended from the measure 
$\NM{ \vect{X} }{*} / \NM{ \vect{X} }{F}$
in~\citep{negahban2012restricted,gu2014robust} on matrices.
Note
that $m_\text{rank}(\ten{X})$ 
$\le$ 
$\sum_{i = 1}^D$
$\alpha_i$
$ \sqrt{\text{rank}(\ten{X}_{\ip{i}})}$, with
equality holds 
when
all nonzero
singular values of $\ten{X}_{i}$'s 
are the same.
The target tensor $\ten{X}$ should thus have small
$m_\text{spike}(\ten{X})$ and $m_\text{rank}(\ten{X})$.
In \eqref{eq:pro}, 
assume for simplicity that
$D = M$ and $\lambda_i = \lambda$ for $i = 1, \dots, M$. We then have
the following constrained version of \eqref{eq:pro}:
\begin{align}
\min\nolimits_{\ten{X}}
\frac{1}{2}
\NM{ \SO{ \ten{X} - \ten{O} } }{F}^2
+ \lambda r(\ten{X})
\quad
\text{s.t.}
\quad
\NM{ \ten{X} }{\max} \le C,
\label{eq:stat}
\end{align}
where $r(\ten{X}) =  \sum\nolimits_{i = 1}^D  \phi( \ten{X}_{\ip{i}} )$ 
encourages 
	 $\ten{X}$
to be low-rank
(i.e., small $m_\text{rank}$),
and the constraint
on 	 $\NM{ \ten{X} }{\max}$
avoids 
$\ten{X}$ to be
spiky
(i.e., small $m_\text{spike}$).

\subsubsection{Restricted Strong Convexity (RSC)}
\label{ssec:rsc}

Following~\citep{tomioka2011statistical,negahban2012restricted,loh2015regularized,zhu2018global},
we introduce the
restricted strong convexity (RSC) condition.

\begin{definition}\label{ass:rsc}
(Restricted strong convexity (RSC) condition~\citep{agarwal2010fast})
Let $\Delta$ be an arbitrary 
$M$-order tensor. It satisfies the
RSC condition if
there exist constants
$\alpha_1, \alpha_2 > 0$ and $\tau_1, \tau_2 \ge 0$
such that
\begin{align}
\NM{ \SO{ \Delta } }{F}^2
\ge 
\begin{cases}
\alpha_1 
\NM{\Delta}{F}^2 - \tau_1 \frac{\log I^{\pi}}{\| \vect{\Omega} \|_1} 
\big(  \sum_{i = 1}^M \NM{\Delta_{\ip{i}}}{*} \big)^2
& 
\text{if } \NM{\Delta}{F} \le 1
\\
\alpha_2
\NM{\Delta}{F}^2 - \tau_2 \sqrt{\frac{\log I^{\pi}}{\| \vect{\Omega} \|_1}} 
\big(  \sum_{i = 1}^M \NM{\Delta_{\ip{i}}}{*} \big)
& 
\text{otherwise}
\end{cases}.
\label{eq:stat1}
\end{align}
\end{definition}
Let $d_i = \frac{1}{2} (I_i + \frac{I^{\pi}}{I_i})$ for $i = 1, \dots, M$.
Consider the set of tensors
parameterized by $n, \gamma \ge 0$:
\begin{align*}
\tilde{\mathcal{C}}
(n, \gamma)
= 
\left\lbrace 
\ten{X} \in \R^{I_1 \times \dots \times I_M},
\ten{X} \not= 0
\;|\;
m_\text{spike}(\ten{X})
\cdot 
m_\text{rank} ( \ten{X} ) 
\le
\frac{1}{\gamma L}
\min_{i = 1, \dots, M}
\sqrt{\frac{ n }{ d_i \log d_i }}
\right\rbrace,
\end{align*}
where $L$ is a positive constant.
The following Lemma shows that the RSC condition holds when the low-rank tensor is not too spiky.
If the RSC condition does not hold,
the tensor can be too hard to be recovered.


\begin{lemma} \label{lem:rsc:hold}
There exists $c_0$, $c_1$, $c_2$, $c_3 \ge 0$ such that 
$\forall \Delta \in \tilde{\mathcal{C}} (\NM{ \vect{\Omega} }{1}, c_0)$,
where $\NM{\vect{\Omega}}{1} > c_3 \underset{i = 1, \dots, M}{\max} (d_i \log
d_i)$, we have
\begin{align}
\frac{ \NM{ \SO{\Delta} }{F} }{ \NM{\vect{\Omega}}{1} }
\ge
\frac{1}{8}
\NM{ \Delta }{F}
\left\lbrace 
1 - \frac{ 128 L \cdot m_\text{spike}(\Delta) }{\sqrt{ \NM{ \vect{\Omega} }{1} }}
\right\rbrace,
\label{eq:rschold}
\end{align}
with a high probability of at least 
$1 - \max_{i = 1, \dots, M} c_1 \exp( - c_2 d_i \log d_i)$.
\end{lemma}

Another 
condition commonly used in
low-rank matrix/tensor learning is incoherence~\citep{candes2009exact,mu2014square,yuan2016tensor},
which prevents information of the row/column
spaces of the matrix/tensor from being too concentrated in a few rows/columns.
However,
as discussed in~\citep{negahban2012restricted},
the RSC condition is less restrictive than the incoherence condition,
and can better describe ``spikiness''
(details are in Appendix~\ref{app:incoh}).
Thus, we adopt the RSC instead of the incoherence condition here.

\subsubsection{Main results}
\label{sec:mainrst}

Let $\ten{X}^* \in \R^{I_1 \times \dots\times I_M}$ be the
ground-truth
tensor,
and $\tilde{\ten{X}}$ be 
an estimate of $\ten{X}^*$ obtained as 
a critical point of \eqref{eq:stat}.
The following
bounds the distance between $\ten{X}^*$ and $\tilde{\ten{X}}$.

\begin{theorem}
\label{thm:stat}
Assume that $\kappa$ is differentiable,
and the RSC condition
holds with $3 \kappa_0 M / 4 < \alpha_1$.
Assume that there exists positive constant $R > 0$ such that $\sum_{i = 1}^M \NM{ \ten{X}_{\ip{i}} }{*} \le R$,
and $\lambda$ satisfies
\begin{align} 
\frac{4}{\kappa_0}
\max
\left( 
\max_i 
\left\|  \left[ \SO{ \ten{X}^* - \ten{O} } \right]_{\ip{i}} \right\|_{\infty}
,
\alpha_2 \sqrt{ \log I^{\pi} / \NM{\vect{\Omega}}{1} }
\right) 
\le
\lambda
\le
\frac{\alpha_2}{ 4 R \kappa_0 },
\label{eq:lambda}
\end{align}
where
$\| \vect{\Omega} \|_1 \ge \max\left( \tau_1^2, \tau_2^2 \right) \frac{16 R^2 \log (I^{\pi} )}{\alpha_2^2}$.
Then, 
\begin{align}
\big\| \ten{X}^* - \tilde{\ten{X}} \big\|_F
\le \frac{ \lambda \kappa_0 c_v }{a_v} \sum\nolimits_{i = 1}^M \sqrt{k_i},
\label{eq:guarantee}
\end{align}
where 
$a_v = \alpha_1 - \frac{3 M \kappa_0 }{ 4 }$,
$c_v = 1 - \frac{1}{2 M}$,
and $k_i$ is the rank of $\ten{X}^*_{\ip{i}}$.
\end{theorem}

\begin{proof}[Sketch, details are in Appendix~\ref{app:thm:stat}.]
\it
The general idea of this proof is inspired
from~\citep{loh2015regularized}.\footnote{Note, however, that
\cite{loh2015regularized}
use different mathematical tools 
as they consider sparse vectors with separable dimensions, while we consider overlapped tensor regularization
with coupled singular values.}
There are three main steps:
\begin{itemize}[leftmargin=*]
\item 
Let $\tilde{\ten{V}} = \tilde{\ten{X}} - \ten{X}^*$.
We prove by contradiction that $\| \tilde{\ten{V}} \|_F \le 1$.
Thus,
we only need to consider the 
first condition in \eqref{eq:stat1}.

\item 
Let $h_i(\ten{X}) \! = \! \phi( \ten{X}_{\ip{i}} )$.
From Assumption~\ref{ass:kappa},
we have that $h_i(\ten{X}) + \frac{\mu }{2} \NM{\ten{X}}{F}^2$ is convex.
Using this together with
the first condition in \eqref{eq:stat1},
we obtain
\begin{align*}
\left(
\alpha_1 - \frac{\mu M}{2}
\right) 
\| \tilde{\ten{V}} \|_F^2
&  \le  
\lambda \sum\nolimits_{i = 1}^M 
\big( 
h_i(\ten{X}^*)
-  h_i(\tilde{\ten{X}}) 
\big) + 
\frac{\lambda \kappa_0}{2} \sum\nolimits_{i = 1}^M \| \tilde{\ten{V}}_{\ip{i}} \|_*.
\end{align*}

\item Using the above inequality
and properties of $h_i$,
we obtain
\begin{align*}
a_v \| \tilde{\ten{V}} \|_F^2
\le \lambda
\sum\nolimits_{i = 1}^M
b_v h_i(\ten{X}^*) 
- c_v h_i(\tilde{\ten{X}}),
\end{align*}
where 
$a_v = \alpha_1 - \frac{3 M}{4} \kappa_0$,
$b_v = 1 + \frac{1}{2 M}$
and
$c_v = 1 - \frac{1}{2 M}$.
Finally,
using Lemma~\ref{lem:app5} in Appendix~\ref{app:auxlem}
on the above inequality,
we have
$	\| \tilde{\ten{V}} \|_F
\le 
\frac{ \lambda \kappa_0 c_v }{a_v}
\sum\nolimits_{i = 1}^M  \sqrt{k_i}$.
\end{itemize}
\vspace{-20px}
\end{proof}

Since 
$\NM{ \ten{X} }{\max} \le C$ in \eqref{eq:stat},
we have $\sum_{i = 1}^M \NM{ \ten{X}_{\ip{i}} }{*}
\le \sum_{i = 1}^M \sqrt{k_i (I_i + \frac{I^{\pi}}{I_i})} C$
(as
$\NM{\vect{X}}{*} \le \sqrt{k} \NM{\vect{X}}{F} \le \sqrt{m n k} \NM{\vect{X}}{\max}$
for a rank-$k$ matrix $\vect{X} \in \R^{m \times n}$).
Thus,
in Theorem~\ref{thm:stat}, 
we can 
take
$R = \sum_{i = 1}^M$
$\sqrt{k_i (I_i + I^{\pi} / I_i)} C$,
which
is finite and cannot be arbitrarily large.
While we do not have access to the ground-truth $\ten{X}^*$ in practice,
Theorem~\ref{thm:stat} shows that the
critical point $\tilde{\ten{X}}$ can be bounded by a
finite
distance
from 
$\ten{X}^*$,
which means that an arbitrary critical point may not be bad.
From (\ref{eq:guarantee}),
we can also see that
the error $\| \ten{X}^* - \tilde{\ten{X}} \|_F$ increases
with the tensor order
$M$ and rank $k_i$.
This is reasonable as tensors with higher orders or larger ranks are usually harder to
estimate.
Besides,
recall that $\kappa_0$ in Assumption~\ref{ass:kappa}
reflects how nonconvex the function $\kappa(\alpha)$ is; while
$\alpha_1$ in Definition~\ref{ass:rsc}
measures strong convexity.
Thus,
these two quantities play opposing roles in \eqref{eq:guarantee}.
Specifically,
a larger $\alpha_1$ leads to a larger $a_v$,
and subsequently smaller 
$\| \ten{X}^* - \tilde{\ten{X}} \|_F$; whereas a 
larger $\kappa_0$ leads to a larger $\frac{ \lambda \kappa_0 c_v }{a_v}$,
and subsequently larger 
$\| \ten{X}^* - \tilde{\ten{X}} \|_F$.

Finally,
note that the range for $\lambda$ is (\ref{eq:lambda}) can be empty,
which means 
there can be no $\lambda$ to ensure Theorem~\ref{thm:stat}.
To understand when this can happen,
consider the two extreme cases:
\begin{itemize}
\item[C1.]
There is no noise in the observations,
i.e.,
$\SO{\ten{X}^* - \ten{O}}$ 
$ = 0$:
In this case, 
(\ref{eq:lambda}) reduces to 
\begin{align*} 
	\frac{4 \alpha_2}{\kappa_0}
	\sqrt{ \log I^{\pi} / \NM{\vect{\Omega}}{1} }
	\le
	\lambda
	\le
	\frac{\alpha_2}{ 4 R \kappa_0 }.
\end{align*}
Thus,
such a $\lambda$ may not exist
when the number of observations $\NM{\bm{\Omega}}{1}$
is too small.

\item [C2.]
All elements are observed:
we
then 
have
$\NM{ \SO{ \Delta } }{F} = \NM{ \Delta }{F}$, and so
$\alpha_1 = \alpha_2 = 1$ and $\tau_1 = \tau_2 = 0$ in Definition~\ref{ass:rsc}.
Besides,
the noise is not too small,
which means
$\sqrt{ \log I^{\pi} / \NM{\vect{\Omega}}{1} } \le \max_i 
\|  \left[ \ten{X}^* - \ten{O} \right]_{\ip{i}} \|_{\infty}$.
Then, 
(\ref{eq:lambda}) reduces to 
\begin{align*} 
\frac{4}{\kappa_0}
\max_{i}
\left\|  \left[ \ten{X}^* - \ten{O} \right]_{\ip{i}} \right\|_{\infty}
\le
\lambda
\le
\frac{1}{ 4 R \kappa_0 }.
\end{align*}
Thus,
such a $\lambda$ may not exist when the noise is too high.

\end{itemize}
Overall, when $\lambda$ does not exist, it is likely that the tensor completion
problem is too hard to have good recovery performance.

On the other hand, there are cases that $\lambda$ always exists.
	For example, 
	when $\ten{O} = \ten{X}^* = 0$,
	we have $R = 0$.
The requirement on $\lambda$ is then $	\frac{4 \alpha_2}{\kappa_0} \sqrt{ \log
I^{\pi} / \NM{\vect{\Omega}}{1} } \le \lambda \le +\infty$, and such a $\lambda$ always exists.  

\subsubsection{Dependencies on Noise Level and Number of Observations }
\label{sec:depobv}

In this section,
we demonstrate how the noise level affects \eqref{eq:guarantee}.
We assume that the observations 
are contaminated by additive Gaussian noise,
i.e.,
\begin{align}
\ten{O}_{i_1 \dots i_M} =
\begin{cases}
\ten{X}^*_{i_1 \dots i_M} + \xi_{i_1 \dots i_M}
&
\text{if}
\quad
\vect{\Omega}_{i_1 \dots i_M} = 1
\\
0
&
\text{otherwise}
\end{cases},
\label{eq:obvtc}
\end{align}
where $\xi_{i_1 \dots i_M}$ is a random variable following the normal distribution
$\mathcal{N}(0, \sigma^2)$.
The effects of the noise level
$\sigma$
and number of observations
in $\vect{\Omega}$
are shown in Corollaries~\ref{cor:noisy}
and~\ref{cor:number}, respectively, which
can be derived from Theorem~\ref{thm:stat}.

\begin{corollary}
\label{cor:noisy}
Let $\ten{E} = \ten{O} - \ten{X}^*$
and
$\lambda = b_1 \max_i \| \left[ \SO{\ten{E}} \right]_{\ip{i}} \|_{\infty}$.
When $\NM{\vect{\Omega}}{1}$ is sufficiently large
and
$b_1 \in [  \frac{ 4 }{ \kappa_0 }, \frac{ \alpha_2 }{ 4 R \kappa_0 \max_i \|  \left[ \SO{\ten{E}} \right]_{\ip{i}} \|_{\infty} } ] $
(to ensure $\lambda$ satisfies \eqref{eq:lambda}),
then
$\mathbb{E}
[ 
\| \ten{X}^* - \tilde{\ten{X}} \|_F
] 
\le 
\sigma
\frac{ \kappa_0 c_v \sqrt{I^{\pi}}}{a_v} 
\sum\nolimits_{i = 1}^M  \sqrt{k_i}$.
\end{corollary}

Corollary~\ref{cor:noisy} shows 
that 
the recovery error decreases
as the noise level $\sigma$ gets smaller,
and 
we can expect an exact recovery 
when $\sigma = 0$,
which is empirically verified in Section~\ref{sec:exp:obvnsy}.
When $\kappa(\alpha) = \alpha$,
$r(\ten{X})$ becomes the convex overlapping nuclear norm.
In this case,
Theorem~2 in~\citep{tomioka2011statistical}
shows that 
the recovery error can be bounded as
$\big\| \ten{X}^* - \tilde{\ten{X}} \big\|_F \le O(\sigma \sum\nolimits_{i = 1}^M \sqrt{k_i})$.
Thus,
Corollary~\ref{cor:noisy} can be seen as an extension of
Theorem~2 in~\citep{tomioka2011statistical}
to the nonconvex case.

\begin{corollary}
\label{cor:number}
Let $\lambda = b_3 \sqrt{ \log I^{\pi} / \NM{\vect{\Omega}}{1} }$.
Suppose that the noise level $\sigma$ is sufficiently small
and
$b_3 \in 
\left[ 
4, 
\frac{  1}{ (4 R \sqrt{ \log I^{\pi} / \NM{\vect{\Omega}}{1} }) } 
\right]$
(to ensure $\lambda$ satisfies \eqref{eq:lambda}).
Then,
$\big\| \ten{X}^* \! - \! \tilde{\ten{X}} \big\|_F
\! \le \! 
\frac{ b_3 \kappa_0 c_v }{a_v} 
\sqrt{ \frac{ \log I^{\pi} }{ \NM{\vect{\Omega}}{1} } }
\sum\nolimits_{i = 1}^M \sqrt{k_i}$.
\end{corollary}

Corollary~\ref{cor:number} shows that
the recovery error decays as $\sqrt{\NM{\vect{\Omega}}{1}}$ gets larger.
Such a dependency on the number of observed elements is
the same as in matrix completion problems
with nonconvex regularization~\citep{gui2016towards}.
Corollary~\ref{cor:number}
can be seen as an extension of
Corollary~3.6 in~\citep{gui2016towards} to the tensor 
case.


\section{Extensions}
\label{sec:extend}

In this section,
we show how the proposed NORT algorithm in Section~\ref{sec:proalg}
can be extended for robust tensor completion (Section~\ref{sec:robustc}) and
tensor completion with graph Laplacian regularization (Section~\ref{sec:rtc}).

\subsection{Robust Tensor Completion}
\label{sec:robustc}

In 
tensor completion
applications such as
video recovery and shadow removal,
the observed data
often have outliers 
\citep{candes2011robust,lu2016tensor}.
Instead of using the square loss,
more robust losses 
like the $\ell_1$ loss~\citep{candes2011robust,lu2013online,gu2014robust} and capped-$\ell_1$ loss~\citep{jiang2015robust},
are preferred.

In the following, we assume that the loss is of the form
$\ell(a) = \kappa_{\ell}(|a|)$, where $\kappa_\ell$ 
is smooth and satisfies
Assumption~\ref{ass:kappa}.
The optimization problem then becomes
\begin{align}
\min\nolimits_{\ten{X}}
F_{\ell}(\ten{X}) =
\sum\nolimits_{\mathbf{\Omega}_{i_1 ... i_M} = 1}
\kappa_{\ell}\left( | \ten{X}_{i_1 \dots i_M} - \ten{O}_{i_1 \dots i_M} | \right) 
+ \sum\nolimits_{i = 1}^D \lambda_i \phi( \ten{X}_{\ip{i}} ).
\label{eq:robustpro}
\end{align}
Since $\kappa_\ell(|a|)$ is non-differentiable at $a = 0$,
Algorithm~\ref{alg:nort} cannot be directly used.
Motivated by smoothing the $\ell_1$ loss with the Huber loss~\citep{huber1992robust} and 
the difference-of-convex decomposition of $\kappa_\ell$~\citep{tao2005dc,yao2018efficient},
we propose to smoothly approximate $\kappa_\ell(|a|)$ by
\begin{align}
\tilde{\kappa}_{\ell}(|a|; \delta)
=  \kappa_0 \cdot \tilde{\ell}(|a|; \delta) + 
\Big(
\kappa_{\ell}(|a|) - \kappa_0 \cdot |a|
\Big),
\label{eq:smtl1}
\end{align}
where $\kappa_0$ is in Assumption~\ref{ass:kappa},
$\delta$ is a smoothing parameter,
and $\tilde{\ell}$ is the Huber loss~\citep{huber1992robust}:
\begin{align*}
\tilde{\ell}(a; \delta) =
\begin{cases}
|a| &  
|a| \ge \delta
\\
\frac{1}{2 \delta}a^2 + \frac{1}{2} \delta 
& 
\text{otherwise}
\end{cases}.
\end{align*}
The following
Proposition
shows that 
$\tilde{\kappa}_{\ell}$ is 
smooth, and
a small $\delta$ ensures 
that it is 
a close approximation to $\kappa_\ell$.
\begin{proposition}\label{pr:smtl1}
$\tilde{\kappa}_{\ell}(|a|; \delta)$ is differentiable
and $\lim_{\delta \rightarrow 0} \tilde{\kappa}_{\ell}(|a|; \delta) = \kappa_\ell(|a|)$.
\end{proposition}

Problem~\eqref{eq:robustpro} is then transformed to
\begin{align}
\min\nolimits_{\ten{X}}
\sum\nolimits_{\mathbf{\Omega}_{i_1 ... i_M} = 1}
\tilde{\kappa}_{\ell}(| \ten{X}_{i_1 ... i_M} - \ten{O}_{i_1... i_M} |; \delta)
+ \sum\nolimits_{i = 1}^D \lambda_i \phi( \ten{X}_{\ip{i}} ).
\label{eq:smoothpro}
\end{align}
In Algorithm~\ref{alg:snort}, we gradually reduce the 
smoothing factor
in step~3, and 
use Algorithm~\ref{alg:nort} to solve the smoothed problem~\eqref{eq:smoothpro}
in each iteration.

\begin{algorithm}[H]
	\caption{Smoothing NORT for \eqref{eq:robustpro}.}
	\label{alg:snort}
	\small
	\begin{algorithmic}[1] 
		\STATE \textbf{Initialize} $\delta_0 \in (0, 1)$ and $s = 1$;
		\WHILE{not converged}
		\STATE transform to problem \eqref{eq:smoothpro} 
		with $\tilde{\kappa}_{\ell}$  using
		$\delta = (\delta_0)^s$;
		\STATE obtain $\ten{X}_s$ by solving the smoothed objective with Algorithm~\ref{alg:nort};
		\STATE $s = s + 1$;
		\ENDWHILE
		\RETURN $\ten{X}_s$. 
	\end{algorithmic}
\end{algorithm}

Convergence of Algorithm~\ref{alg:snort} is ensured in Theorem~\ref{thm:snort}.
However, the statistical guarantee in Section~\ref{sec:stat} does not hold 
as the robust loss is not smooth.

\begin{theorem} 
\label{thm:snort}
The sequence $\{ \ten{X}_s \}$ generated from Algorithm~\ref{alg:snort} has at least one limit point,
and all limits points are critical points of 
$F_{\ell \tau}(\ten{X}) =
\sum\nolimits_{\mathbf{\Omega}_{i_1 ... i_M} = 1} \kappa_{\ell}\left( | \ten{X}_{i_1 \dots i_M} - \ten{O}_{i_1 \dots i_M} | \right)  + \bar{g}_{\tau}(\ten{X})$.
\end{theorem}

\subsection{Tensor Completion with Graph Laplacian Regularization}
\label{sec:rtc}

The graph Laplacian regularizer is often used in  tensor
completion~\citep{narita2012tensor,song2017tensor}.
For example, 
in Section~\ref{sec:exp:sptemp}, 
we will consider
an application in spatial-temporal analysis~\citep{bahadori2014fast}, namely,
climate prediction 
based on meteorological records.
The spatial-temporal data is represented by a 3-order tensor $\ten{O}\in\R^{I^1\times I^2 \times I^3}$, 
where $I^1$ is the number of locations, 
$I^2$ is the number of time stamps, 
and $I^3$ is the number of variables corresponding to climate observations
(such as temperature and precipitation).
Usually, 
observations are only available at
a few stations, and
slices in $\ten{O}$ 
corresponding to 
the unobserved locations
are missing. 
Learning these entries can then be formulated as a tensor completion problem. 
To allow generalization to the unobserved locations, correlations among locations
have to be leveraged. 
This can be achieved by using
the graph Laplacian regularizer~\citep{belkin2006manifold}
on 
a graph $G$
with nodes being 
the locations
\citep{bahadori2014fast}.
Let the affinity matrix of $G$ be $\bm{A}\in\R^{m\times m}$, 
and the corresponding graph Laplacian matrix be $\bm{G}=\bm{D}-\bm{A}$, where $D_{ii}=\sum_j A_{ij}$.
As the 
spatial locations are stored along
the tensor's
first dimension,
the graph Laplacian regularizer is
defined as $h(\ten{X}_{\ip{1}}) = \text{Tr}(\ten{X}_{\ip{1}}^\top\bm{G}\ten{X}_{\ip{1}})$, 
which encourages nearby stations to have similar observations. 
When $\bm{G} = \bm{I}$,
it reduces to the  commonly used
Frobenius-norm regularizer 
$\NM{\ten{X}}{F}^2$
\citep{hsieh2015pu}.
With regularizer $h(\ten{X}_{\ip{1}})$,
problem~\eqref{eq:pro} 
is then extended
to:
\begin{align}
\label{eq:pro_reg}
\min\nolimits_{\ten{X}}
\sum\nolimits_{\mathbf{\Omega}_{i_1 ... i_M} = 1}
\ell\left( \ten{X}_{i_1 \dots i_M}, \ten{O}_{i_1 \dots i_M} \right) 
+ \sum\nolimits_{i = 1}^D \lambda_i \, \phi( \ten{X}_{\ip{i}} )
+ \mu \, h(\ten{X}_{\ip{1}}),
\end{align}
where $\mu$ is a hyperparameter.

Using the PA algorithm, it can be easily seen that 
the updates in \eqref{eq:tsplr_1}-\eqref{eq:tsplr_3} for $\ten{X}_t$ and $\ten{Y}_{t}$ 
remain the same,
but that for $\ten{Z}_t$ becomes
\[ \ten{Z}_t = \ten{X}_t - \frac{1}{\tau} \xi( \ten{X}_t ) + \mu \nabla
\Tr{\ten{X}_{\ip{1}}^\top\bm{G}\ten{X}_{\ip{1}}}. \]
To maintain efficiency of NORT, the key is to exploit the low-rank structures.
Using \eqref{eq:lowrf},
 $\ten{Z}_t$ can be written as 
\begin{eqnarray}
\ten{Z}_t 
& = & 
\sum\nolimits_{i = 1}^D ( \bm{U}_{t}^i  (\bm{V}_{t}^i)^{\top} )^{\ip{i}}
- \frac{1}{\tau} \xi(\ten{X}_t) - \mu [ \bm{G}\ten{X}_{\ip{1}} ]^{\ip{1}}.
\label{eq:zt_cokriging}
\end{eqnarray}  
$\bm{G}\ten{X}_{\ip{1}}$ 
can also be rewritten in low-rank form as
\begin{align*}
\bm{G}\ten{X}_{\ip{1}}= (\bm{G}\bm{U}_{t}^1 )(\bm{V}_{t}^1)^{\top} + 
\bm{G}
\sum\nolimits_{j \neq 1} 
\big[ ( \bm{U}_{t}^j (\bm{V}_{t}^j)^{\top} )^{\ip{j}} \big]_{\ip{1}}.
\end{align*}
For matrix multiplications of the forms
$\bm{a}^{\top} ( \ten{Z}_t )_{\ip{i}} $ and
$( \ten{Z}_t )_{\ip{i}} \bm{b}$ involved in the SVD of the proximal
step, we have
\begin{align}
\!\!
\bm{a}^{\top} ( \ten{Z}_t )_{\ip{i}} 
\! = \!   
( \bm{a}^{\top}(\bm{I} \! - \! \mu\bm{G})  \bm{U}_{t}^i )  ( \bm{V}_{t}^i )^{\top}
\!\!\! + \! 
 \sum\nolimits_{j \neq i} 
\bm{a}^{\top} (\bm{I} \! - \! \mu\bm{G})
\big[ ( \bm{U}_{t}^j (\bm{V}_{t}^j)^{\top} )^{\ip{j}} \big]_{\ip{i}}
\!\! - \! \frac{1}{\tau} \bm{a}^{\top}[ \xi(\ten{X}_t) ]_{\ip{i}},
\!\!
\label{eq:tenztu_cokrging}
\end{align}
and
\begin{align}
( \ten{Z}_t )_{\ip{i}} \bm{b} 
=  (\bm{I} \! - \! \mu\bm{G}) \bm{U}_{t}^i 
\big[  (\bm{V}_{t}^i)^{\top} \bm{b} \big] 
+   
 (\bm{I}-\mu\bm{G}) 
\sum\nolimits_{j \neq i}
\big[ ( \bm{U}_{t}^j (\bm{V}_{t}^j)^{\top} )^{\ip{j}} \big]_{\ip{i}}
\bm{b}
- \frac{1}{\tau} [ \xi(\ten{X}_t) ]_{\ip{i}} \bm{b}.
\label{eq:tenztv_cokrging}
\end{align}
Thus, one can still leverage the efficient computational procedures 
in Proposition~\ref{pr:mulv} to compute
$\bm{\hat{a}}^{\top} [ ( \bm{U}_t^j (\bm{V}_t^j)^{\top} )^{\ip{j}} ]_{\ip{i}}$,
where 
$\bm{\hat{a}}^{\top} \!\!\! = \! \bm{a}^{\top} \! (\bm{I} \! - \! \mu\bm{G})$ in 
(\ref{eq:tenztu_cokrging}), and
$[  ( \bm{U}_t^j (\bm{V}_t^j)^{\top} )^{\ip{j}}  ]_{\ip{i}} \bm{b}$  in
(\ref{eq:tenztv_cokrging}).

By taking $f (\ten{X}) = \sum\nolimits_{\mathbf{\Omega}_{i_1 ... i_M} = 1}
\ell\left( \ten{X}_{i_1 \dots i_M}, \ten{O}_{i_1 \dots i_M} \right) + \mu \,
h(\ten{X}_{\ip{1}})$, it is easy to see that the statistical analysis in Section~\ref{sec:convana}
	and convergence analysis in Section~\ref{sec:stat} still hold.


\section{Experiments} 
\label{sec:exps}

In this section, experiments are performed on both synthetic
(Section~\ref{sec:syn}) and real-world data sets (Sections~\ref{sec:expreal}-\ref{sec:exp:sptemp}),
using
a PC with Intel-i9 CPU and 32GB memory.
To reduce statistical variation, all results are averaged over five repetitions.

\subsection{Synthetic Data} 
\label{sec:syn}

We follow
the setup in~\citep{song2017tensor}. 
First,
we generate a 3-order tensor 
(i.e., $M  =  3$)
$\bar{\ten{O}}  =  \sum_{i = 1}^{r_g}  s_i 
(  \bm{a}_i  \circ  \bm{b}_i^{}  \circ  \bm{c}_i )$,
where $\bm{a}_i  \in  \R^{I^1}$, 
$\bm{b}_i  \in  \R^{I^2}$ and $\bm{c}_i  \in  \R^{I^3}$, 
$\circ$ denotes the outer product (i.e., 
$[\bm{a} \circ \bm{b} \circ \bm{c}]_{ijk} = a_i b_j c_k$). 
$r_g$ 
denotes the ground-truth rank and is
set to 5, with all $k_i$'s
equal to
$r_g=5$.
All elements in $\bm{a}_i$'s, $\bm{b}_i$'s, $\bm{c}_i$'s and $s_i$'s are sampled independently from the standard normal
distribution. Each element of
$\bar{\ten{O}}$
is then corrupted 
by noise from $\mathcal{N}(0, 0.01^2)$
to form $\ten{O}$.
A total of $\| \vect{\Omega} \|_1 = \frac{I^3
}{r_g}
\sum_{i = 1}^3 I^i \log(I^{\pi})$
random elements are observed from $\ten{O}$.
We use $50\%$ of them for training, 
and the remaining $50\%$ for validation.
Testing is evaluated on the unobserved elements in $\ten{\bar{O}}$.

We use
the square loss and 
three nonconvex penalties:
capped-$\ell_1$~\citep{zhang2010nearly},
LSP~\citep{candes2008enhancing} and TNN~\citep{hu2013fast}.
The following methods are compared:
\begin{itemize}[leftmargin=*,parsep=0pt,partopsep=0pt]
\item PA-APG~\citep{yu2013better},
which solves the convex overlapped nuclear norm minimization problem; 

\item GDPAN~\citep{zhong2014gradient},
which directly applies the PA algorithm to \eqref{eq:pro} as described in \eqref{eq:tsplr_1}-\eqref{eq:tsplr_3}; 

\item LRTC~\citep{Chen2020ANL},
which uses ADMM~\citep{boyd2011distributed} on \eqref{eq:pro}
as described in \eqref{eq:admm1}-\eqref{eq:admm3}; and

\item The proposed NORT algorithm (Algorithm~\ref{alg:nort}),
and its slower variant without adaptive momentum (denoted ``sNORT'').
Recall from Corollary~\ref{cor:rate}
that
$\tau$ has to be larger than $\rho + D \kappa_0$.
However, 
a large $\tau$ leads to slow convergence
(Remark~\ref{rmk:tau}).
Hence, 
we set $\tau = 1.01(\rho + D \kappa_0)$.
Moreover, 
as in~\citep{Li2017ada},
we set $\gamma_1 = 0.1$ and $p = 0.5$ 
in Algorithm~\ref{alg:nort}.
\end{itemize}

All algorithms are 
implemented in Matlab, with sparse tensor and matrix
operations performed via Mex files in C.  
All hypeprparamters (including the  $\lambda_i$'s in (\ref{eq:pro}) and hyperparameter in 
the baselines)
are tuned by grid search using the validation set. 
We early stop training if the 
relative change of objective in consecutive iterations is smaller than $10^{-4}$
or reaching the maximum of $2000$ iterations.

\subsubsection{Recovery Performance Comparison}
\label{ssec:recperf}

In this experiment,
we set 
$I^1 = I^2 = I^3 = \hat{c}$, where $\hat{c} =  200$ and $400$.
Following~\citep{lu2016nonconvex,yao2017efficient,yao2018large},
performance is evaluated by the (i) root-mean-square-error on the 
unobserved elements:
$\text{RMSE}  = \NM{P_{\bar{\vect{\Omega}}}(\ten{X}  -  \bar{\ten{O}})}{F} / \NM{ \bar{\vect{\Omega}} }{1}^{0.5}$,
where $\ten{X}$ is the low-rank tensor recovered,
and $\bar{\vect{\Omega}}$ contains the unobserved elements in $\bar{\ten{O}}$;
(ii) CPU time; and (iii)
space, which is measured as the memory used by MATLAB when running each algorithm.

Results on RMSE and space
are shown in Table~\ref{tab:synperf:large}.
We can see that
the nonconvex
regularizers
(capped-$\ell_1$, LSP and TNN,
with methods GDPAN, LRTC, sNORT and NORT)
all yield almost the same RMSE, which
is much lower than that of using the convex nuclear norm regularizer in PA-APG. 
As for the space
required,
sNORT and NORT take orders of magnitude smaller space
than the others. 
LRTC takes the largest space due to the use of multiple auxiliary and dual variables. 
Convergence of the optimization objective is shown in 
Figure~\ref{fig:syn3}.
As can be seen, NORT is the fastest, followed by sNORT and GDPAN, while LRTC is the slowest.
These demonstrate the benefits of avoiding repeated tensor folding/unfolding
operations and faster convergence of the proximal average algorithm.

\begin{table}[ht]
	\centering
	\caption{Testing RMSE and space required for the synthetic data.}
	\label{tab:synperf:large}
	\small
	\begin{tabular}{c| c | c | c | c| c} \toprule
		\multicolumn{2}{c|}{} & \multicolumn{2}{c|}{$\hat{c}=200$ (sparsity:$4.77\%$)} &
		\multicolumn{2}{c}{$\hat{c}=400$ (sparsity:$2.70\%$) } \\ \cline{3-6}
		\multicolumn{2}{c|}{} & RMSE                       & space (MB)              & RMSE                       & space (MB)             \\ \midrule
		convex                 & PA-APG & 0.0110$\pm$0.0007          & 600.8$\pm$70.4          & 0.0098$\pm$0.0001          & 4804.5$\pm$598.2       \\ \midrule
		& GDPAN  & \textbf{0.0010$\pm$0.0001} & 423.1$\pm$11.4          & \textbf{0.0006$\pm$0.0001} & 3243.3$\pm$489.6       \\\cmidrule(r){2-6}
		nonconvex & LRTC   &\textbf{0.0010$\pm$0.0001}                            &  698.9$\pm$21.5                       &\textbf{0.0006$\pm$0.0001}                            & 5870.6$\pm$514.0                       \\\cmidrule(r){2-6} 
		(capped-$\ell_1$)     & sNORT  & \textbf{0.0010$\pm$0.0001} & \textbf{10.1$\pm$0.1}   & \textbf{0.0006$\pm$0.0001} & \textbf{44.6$\pm$0.3}  \\
		\cmidrule(r){2-6}           & NORT   & \textbf{0.0009$\pm$0.0001} & 14.4$\pm$0.1            & \textbf{0.0006$\pm$0.0001} & 66.3$\pm$0.6           \\ \midrule
		& GDPAN  & \textbf{0.0010$\pm$0.0001} & 426.9$\pm$9.7           & \textbf{0.0006$\pm$0.0001} & 3009.3$\pm$376.2       \\
		\cmidrule(r){2-6}
		nonconvex      & LRTC   &\textbf{0.0010$\pm$0.0001}                            &  714.0$\pm$24.1                                                  &\textbf{0.0006$\pm$0.0001}   & 5867.7$\pm$529.1                  \\
		\cmidrule(r){2-6}     	(LSP)      & sNORT  & \textbf{0.0010$\pm$0.0001} & 10.8$\pm$0.1            & \textbf{0.0006$\pm$0.0001} & \textbf{44.6$\pm$0.2}  \\
		\cmidrule(r){2-6}           & NORT   & \textbf{0.0010$\pm$0.0001} & 14.0$\pm$0.1            & \textbf{0.0006$\pm$0.0001} & 62.1$\pm$0.5           \\ \midrule
		& GDPAN  & \textbf{0.0010$\pm$0.0001} & 427.3$\pm$10.1          & \textbf{0.0006$\pm$0.0001} & 3009.2$\pm$412.2       \\
		\cmidrule(r){2-6}
		nonconvex   & LRTC   &\textbf{0.0010$\pm$0.0001}                            &     759.0$\pm$24.3                    &\textbf{0.0006$\pm$0.0001}                            &5865.5$\pm$519.3                    \\
		\cmidrule(r){2-6}       	(TNN)     & sNORT  & \textbf{0.0010$\pm$0.0001} & \textbf{10.2$\pm$0.1}   & \textbf{0.0006$\pm$0.0001} & \textbf{44.7$\pm$0.2}  \\
		\cmidrule(r){2-6}           & NORT   & \textbf{0.0010$\pm$0.0001} & 14.4$\pm$0.2            & \textbf{0.0006$\pm$0.0001} & 63.1$\pm$0.6           \\ \bottomrule
	\end{tabular}
\end{table}

\begin{figure}[ht]
	\centering
	\includegraphics[width=0.32\textwidth]{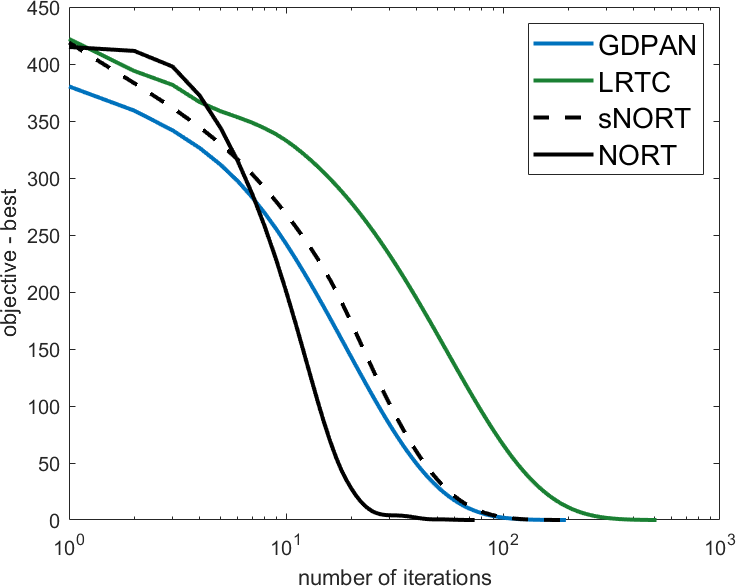}
	\includegraphics[width=0.32\textwidth]{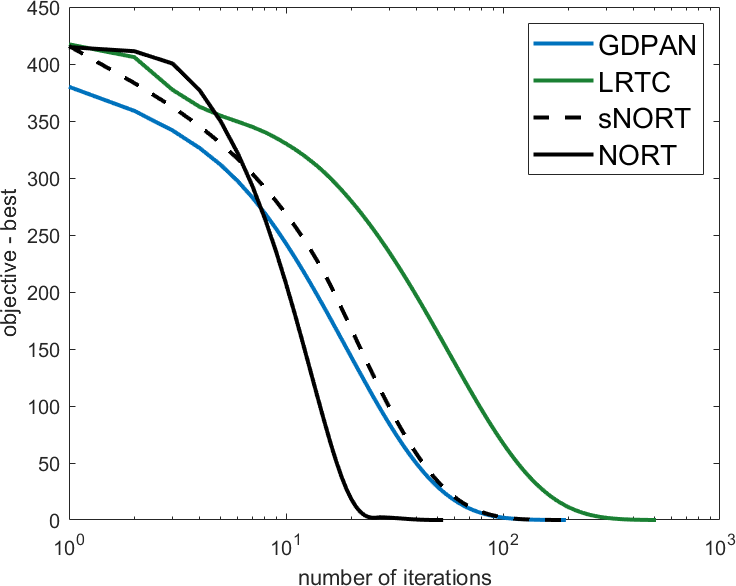}
	\includegraphics[width=0.32\textwidth]{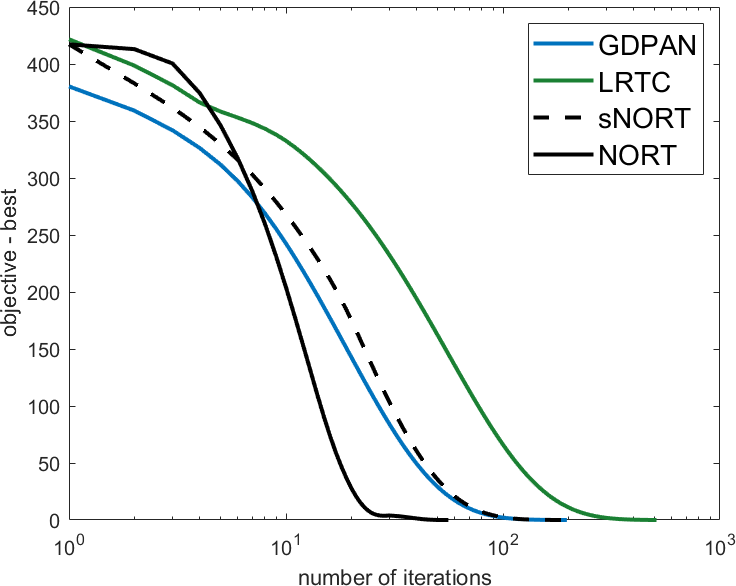}
	
	\subfigure[capped-$\ell_1$.]
	{\includegraphics[width=0.32\textwidth]{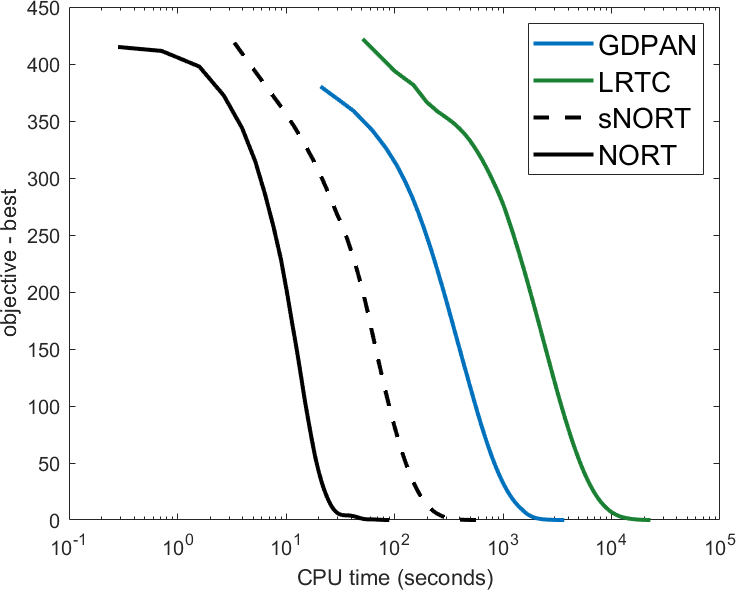}}
	\subfigure[LSP.]
	{\includegraphics[width=0.32\textwidth]{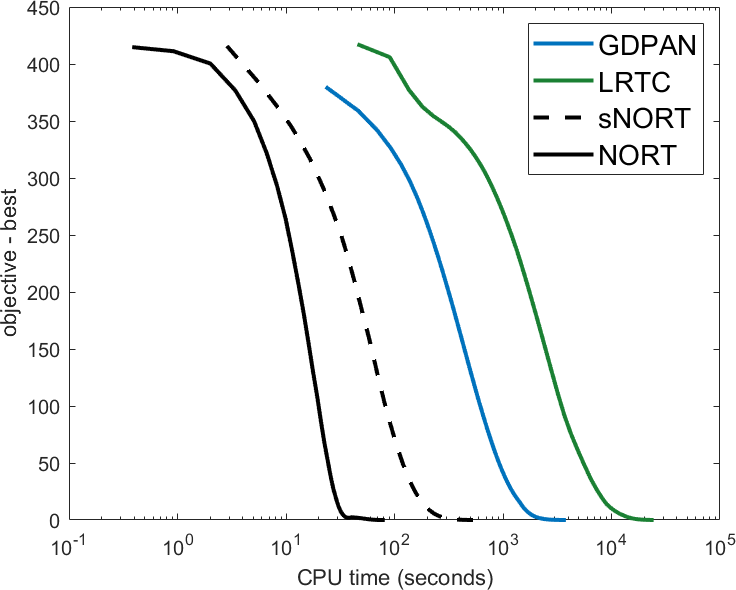}}
	\subfigure[TNN.]
	{\includegraphics[width=0.32\textwidth]{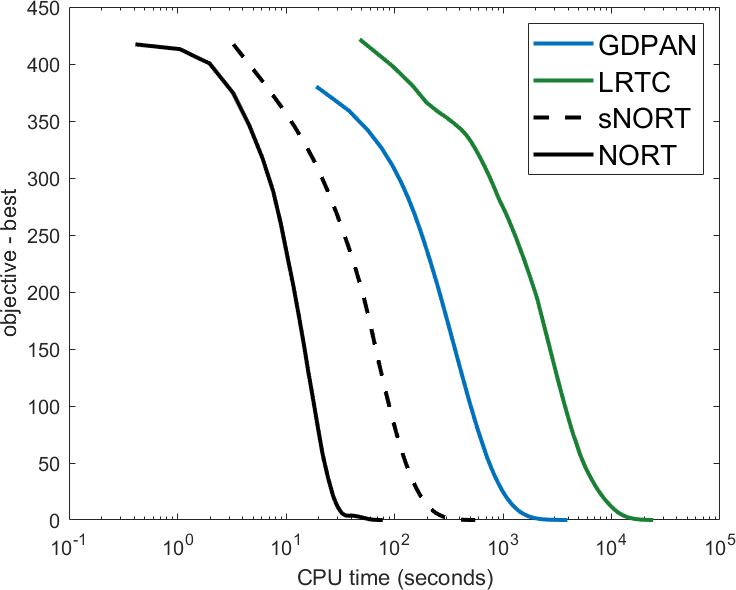}}
	
	\vspace{-10px}
\caption{Convergence of the objective versus number of iterations (top) and
		CPU time (bottom) on the synthetic data (with $\hat{c} = 400$).}
	\label{fig:syn3}
\end{figure}

\subsubsection{Ranks during Iteration }
\label{sec:exp:rank}

Unlike factorization methods
which explicitly constrain the iterate's rank,
in NORT (Algorithm~\ref{alg:nort}),  
this is  only 
implicitly controlled by the nonconvex regularizer.
As shown in Table~\ref{tab:overview},
having a large rank during the iteration may affect the
efficiency of
NORT.
Figure~\ref{fig:syn:rank}
shows the ranks of $(\ten{Z}_t)_{\ip{i}}$ and $\vect{X}_{t + 1}^i$
at step~11 of Algorithm~\ref{alg:nort}.
As can be seen,
the ranks of the iterates remain
small
compared with the tensor size ($\hat{c}=400$).
Moreover,
the ranks of  $\vect{X}_{t + 1}^1$,
$\vect{X}_{t + 1}^2$,
and $\vect{X}_{t + 1}^3$ all converge to the true rank
(i.e., $5$)
of the ground-truth tensor.

\begin{figure}[ht]
	\centering
	\subfigure[capped-$\ell_1$.]
	{\includegraphics[width=0.32\textwidth]{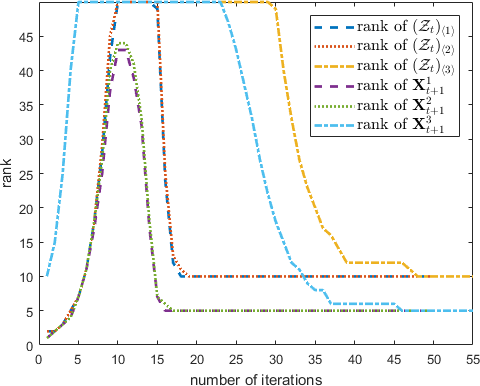}}
	\subfigure[LSP.]
	{\includegraphics[width=0.32\textwidth]{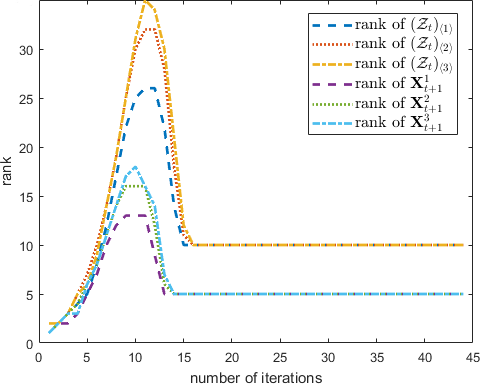}}
	\subfigure[TNN.]
	{\includegraphics[width=0.32\textwidth]{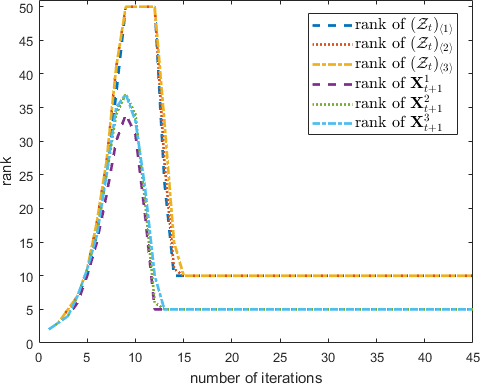}}

	\vspace{-10px}
	\caption{Ranks of $\{(\ten{Z}_t)_{\ip{i}}, \vect{X}_{t + 1}^i\}_{i=1,2,3}$
	versus number of iterations on synthetic data (with $\hat{c} = 400$).}
	\label{fig:syn:rank}
\end{figure}

\subsubsection{Quality of Critical Points}
\label{sec:quacri}

In this experiment, 
we empirically validate the statistical performance of critical points analysed in Theorem~\ref{thm:stat}.
Note that 
$\ten{X}_0$ and $\ten{X}_1$ are initialized as the zero tensor in Algorithm~\ref{alg:nort},
and $\ten{X}_t$ is implicitly stored by
a summation of $D$ factorized matrices in \eqref{eq:lowrf}.
We randomly  generate
$\ten{X}_0 = \ten{X}_1 = \sum\nolimits_{i = 1}^D \big( \vect{u}^i ( \vect{v}^i
)^{\top} \big)^{\ip{i}}$,
where elements in $\bm{u}^i$'s and $\bm{v}^i$'s follow $\mathcal{N}(0, 1)$.
The statistical error
is measured as the RMSE 
between $\ten{X}_t$ during iterating of
NORT (Algorithm~\ref{alg:nort})
and the
underlying ground-truth $\ten{X}^*$
(i.e., $\| \ten{X}_t - \ten{X}^* \|_F^2$),
while the optimization error is measured as the
RMSE between iterate $\ten{X}_t$ and the globally optimal solution $\dot{\ten{X}}$
of \eqref{eq:pro}
(i.e., $\| \ten{X}_t - \dot{\ten{X}} \|_F^2$).
We use the same experimental setup as in Section~\ref{ssec:recperf}.
As the exact $\dot{\ten{X}}$ is not known, it is approximated by the
$\tilde{\ten{X}}$ which obtains the lowest training objective value over 20
repetitions. 

Figure~\ref{fig:syn:staterr_vs_opterr} 
shows the statistical error versus
optimization error obtained by NORT with the 
(smooth) LSP regularizer
and 
(nonsmooth) capped-$\ell_1$ regularizer.
While
both the statistical and optimization errors 
decrease with more iterations, the
statistical error is generally larger 
than the optimization error
since
we may not have exact recovery
when noise is present.
Moreover,
the optimization errors for different runs terminate at different values,
indicating that NORT indeed converges to different local solutions. However,
all these have similar statistical errors,
which validates Theorem~\ref{thm:stat}. 
Finally, while the capped-$\ell_1$ regularizer
does not satisfy Assumption~\ref{ass:kappa} (which is required by
Theorem~\ref{thm:stat}),
Figure~\ref{fig:stat-l1}
still shows a similar pattern
as 
Figure~\ref{fig:stat-lsp}.
This helps
explain 
the good empirical performance obtained
by the capped-$\ell_1$ regularizer~\citep{jiang2015robust,lu2016nonconvex,yao2018large}.

\begin{figure}[ht]
\centering
\subfigure[LSP. \label{fig:stat-lsp}]{\includegraphics[width=0.32\textwidth]{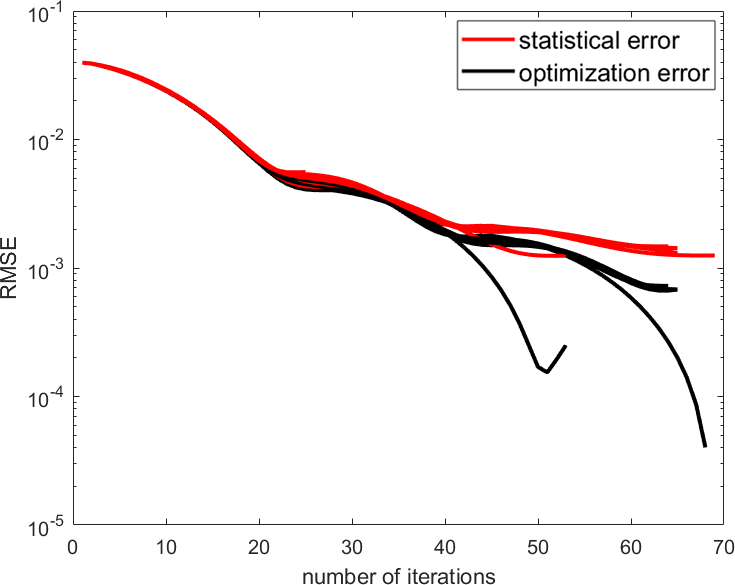}}
\qquad
\subfigure[capped-$\ell_1$. \label{fig:stat-l1}]{\includegraphics[width=0.32\textwidth]{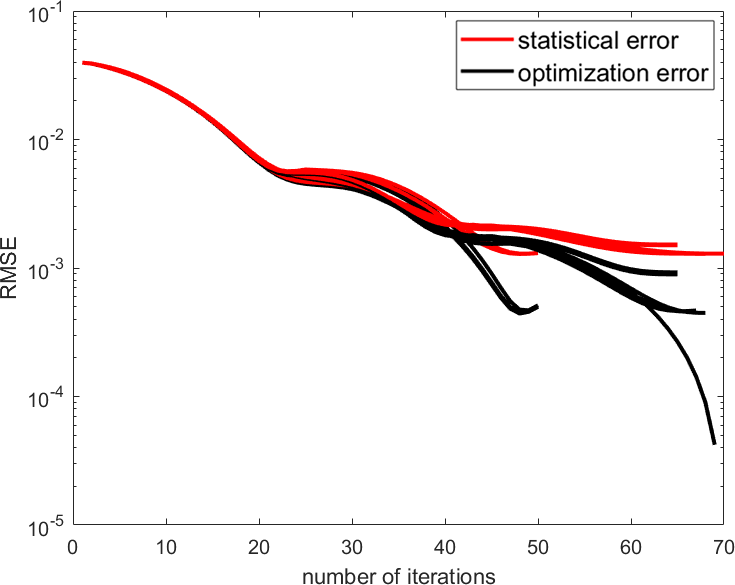}}
	
	\vspace{-10px}
\caption{Statistical error (red) and optimization error (black) versus the
number of NORT iterations (with $\hat{c} = 400$)
from $20$ runs of NORT (with different random seeds).}
	\label{fig:syn:staterr_vs_opterr}
\end{figure}

\subsubsection{Effects of Noise Level and Number of Observations }
\label{sec:exp:obvnsy}

In this section, we show 
the effects of 
noise level $\sigma$ 
and 
number of observed elements $\NM{\mathbf{\Omega}}{1}$ 
on the testing RMSE and training time.
We use the same experimental setup as in 
Section~\ref{ssec:recperf}.
Since PA-APG is much worse 
(see Table~\ref{tab:synperf:large})
while
LRTC and sNORT are slower than NORT
(see Figure~\ref{fig:syn3}),
we only use GDPAN as comparison baseline.

Figure~\ref{fig:rmse:noisy}  shows the testing RMSE with 
$\sigma$ at different
$\NM{\mathbf{\Omega}}{1}$'s (here, we
plot
$s = \NM{\vect{\Omega}}{1} / I^{\pi}$).
As can be seen,
the curves show a linear dependency on $\sigma$ when $\NM{\vect{\Omega}}{1}$ is
sufficiently large, which agrees with 
Corollary~\ref{cor:noisy}.
Figure~\ref{fig:rmse:obv} shows the testing RMSE versus  
$\sqrt{ \log I^{\pi} / \NM{\vect{\Omega}}{1} }$ 
at different $\sigma$'s.
As can be seen, there is a linear dependency 
when the noise level $\sigma$ is small, which agrees with 
Corollary~\ref{cor:number}. 
Finally, 
note that NORT and GDPAN obtain very similar testing RMSEs 
as both solve the same objective (but with different algorithms).

\begin{figure}[ht]
	\centering
	\subfigure[Different noise levels. \label{fig:rmse:noisy}]
	{\includegraphics[width=0.32\textwidth]{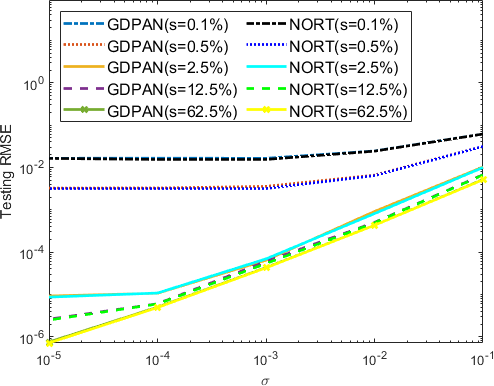}}
	\qquad
	\subfigure[Different numbers of observations. \label{fig:rmse:obv}]
	{\includegraphics[width=0.325\textwidth]{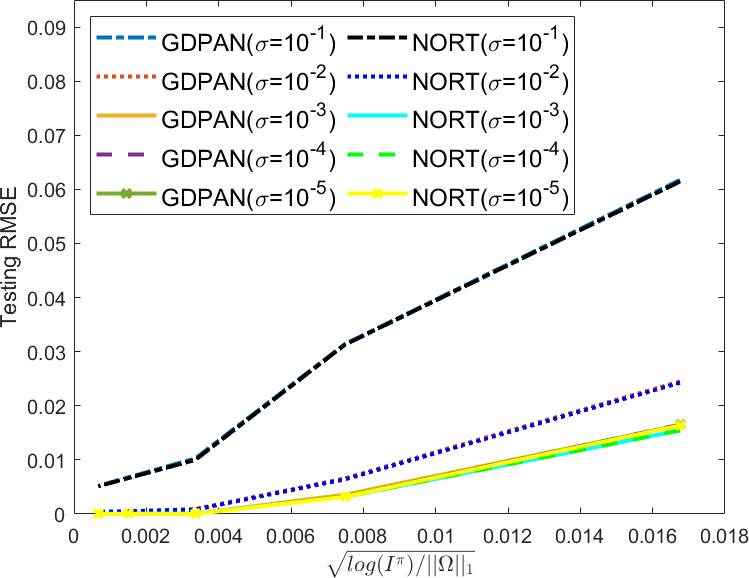}}
	
	\vspace{-10px}
	\caption{Effect of noise level and number of observations on the testing RMSE 
		on the synthetic data (with $\hat{c} = 400$). 
		Note that NORT and GDPAN obtain similar performance
		and their curves overlap with each other.}
\end{figure}

Figure~\ref{fig:noise}
shows
the effects of noise level on 
the convergence of testing RMSE versus (training) CPU time.
As can be seen,
testing RMSEs generally terminates at a higher level when
the noise level gets larger,
and NORT is much faster than GDPAN under all noise level.

\begin{figure}[H]
	\centering
	\subfigure[NORT. ]
	{\includegraphics[width=0.32\textwidth]{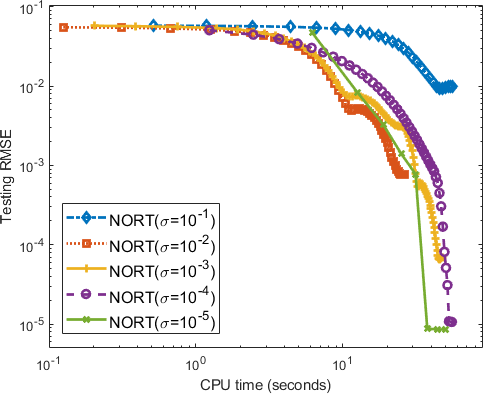}}
	\qquad
	\subfigure[GDPAN.]
	{\includegraphics[width=0.32\textwidth]{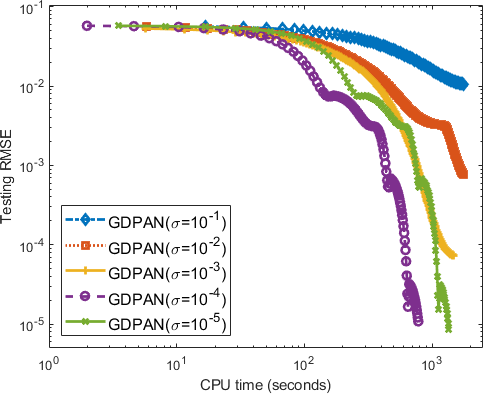}}
	
	\vspace{-10px}
\caption{Effects of the noise level on the convergence
		on synthetic data (with $\hat{c} = 400$, $s = 2.5\%$).}
	\label{fig:noise}
\end{figure}

Figure~\ref{fig:sparse} 
shows 
the effects of numbers of observations on 
the convergence of testing RMSE versus (training) CPU time.
First,
we can see that
NORT is much faster than GDPAN under 
various numbers of observations.
Then,
when $s$ gets smaller and the tensor completion problem is more ill-posed,
more iterations are needed by both
NORT and GDPAN,
which makes them take more time to converge.

\begin{figure}[ht]
	\centering
	\subfigure[NORT.]
	{\includegraphics[width=0.32\textwidth]{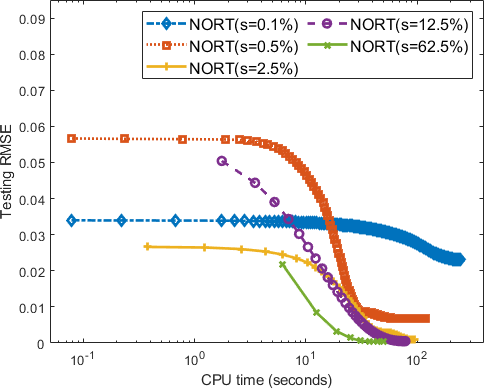}}
	\qquad
	\subfigure[GDPAN. ]
	{\includegraphics[width=0.32\textwidth]{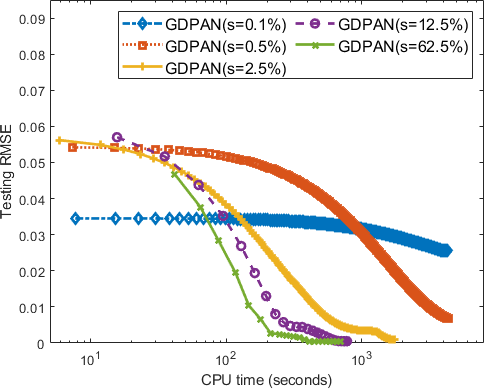}}
	
	\vspace{-10px}
	\caption{Effect of the number of observations on the convergence
		on synthetic data (with $\hat{c} = 400$, $\sigma = 10^{-2}$).}
	\label{fig:sparse}
\end{figure}

\subsubsection{Effects of Tensor Order and Rank}
\label{sec:orderrank}

In this experiment,
we use a similar experimental setup as in Section~\ref{ssec:recperf}, except that
the tensor order 
$M$ is varied from $2$ to $5$.  
As high-order tensors have large memory requirements,
while
we always set $I^1 = I^2 = I^3 = \hat{c} = 400$,
we set $I^4 = 5$ when $M = 4$
and $I^4 = I^5 = 5$ when $M = 5$.
Figure~\ref{fig:rmse:syn-order} shows the testing RMSE versus $M$. 
As can be seen,
the error grows almost linearly, which agrees with 
Theorem~\ref{thm:stat}.
Moreover, note that 
at $M=5$,
GDPAN runs out of memory  because 
it needs to maintain dense tensors in each iteration.

Figure~\ref{fig:rmse:syn-rank} shows the testing RMSE w.r.t.  
$\sqrt{r_g}$
(where $r_g$
is the ground-truth tensor rank).
As can be seen, the error grows linearly w.r.t. $\sqrt{r_g}$, which again agrees with 
Theorem~\ref{thm:stat}.

\begin{figure}[H]
	\centering
	
	\subfigure[Effect of tensor order $M$ ($r_g$ $=$ $5$).]
	{\includegraphics[width=0.315\textwidth]{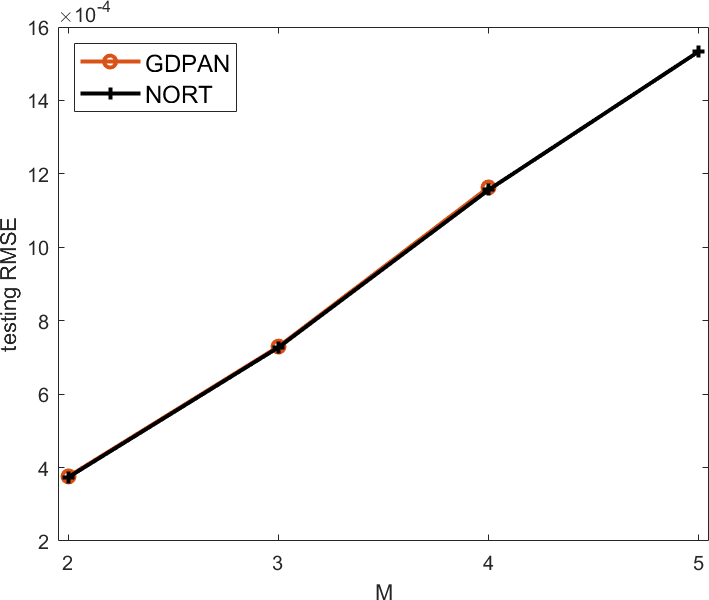}
		\label{fig:rmse:syn-order}}
	\qquad
	\subfigure[Effect of ground-truth rank $r_g$ at
	$M = 3$ (the corresponding
		$r_g$ values are $5$, $10$, $15$ and $20$).]
	{\includegraphics[width=0.335\textwidth]{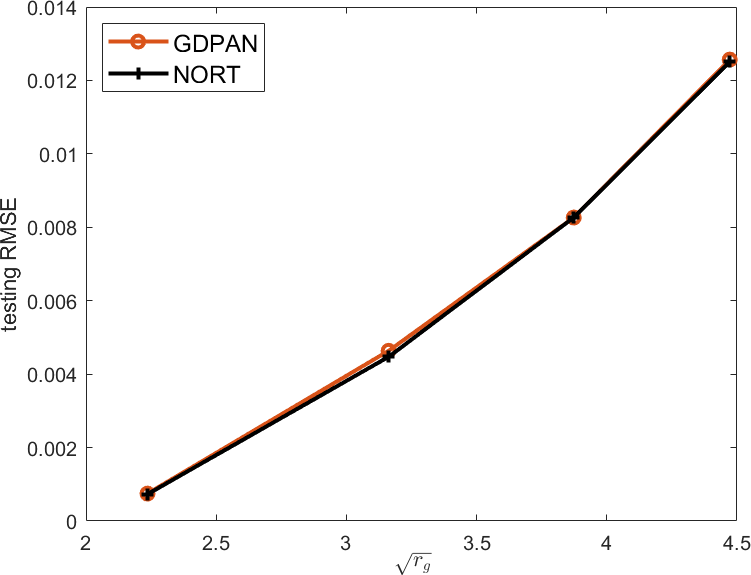}
		\label{fig:rmse:syn-rank}}
	\vspace{-10px}
	\caption{Effects of tensor order and ground-truth rank on the testing RMSE 
		on the synthetic data (with $\hat{c} = 400$). 
		Note that NORT and GDPAN obtain similar performance
		and their curves overlap with each other.}
	\label{fig:mk}
\end{figure}

Figure~\ref{fig:time:syn-order}
shows 
the convergence of testing RMSE versus (training) CPU time
at different tensor orders.
As can be seen,
while both GDPAN and NORT need more time to converge for
higher-order
tensors,
NORT is consistently faster than GDPAN.
Figure~\ref{fig:time:syn-rank}
shows the convergence of testing RMSE at
different ground-truth ranks.
As can be seen,
while NORT is still faster than GDPAN
at different ground-truth tensor ranks
($r_g$),
the relative speedup gets smaller when $r_g$ gets larger.
This is because
NORT needs to construct sparse tensors (e.g., Algorithm~\ref{alg:compp3})
before using them for multiplications, and 
empirically,
the
handling of sparse tensors 
requires more time on memory addressing
as the rank increases
\citep{bader2007efficient}.

%

\begin{figure}[ht]
	\centering
	\subfigure[Effect of tensor order $M$ ($r_g = 5$).]
	{\includegraphics[width=0.33\textwidth]{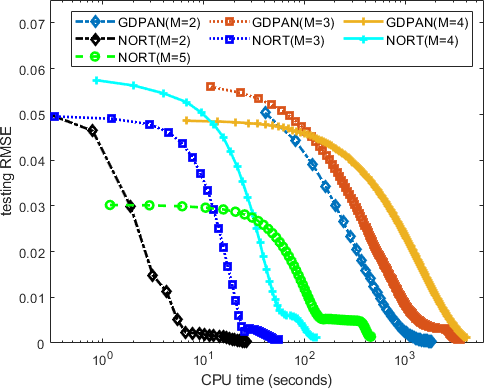}
		\label{fig:time:syn-order}}
	\qquad
	\subfigure[Effect of ground-truth rank ($M \! = \! 3$).]
	{\includegraphics[width=0.33\textwidth]{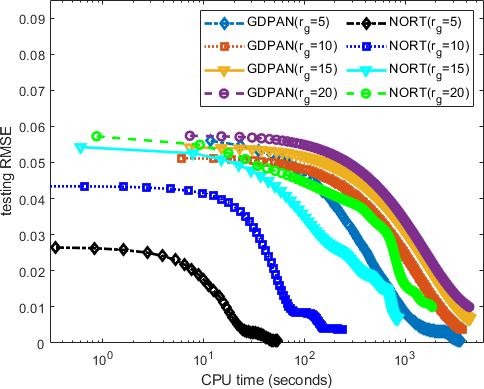}
		\label{fig:time:syn-rank}}
	\vspace{-10px}
	\caption{Effects of tensor order and ground-truth rank on the convergence 
		on the synthetic data (with $\hat{c} = 400$). 
		GDPAN runs out of memory
		when $M=5$.}
	\label{fig:mk:time}
\end{figure}

\subsection{Tensor Completion Applications}
\label{sec:expreal}

\begin{table}[t]
	\centering
	\caption{Algorithms compared on the real-world data sets.}
	\setlength\tabcolsep{3pt}
	\small
	\begin{tabular}{C{65px} | C{110px} | C{110px} | C{130px}}
		\toprule
		& algorithm                           & model                               & basic solver                                    \\ \midrule
		\multirow{11}{*}{convex}      & ADMM~\citep{tomioka2010estimation}    &                                     & ADMM                                           \\ \cmidrule(r){2-2} \cmidrule(r){4-4}
		& FaLRTC~\citep{liu2013tensor}        & overlapped  nuclear norm            & accelerated proximal algorithm on  dual problem \\ \cmidrule(r){2-2} \cmidrule(r){4-4}
		& PA-APG~\citep{yu2013better}         &                                     & accelerated PA algorithm                        \\ \cmidrule(r){2-4}
		& FFW~\citep{guo2017efficient}        & latent nuclear norm                 & efficient Frank-Wolfe algorithm                 \\ \cmidrule(r){2-4}
		& TR-MM~\citep{nimishakavi2018dual}   & squared latent nuclear norm         & Riemannian optimization  on  dual problem  \\ \cmidrule(r){2-4}
		& TenNN~\citep{zhang2017exact}        & tensor-SVD                          & ADMM                                            \\ \midrule
		\multirow{11}{*}{ factorization} & RP~\citep{kasai2016low}             & Turker decomposition                & Riemannian preconditioning                      \\ \cmidrule(r){2-4}
		& TMac~\citep{xu2013tmac}             & multiple matrices factorization     & alternative minimization                        \\ \cmidrule(r){2-4}
		& CP-WOPT~\citep{hong2020generalized}    & CP decomposition                    & nonlinear conjugate gradient                                \\ \cmidrule(r){2-4}
		& TMac-TT~\citep{bengua2017efficient} & tensor-train decomposition          & alternative minimization                        \\ \cmidrule(r){2-4}
		& TRLRF~\citep{yuan2019tensor}        & tensor-ring decomposition           & ADMM                                            \\ \midrule
		\multirow{6}{*}{ non-convex}   & GDPAN~\citep{zhong2014gradient}     &                                     & nonconvex PA algorithm                          \\ \cmidrule(r){2-2} \cmidrule(r){4-4}
		& LRTC~\citep{Chen2020ANL}            & nonconvex overlapped nuclear norm regularization & ADMM                                            \\ \cmidrule(r){2-2} \cmidrule(r){4-4}
		& NORT~(Algorithm~\ref{alg:nort})     &                                     & proposed algorithm                              \\ \bottomrule
	\end{tabular}
	\label{tab:comparedalgs}
\end{table}

In this section, we use the square loss. 
As different nonconvex regularizers have similar performance, we will only use
LSP in the sequel. 
The proposed NORT algorithm is compared with:\footnote{We used our own
implementations of LRTC, PA-APG and GDPAN as their codes are not publicly available.}
\begin{enumerate}
\item[(i)]
algorithms for various convex regularizers
including:
ADMM~\citep{boyd2011distributed}\footnote{\url{https://web.stanford.edu/~boyd/papers/admm/}},
PA-APG~\citep{yu2013better},
FaLRTC~\citep{liu2013tensor}\footnote{\url{https://github.com/andrewssobral/mctc4bmi/tree/master/algs_tc/LRTC}},
FFW~\citep{guo2017efficient}\footnote{\url{https://github.com/quanmingyao/FFWTensor}},
TR-MM~\citep{nimishakavi2018dual}\footnote{\url{https://github.com/madhavcsa/Low-Rank-Tensor-Completion}},
and 
TenNN~\citep{zhang2017exact}\footnote{\url{http://www.ece.tufts.edu/~shuchin/software.html}};
\item[(ii)] factorization-based algorithms
including:
RP~\citep{kasai2016low}\footnote{\url{https://bamdevmishra.in/codes/tensorcompletion/}},
TMac~\citep{xu2013tmac}\footnote{\url{http://www.math.ucla.edu/~wotaoyin/papers/tmac_tensor_recovery.html}},
CP-WOPT~\citep{hong2020generalized}\footnote{\url{https://www.sandia.gov/~tgkolda/TensorToolbox/}},
TMac-TT~\citep{bengua2017efficient}\footnote{\url{https://sites.google.com/site/jbengua/home/projects/efficient-tensor-completion-}
	\url{for-color-image-and-video-recovery-low-rank-tensor-train}}, and
TRLRF~\citep{yuan2019tensor}\footnote{\url{https://github.com/yuanlonghao/TRLRF}}; 
\item[(iii)] algorithms that can handle nonconvex regularizers
including
GDPAN~\citep{zhong2014gradient}
and
LRTC
\citep{Chen2020ANL}.
\end{enumerate}
More details are in 
	Table~\ref{tab:comparedalgs}.
We do not compare
with (i) sNORT,
as it has already been shown to be slower than NORT;
(ii) iterative hard thresholding~\citep{rauhut2017low}, as its code is not publicly available and the more
recent TMac-TT solves the same problem; (iii) the method in 
\citep{bahadori2014fast}, as it can only deal with cokriging and forecasting
problems. 

Unless otherwise specified,
performance is evaluated by (i) root-mean-squared-error on the 
unobserved elements: 
$\text{RMSE}  = \NM{P_{\vect{\Omega}^\bot}(\ten{X}  -  \ten{O})}{F} / \NM{ \vect{\Omega}^\bot }{1}^{0.5}$, 
where $\ten{X}$ is the low-rank tensor recovered,
and $\vect{\Omega}^\bot$ contains the unobserved elements in $\ten{O}$; 
and (ii) CPU time.

\subsubsection{Color Images}
\label{sec:image}

We use  the
\textit{Windows},
\textit{Tree}
and
\textit{Rice} images from~\citep{hu2013fast},
which are resized to $1000 \times 1000 \times 3$ (Figure~\ref{fig:window}).
Each pixel is normalized to $[0, 1]$.
We randomly sample 5\% of the pixels for training, which are then  corrupted by
Gaussian noise $\mathcal{N}(0, 0.01^2)$;
and another 5\% clean pixels are used for validation.
The remaining
unseen clean pixels are used for testing.
Hyperparameters of the various methods are tuned by using the validation
set.


\begin{figure}[ht]
	\centering
	\subfigure[\textit{Windows}.]
	{\includegraphics[width = 0.14\textwidth]{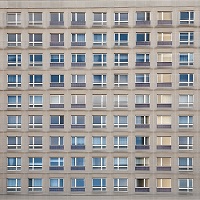}}
	\quad
	\subfigure[\textit{Tree}.]
	{\includegraphics[width = 0.14\textwidth]{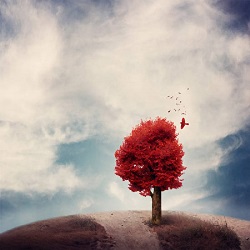}}
	\quad
	\subfigure[\textit{Rice}.]
	{\includegraphics[width = 0.14\textwidth]{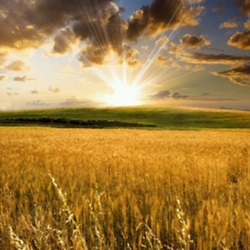}}	

	\vspace{-10px}
	\caption{Color images used in experiments. All are of size $1000 \times 1000 \times 3$.}
	\label{fig:window}
\end{figure}

Table~\ref{tab:colorimg} shows the RMSE
results. As can be seen,
the best convex methods 
(PA-APG and FaLRTC)
are
based on the overlapped nuclear norm.
This agrees with our motivation to build a nonconvex regularizer
based on the overlapped nuclear norm.
GDPAN, LRTC and NORT have similar RMSEs,
which are lower than those by convex regularization
and the factorization approach.
Convergence of the testing RMSE is shown in
Figure~\ref{fig:colorimg}.
As can be seen,
while ADMM solves the same convex model as PA-APG and FaLRTC,
it 
has 
slower convergence.
FFW, RP and TR-MM are very fast but their testing RMSEs are higher than that of NORT.
By utilizing the ``sparse plus low-rank'' structure and adaptive momentum,
NORT is more efficient than GDPAN and LRTC.

\begin{table}[ht]
	\centering
	\caption{Testing RMSEs on color images.
		For all images 5\% of the total pixels, 
		which are corrupted by Gaussian noise $\mathcal{N}(0, 0.01^2)$,
		are used for training.}
	\vspace{-8px}
	\small
	\begin{tabular}{c | c| c | c | c}
		\toprule
		      \multicolumn{2}{c}{dataset}        & \textit{Rice}              & \textit{Tree}              & \textit{Windows}           \\ \midrule
		   \multirow{8}{*}{convex}     & ADMM    & 0.0680$\pm$0.0003          & 0.0915$\pm$0.0005          & 0.0709$\pm$0.0004          \\
		      \cmidrule(r){2-5}        & PA-APG  & 0.0583$\pm$0.0016          & 0.0488$\pm$0.0007          & 0.0585$\pm$0.0002          \\
		      \cmidrule(r){2-5}        & FaLRTC  & 0.0576$\pm$0.0004          & 0.0494$\pm$0.0011          & 0.0567$\pm$0.0005          \\
		      \cmidrule(r){2-5}        & FFW     & 0.0634$\pm$0.0003          & 0.0599$\pm$0.0005          & 0.0772$\pm$0.0004          \\
		      \cmidrule(r){2-5}        & TR-MM   & 0.0596$\pm$0.0005          & 0.0515$\pm$0.0011          & 0.0634$\pm$0.0002          \\
		      \cmidrule(r){2-5}        & TenNN   & 0.0647$\pm$0.0004          & 0.0562$\pm$0.0004          & 0.0586$\pm$0.0003          \\ \midrule
		\multirow{7}{*}{factorization} & RP      & 0.0541$\pm$0.0011          & 0.0575$\pm$0.0010          & 0.0388$\pm$0.0026          \\
		      \cmidrule(r){2-5}        & TMac    & 0.1923$\pm$0.0005          & 0.1750$\pm$0.0006          & 0.1313$\pm$0.0005          \\
		      \cmidrule(r){2-5}        & CP-WOPT & 0.0912$\pm$0.0086          & 0.0750$\pm$0.0060          & 0.0964$\pm$0.0102          \\
		      \cmidrule(r){2-5}        & TMac-TT & 0.0729$\pm$0.0022          & 0.0665$\pm$0.0147          & 0.1045$\pm$0.0107          \\
		      \cmidrule(r){2-5}        & TRLRF   & 0.0640$\pm$0.0004          & 0.0780$\pm$0.0048          & 0.0588$\pm$0.0035          \\ \midrule
		  \multirow{4}{*}{nonconvex}   & GDPAN   & \textbf{0.0467$\pm$0.0002} & 0.0394$\pm$0.0006          & 0.0306$\pm$0.0007          \\
		      \cmidrule(r){2-5}        & LRTC    & \textbf{0.0468$\pm$0.0001} & 0.0392$\pm$0.0006          & 0.0304$\pm$0.0008          \\
		      \cmidrule(r){2-5}        & NORT    & \textbf{0.0468$\pm$0.0001} & \textbf{0.0386$\pm$0.0009} & \textbf{0.0297$\pm$0.0007} \\ \bottomrule
	\end{tabular}
	\label{tab:colorimg}
\end{table}

\begin{figure}[ht]
	\centering
	\includegraphics[width=0.32\textwidth]{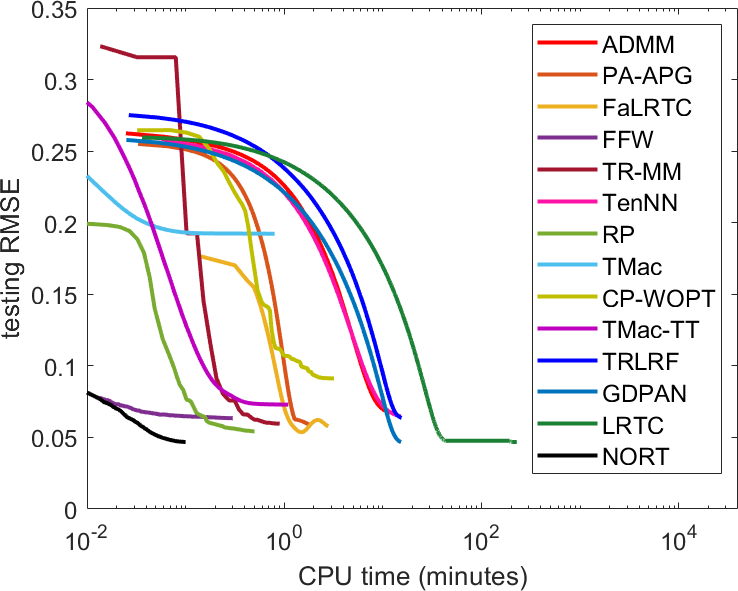}
	\includegraphics[width=0.315\textwidth]{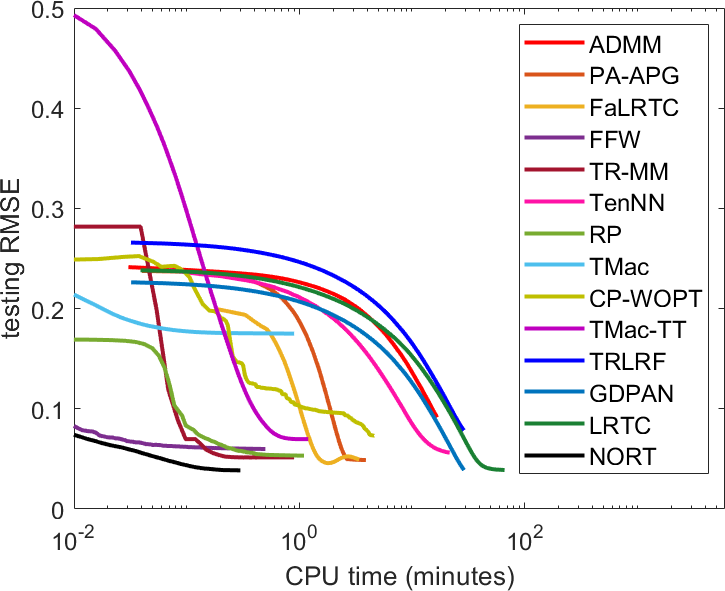}
	\includegraphics[width=0.32\textwidth]{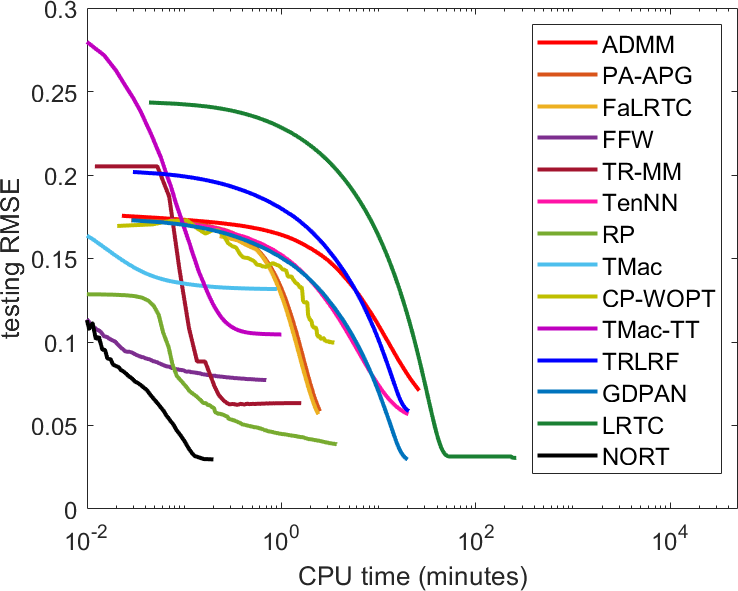}
	
	\subfigure[\textit{Rice}.]
	{\includegraphics[width=0.32\textwidth]{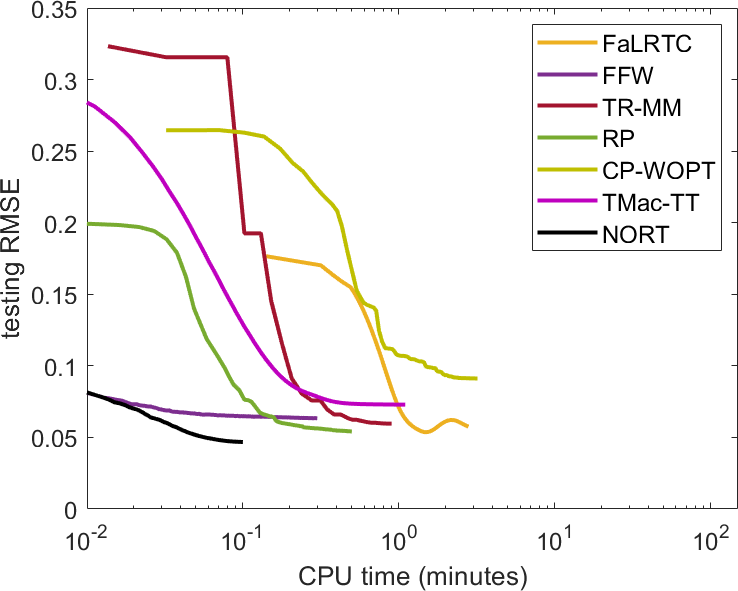}}
	\subfigure[\textit{Tree}.]
	{\includegraphics[width=0.315\textwidth]{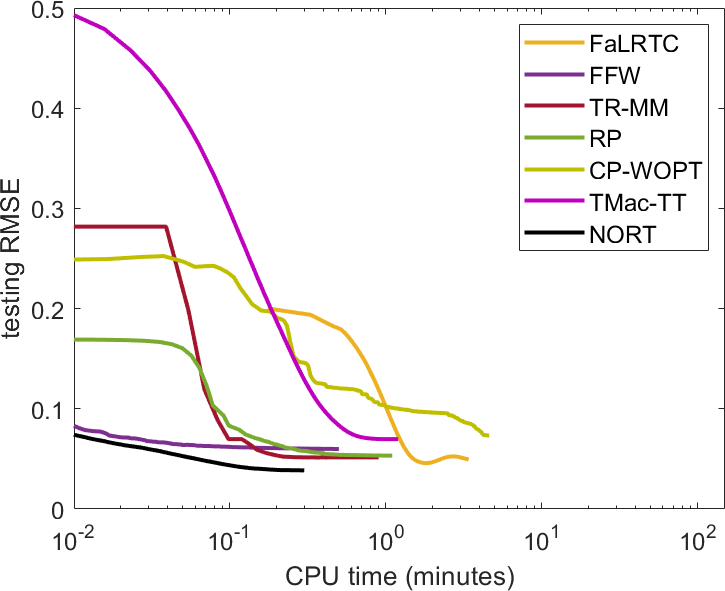}}
	\subfigure[\textit{Windows}.]
	{\includegraphics[width=0.32\textwidth]{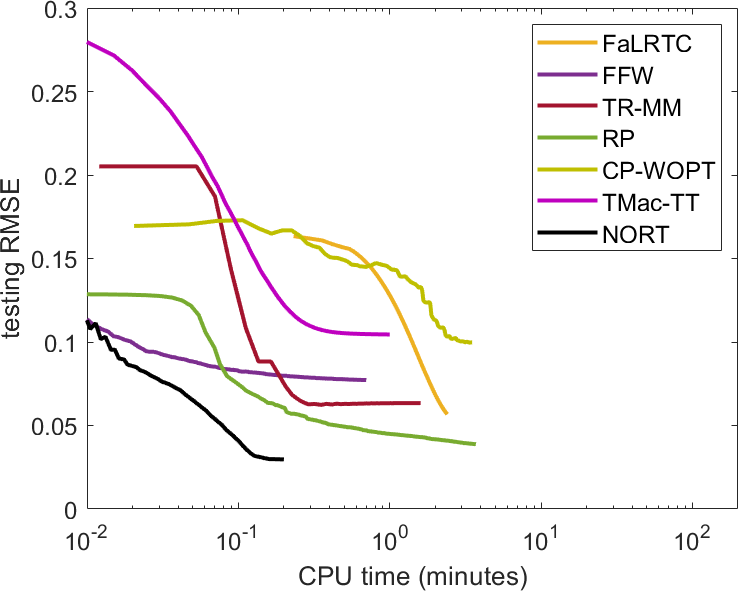}}

	\vspace{-10px}
	\caption{Testing RMSE versus CPU time on color images.
		Top: All methods;
		Bottom: For improved clarity, methods which are too slow or with too poor performance are removed.}
	\label{fig:colorimg}
\end{figure}

Finally,
Table~\ref{tab:colorimg:nlvl}
compares
NORT with PA-APG 
and RP, which are
the best  convex-regularization-based and
factorization-based 
algorithms, respectively, as observed in Table~\ref{tab:colorimg}.
Table~\ref{tab:colorimg:nlvl} shows
the testing RMSEs
at different noise levels $\sigma$'s.
As can be seen,
the testing RMSEs of all methods 
increase as $\sigma$ increases.
NORT has 
lower RMSEs at all $\sigma$ settings.
This is because 
natural images may not be exactly low-rank,
and adaptive penalization of the singular values 
can better preserve the spectrum.
A similar observation has also been made 
for nonconvex regularization on 
images
\citep{yao2018large,lu2016nonconvex}.
However,
when the noise level becomes
very high
($\sigma = 0.1$ with pixel values in $[0, 1]$),
though NORT is still the best,
its testing RMSE is not small.

\begin{table}[ht]
	\centering
	\caption{Testing RMSEs on image \textit{Tree} at different noise levels
	$\sigma$.
		The percentage followed by the marker
		$\uparrow$ indicates the relative increase of 
		testing RMSE compared with NORT.}
	\small
	\vspace{-8px}
	\begin{tabular}{c c|  c | c | c}
		\toprule
&        & $\sigma=0.001$                     & $\sigma=0.01$
& $\sigma=0.1 $                       \\ \midrule
		   (convex)     & PA-APG & 0.0149 (35.8\%$\uparrow$) & 0.0488 (24.6\%$\uparrow$) & 0.1749 (18.6\% $\uparrow$) \\ \midrule
		(factorization) & RP     & 0.0139 (26.0\%$\uparrow$) & 0.0575 (15.6\%$\uparrow$) & 0.1623 (10.1\% $\uparrow$) \\ \midrule
		  (nonconvex)   & NORT   & 0.0110                     & 0.0386                     & 0.1474                     \\ \bottomrule
	\end{tabular}
	\label{tab:colorimg:nlvl}
\end{table}

\subsubsection{Remote Sensing Data}

Experiments are performed on three hyper-spectral images
(Figure~\ref{fig:herspe}):
\textit{Cabbage} (1312$\times$432$\times$49),
\textit{Scene} (1312$\times$951$\times$49)
and
\textit{Female} (592$\times$409$\times$148).\footnote{\textit{Cabbage} and \textit{Scene} images are from \url{https://sites.google.com/site/hyperspectralcolorimaging/dataset},
while the 
\textit{Female} images are downloaded from \url{http://www.imageval.com/scene-database-4-faces-3-meters/}.
}:
The third dimension is for the bands of images.

\begin{figure}[ht]
	\centering
	\subfigure[\textit{Cabbage}.]
	{\includegraphics[height=0.15\columnwidth]{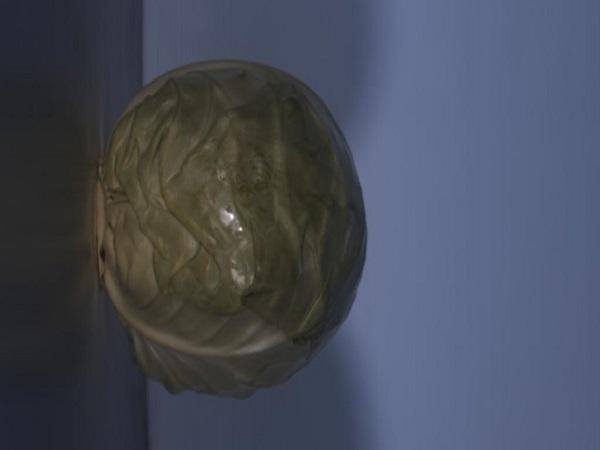}}
	\subfigure[\textit{Scene}.]
	{\includegraphics[height=0.15\columnwidth]{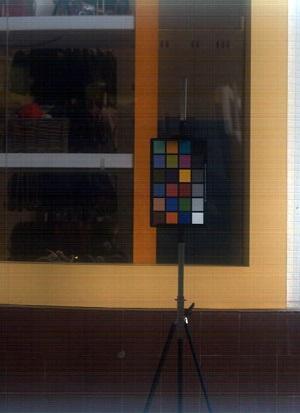}}
	\subfigure[\textit{Female}.]
	{\includegraphics[height=0.15\columnwidth]{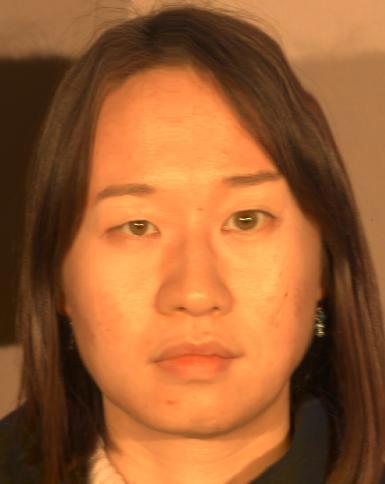}}
	
	\vspace{-10px}
	\caption{Hyperspectral images used in the experiment.}
	\label{fig:herspe}
\end{figure}

We use the same setup as in Section~\ref{sec:image},
and hyperparameters are tuned on the validation set.
ADMM, TenNN, GDPAN, LRTC, TMac-TT and TRLRF
are 
slow
and so not compared.
Results are shown in Table~\ref{tab:hypspec}.
Again, NORT achieves much lower testing RMSE than convex regularization and
factorization approach.
Figure~\ref{fig:hyper} shows convergence of the testing RMSE.
As can be seen, NORT is the fastest. 

\begin{table}[ht]
	\centering
	\caption{Testing RMSEs on remote sensing data.}
	\small
	\begin{tabular}{c | c | c | c | c}
		\toprule
\multicolumn{2}{c|}{} & \textit{Cabbage}           & \textit{Scene}             & \textit{Female}            \\ \midrule
		\multirow{5}{*}{convex}     & PA-APG  & 0.0913$\pm$0.0006          & 0.1965$\pm$0.0002          & 0.1157$\pm$0.0003          \\ \cmidrule(r){2-5}
		& FaLRTC  & 0.0909$\pm$0.0002          & 0.1920$\pm$0.0001          & 0.1133$\pm$0.0004          \\ \cmidrule(r){2-5}
		& FFW     & 0.0962$\pm$0.0004          & 0.2037$\pm$0.0002          & 0.2096$\pm$0.0006          \\ \cmidrule(r){2-5}
		& TR-MM   & 0.0959$\pm$0.0001          & 0.1965$\pm$0.0002          & 0.1397$\pm$0.0006          \\ \midrule
		\multirow{4}{*}{factorization} & RP      & 0.0491$\pm$0.0011          & 0.1804$\pm$0.0005          & 0.0647$\pm$0.0003          \\ \cmidrule(r){2-5}
		& TMac    & 0.4919$\pm$0.0059          & 0.5970$\pm$0.0029          & 1.9897$\pm$0.0006          \\ \cmidrule(r){2-5}
		& CP-WOPT & 0.1846$\pm$0.0514          & 0.4811$\pm$0.0082          & 0.1868$\pm$0.0013          \\ \midrule
		nonconvex   & NORT    & \textbf{0.0376$\pm$0.0004} & \textbf{0.1714$\pm$0.0012} & \textbf{0.0592$\pm$0.0002} \\ \bottomrule
	\end{tabular}
	\label{tab:hypspec}
\end{table}

\begin{figure}[ht]
	\centering
	\subfigure[\textit{Cabbage}.
	\label{fig:cabbage}]
	{\includegraphics[width=0.32\textwidth]{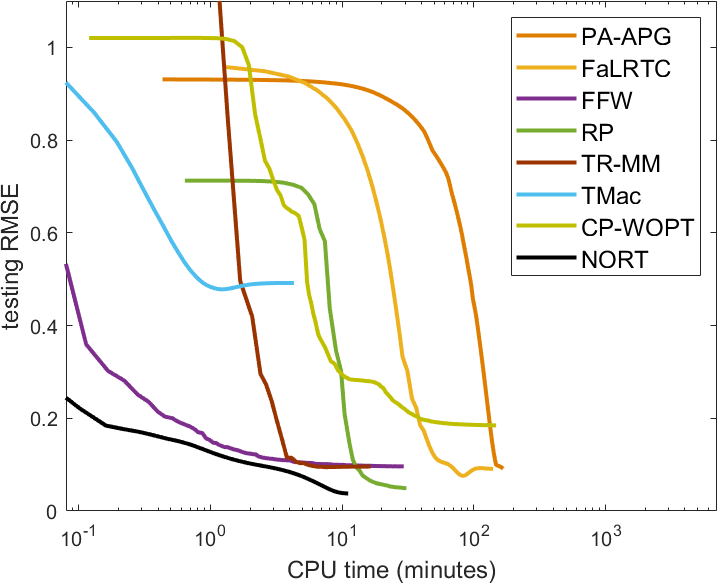}}
	\subfigure[\textit{Female}.
	\label{fig:leresfemale}]
	{\includegraphics[width=0.32\textwidth]{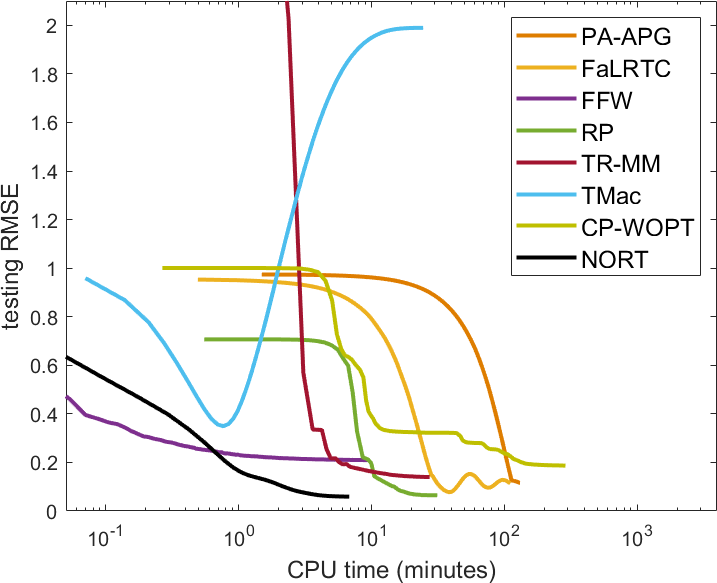}}
	\subfigure[\textit{Scene}.]
	{\includegraphics[width=0.32\textwidth]{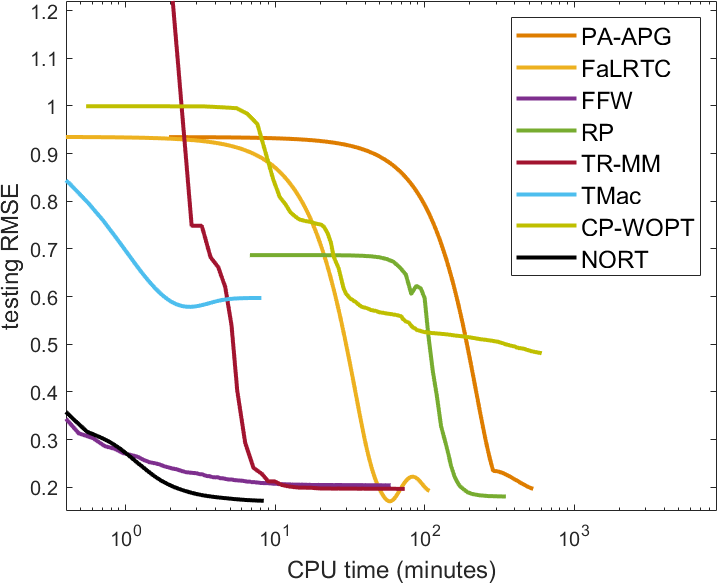}}
	
	\vspace{-10px}
	\caption{Testing RMSE versus CPU time on remote sensing data.}
	\label{fig:hyper}
\end{figure}

\subsubsection{Social Networks}

In this experiment,
we
consider multi-relational link prediction 
\citep{guo2017efficient}
as a tensor completion problem. Experiment is performed 
on the \textit{YouTube} data set\footnote{\url{http://leitang.net/data/youtube-data.tar.gz}}
\citep{lei2009analysis},
which contains 15,088 users and five types of user interactions.
Thus, it forms a
15088$\times$15088$\times$5 tensor,
with a total of 27,257,790 nonzero elements.
Besides the full set,
we also experiment with a \textit{YouTube} subset
obtained by randomly selecting 1,000 users (leading to 12,101 observations).
We use $50\%$ of the observations for training, another $25\%$ for validation and the rest for testing.
Table~\ref{tab:linkpred} shows the testing RMSE, 
and Figure~\ref{fig:linkpred} shows the convergence.
As can be seen, 
NORT achieves smaller RMSE and is also much faster.

\begin{table}[ht]
	\centering
	\caption{Testing RMSEs on \textit{YouTube} data sets.  FaLRTC, PA-APG, TR-MM and CP-WOPT are slow, and thus not run on the full set.}
	\vspace{-10px}
	\small
	\begin{tabular}{c | c | c |c}
		\toprule
\multicolumn{2}{c|}{} & subset                   & full set                 \\ \midrule
		   \multirow{5}{*}{convex}     & FaLRTC  & 0.657$\pm$0.060          & ---                      \\ \cmidrule(r){2-4}
		                               & PA-APG  & 0.651$\pm$0.047          & ---                      \\ \cmidrule(r){2-4}
		                               & FFW     & 0.697$\pm$0.054          & 0.395$\pm$0.001          \\ \cmidrule(r){2-4}
		                               & TR-MM   & 0.670$\pm$0.098          & ---                      \\ \midrule
		\multirow{4}{*}{factorization} & RP      & 0.522$\pm$0.038          & 0.410$\pm$0.001          \\ \cmidrule(r){2-4}
		                               & TMac    & 0.795$\pm$0.033          & 0.611$\pm$0.007          \\ \cmidrule(r){2-4}
		                               & CP-WOPT & 0.785$\pm$0.040          & ---                      \\ \midrule
		          nonconvex            & NORT    & \textbf{0.482$\pm$0.030} & \textbf{0.370$\pm$0.001} \\ \bottomrule
	\end{tabular}
	\label{tab:linkpred}
\end{table}

\begin{figure}[ht]
	\centering
	\subfigure[Subset.]
	{\includegraphics[width=0.32\textwidth]{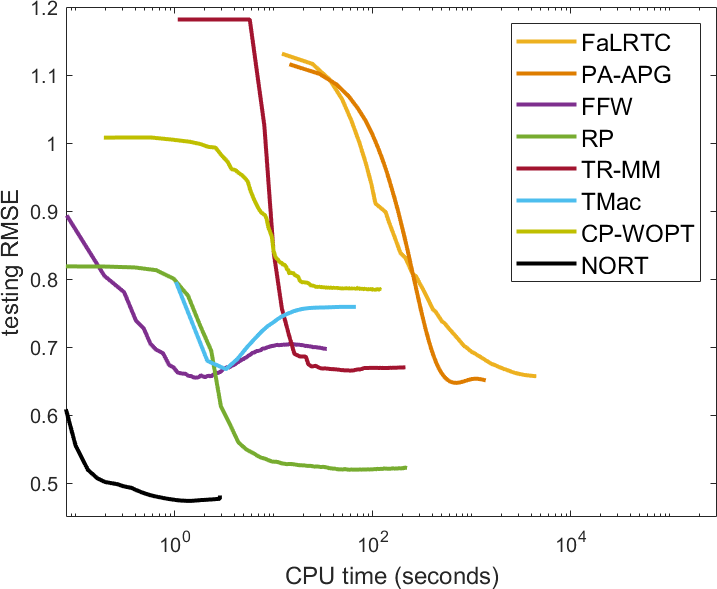}}
	\qquad
	\subfigure[Full set.]
	{\includegraphics[width=0.33\textwidth]{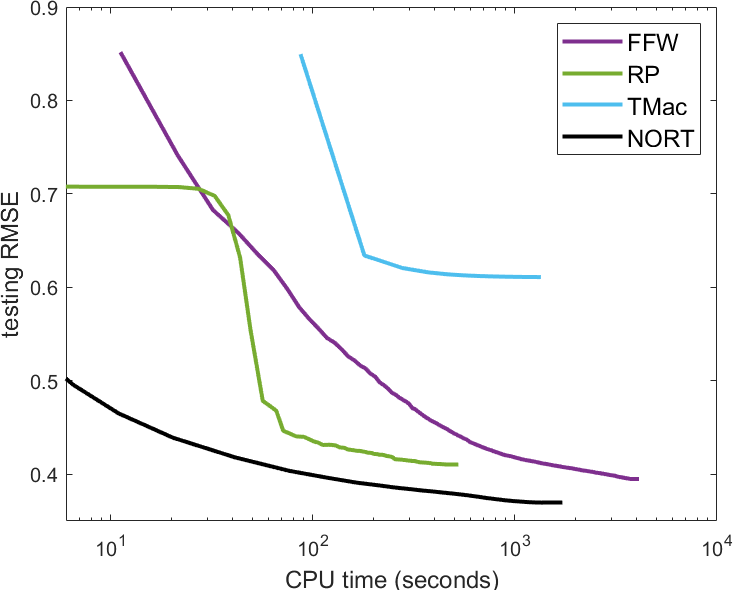}}
	
	\vspace{-10px}
	\caption{Testing RMSE versus CPU time on \textit{Youtube}.}
	\label{fig:linkpred}
\end{figure}

\subsection{Link Prediction in Knowledge Graph} 
\label{sec:exp:kg}

Knowledge Graph (KG)~\citep{nickel2015review,toutanova2015representing}
is an active research topic in data mining and machine learning.
Let $\mathcal{E}$ be the entity set and $\mathcal{R}$ be the relation set.
In a KG, nodes are the entities, while edges
are relations
representing the triplets 
$\mathcal{S}=\{(h,r,t)\}$, where $h\in\mathcal{E}$ is the {\em head entity\/},
$t\in\mathcal{E}$ is the {\em tail entity\/}, and $r\in\mathcal{R}$ is the {\em
relation\/} between $h$ and $t$. 

KGs have many downstream applications, such as
link prediction and triplet classification.
It is common to store KGs as tensors, and solve the KG learning
tasks with tensor methods
\citep{lacroix2018canonical,balazevic2019tucker}.
Take link prediction as an example.
The KG can be seen as a 3-order incomplete tensor $\ten{O} = \{\pm1\}\in\R^{I^1\times
I^2 \times I^3}$,
where $I^1=I^2=|\mathcal{E}|$ and $I^3=|\mathcal{R}|$. 
$\ten{O}_{i_1i_2i_3}=1$ when 
 entities $i_1$ and $i_2$ have the relation $i_3$, and $-1$ otherwise. 
Let $\vect{\Omega}$ be a mask tensor denoting the observed values in $\ten{O}$, i.e.,
$\vect{\Omega}_{i_1i_2i_3}=1$ if $\ten{O}_{i_1i_2i_3}$ is observed and 
$0$ otherwise.
The task is to predict elements in $\ten{O}$
which are not observed.
Since $\ten{O}$ is binary, 
it is common to  use the log loss
as $\ell(\cdot,\cdot)$ in \eqref{eq:pro}.
The objective then becomes:
\begin{align}\label{eq:tc_kge}
\min\nolimits_{\ten{X}}
\sum\nolimits_{(i_1 i_2 i_3)\in\vect{\Omega}}\log(1+\exp( - \ten{X}_{i_1i_2i_3} \ten{O}_{i_1i_2i_3}))
+ \sum\nolimits_{i = 1}^D \lambda_i \phi( \ten{X}_{\ip{i}} ).
\end{align}
In step~9 of Algorithm~\ref{alg:nort},
it is easy to see that
\begin{align*}
[ \xi(\ten{X}_t) ]_{i_1 i_2 i_3}
=\left\{
\begin{array}{cc}
\frac{- \ten{O}_{i_1 i_2 i_3} \cdot \exp( - \ten{X}_{i_1 i_2 i_3} \ten{O}_{i_1 i_2 i_3})}
{1 + \exp(- \ten{X}_{i_1 i_2 i_3} \ten{O}_{i_1 i_2 i_3})} & (i_1 i_2 i_3)\in\vect{\Omega}
\\
0& (i_1 i_2 i_3)\notin\vect{\Omega}
\end{array}\right..
\end{align*}

Experiments are performed on two benchmark data sets:
\textit{WN18RR}\footnote{\url{https://github.com/TimDettmers/ConvE}}
\citep{dettmers2018convolutional} and
\textit{FB15k-237}\footnote{\url{https://www.microsoft.com/en-us/download/details.aspx?id=52312}}
\citep{toutanova2015representing}, which are subsets of \textit{WN18} and \textit{FB15k}
\citep{bordes2013translating}, respectively. 
\textit{WN18} is a subset of WordNet~\citep{miller1995wordnet}, and \textit{FB15k} is a subset of the Freebase database~\citep{bollacker2008freebase}.
To 
avoid test leakage,
\textit{WN18RR} and \textit{FB15k-237}
do not contain near-duplicate and inverse-duplicate relations.
Hence, link prediction on
\textit{WN18RR} and \textit{FB15k-237} 
is harder but more recommended than that on \textit{WN18} and \textit{FB15k}
\citep{dettmers2018convolutional}.  
To form the entity set $\mathcal{E}$,
we keep the top 500 (head and tail) entities  that appear most frequently in
the relations ($r$'s).
Relations that do not link to any of these 500 entities
are removed, and
those remained form the relation set $\mathcal{R}$. 
Following the public splits on entities in $\mathcal{E}$ and relations in $\mathcal{R}$~\citep{han2018openke},
we split the observed triplets in $\mathcal{S}$ into a training set $\mathcal{S}_\text{train}$, validation set $\mathcal{S}_\text{val}$ and testing set $\mathcal{S}_\text{test}$.
For each observed triplet $(h,r,t)\in\mathcal{S}_\text{train}$, we sample a negative triplet
from
$\mathcal{\hat{S}}_{(h,r,t)} = \{(\hat{h},r,t)\notin\mathcal{S}|\hat{h}\in\mathcal{E}\}\cap\{(h,r,\hat{t})\notin\mathcal{S}|\hat{t}\in\mathcal{E}\}$. 
We avoid duplicate negative triplets during sampling. 
We then represent the KG's by tensors $\ten{O}$'s of size $500\times500\times 8$ for \textit{WN18RR}, and $500\times500\times 39$ for \textit{FB15k-237}  with corresponding mask tensors $\vect{\Omega}$'s.

Following~\citep{bordes2013translating,dettmers2018convolutional}, 
performance is evaluated on the testing triplets in $\bar{\vect{\Omega}}$ by the following metrics: 
(i) mean reciprocal ranking:
$\text{MRR} = 1 / \| \bar{\vect{\Omega}} \|_0 \sum_{(i_1 i_2
i_3)\in\bar{\vect{\Omega}}}$ 
$1 / \text{rank}_{i_3}$, 
where $\text{rank}_{i_3}$ is the ranking of score
$\ten{X}_{i_1i_2i_3}$ among
$\{\ten{X}_{i_1i_2j}\}$ with $j=1,\dots,|\mathcal{R}|$ in descending order; 
(ii) Hits$@1= 1 / \| \bar{\vect{\Omega}} \|_0 \sum_{(i_1 i_2
i_3)\in\bar{\vect{\Omega}}} \mathbb{I}(\text{rank}_{i_3}\le1)$,
where $\mathbb{I}(c)$ is the indicator function which returns 1 if the constraint $c$ is satisfied and 0 otherwise; 
and (iii) Hits$@3= 1 / \|\bar{\vect{\Omega}}\|_0\sum_{(i_1 i_2 i_3)\in\bar{\vect{\Omega}}} \mathbb{I}(\text{rank}_{i_3}\le3)$.
For these three metrics, the higher the better.

The aforementioned algorithms are designed for the square loss, but not 
for the log loss in \eqref{eq:tc_kge}.
We adapt the
gradient-based algorithms including
PA-APG, ADMM and CP-WOPT, as
we only need to change the gradient calculation 
for \eqref{eq:tc_kge}.
As a further baseline, 
we implement the 
classic Tucker decomposition~\citep{tucker1966some,Kolda2009} 
to optimize \eqref{eq:tc_kge}.
While RP~\citep{kasai2016low}
is 
the state-of-the-art Tucker-type algorithm,
it uses Riemannian preconditioning  and cannot be easily modified to handle
nonsmooth loss.

Results on
	\textit{WN18RR} 
and 
	\textit{FB15k-237} 
are shown in Tables~\ref{tab:WN18RR}
and \ref{tab:FB15k}, respectively.
As can be seen, 
NORT again obtains the best ranking results.
Figure~\ref{fig:kge} shows convergence of MRR with CPU time, and
NORT is about two orders of magnitude faster than the other methods.

\begin{table}[H]
	\centering
	\caption{Testing performance on the \textit{WN18RR} data set.}
	\small
	\begin{tabular}{ c | c|c|c|c}
		\toprule
\multicolumn{2}{c|}{} &           MRR            &          Hits@1          &          Hits@3          \\ \midrule
		\multirow{2}{*}{convex}        & ADMM    &     0.362$\pm$0.029      &     0.156$\pm$0.024      &     0.422$\pm$0.038      \\ \cmidrule(r){2-5}
		              & PA-APG  &     0.399$\pm$0.017      &     0.203$\pm$0.023      &     0.500$\pm$0.038      \\ \midrule
		\multirow{2}{*}{factorization} & Tucker  &     0.439$\pm$0.013      &     0.309$\pm$0.016      &     0.438$\pm$0.026      \\ \cmidrule(r){2-5}
		              & CP-WOPT &     0.417$\pm$0.018      &     0.266$\pm$0.027      &     0.453$\pm$0.019      \\ \midrule
		nonconvex     & NORT    & \textbf{0.523$\pm$0.022} & \textbf{0.375$\pm$0.033} & \textbf{0.578$\pm$0.024} \\ \bottomrule
	\end{tabular}
	\label{tab:WN18RR}
\end{table}

\begin{table}[H]
	\centering
	\caption{Testing performance on the \textit{FB15k-237} data set.}
	\small
	\begin{tabular}{ c | c|c|c|c}
		\toprule
\multicolumn{2}{c|}{} &           MRR            &          Hits@1          &          Hits@3          \\ \midrule
		\multirow{2}{*}{convex}        & ADMM    &     0.466$\pm$0.006      &     0.411$\pm$0.006      &     0.452$\pm$0.011      \\ \cmidrule(r){2-5}
		              & PA-APG  &     0.514$\pm$0.013      &     0.463$\pm$0.015      &     0.590$\pm$0.016      \\ \midrule
		\multirow{2}{*}{factorization} & Tucker  &     0.471$\pm$0.018      &     0.355$\pm$0.017      &     0.465$\pm$0.015      \\ \cmidrule(r){2-5}
		              & CP-WOPT &     0.420$\pm$0.021      &     0.373$\pm$0.015      &     0.488$\pm$0.014      \\ \midrule
		nonconvex     & NORT    & \textbf{0.677$\pm$0.007} & \textbf{0.609$\pm$0.007} & \textbf{0.698$\pm$0.011} \\ \bottomrule
	\end{tabular}
	\label{tab:FB15k}
\end{table}

\begin{figure}[H]
	\centering
	\subfigure[\textit{WN18RR}.]
	{\includegraphics[width=0.32\textwidth]{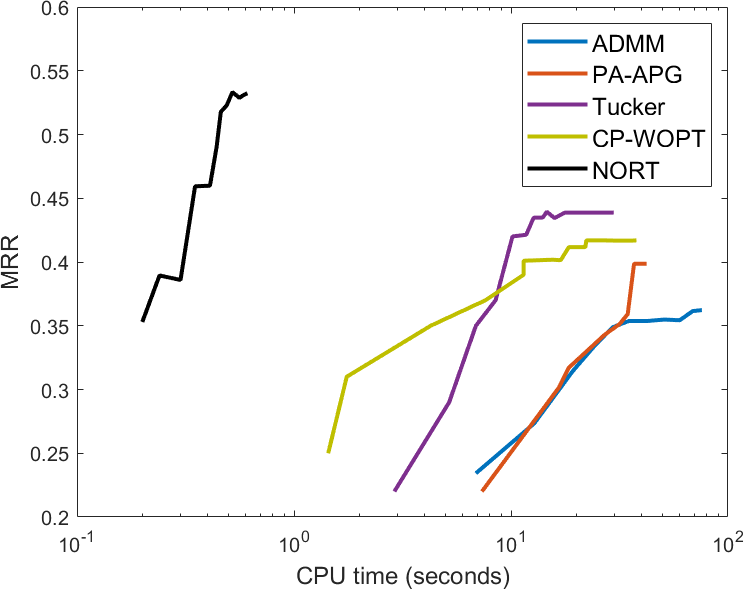}}
	\qquad
	\subfigure[\textit{FB15k-237}.]
	{\includegraphics[width=0.315\textwidth]{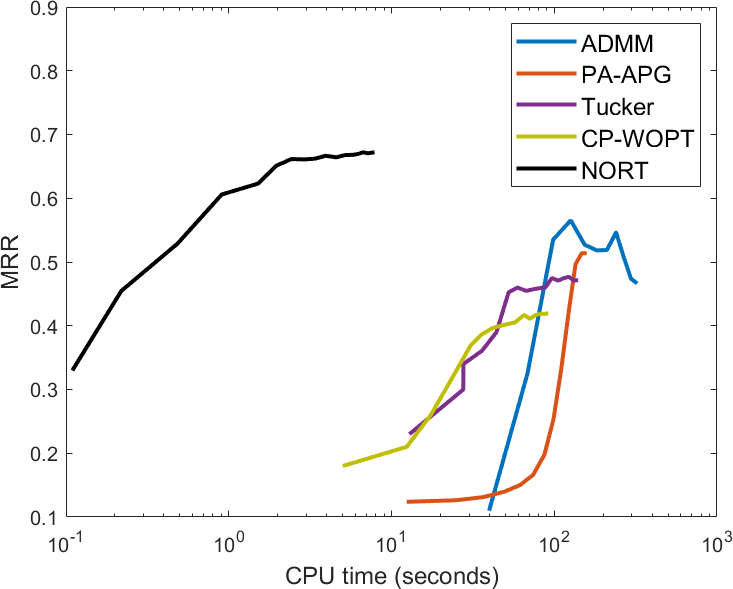}}

	\vspace{-10px}
	\caption{Testing MRR versus CPU time on the \textit{WN18RR} and \textit{FB15k-237} data sets.}
	\label{fig:kge}
\end{figure}

\subsection{Robust Tensor Completion}
\label{sec:exp:rpca}

In this section, we apply the proposed method on robust 
video
tensor completion.
Three videos (\textit{Eagle}\footnote{\url{http://youtu.be/ufnf_q_3Ofg}},  \textit{Friends}\footnote{\url{http://youtu.be/xmLZsEfXEgE}} and \textit{Logo}\footnote{\url{http://youtu.be/L5HQoFIaT4I}}) from~\citep{indyk2019learning} are used. Example frames are shown in
Figure~\ref{fig:exampleImages}.
For each video,
$200$ consecutive 
$360\times 640$
frames
are
downloaded 
from Youtube, and
the pixel values are normalized to $[0,1]$. 
Each video can then be represented as a 
fourth-order 
tensor $\bar{\ten{O}}$ with size 
$360\times 640\times 3\times 200$.
Each element of $\bar{\ten{O}}$ is normalized to $[0,1]$.
This clean tensor $\bar{\ten{O}}$ is corrupted by
a noise tensor $\ten{N}$ to form $\ten{O}$. 
$\ten{N}$ is a sparse random tensor with approximately 1\% nonzero elements. Each
entry is
first drawn uniformly from the
interval $[0, 1]$, and then multiplied by $5$ times the maximum value of $\bar{\ten{O}}$. 
Hyperparameters are chosen based on performance on the first 100 noisy frames.
Denoising performance is measured by the RMSE between the clean tensor $\bar{\ten{O}}$ and reconstructed tensor $\ten{X}$
on the last 100 frames.

\begin{figure}[H]
	\centering
	\subfigure[\textit{Eagle}.]
	{\includegraphics[width = 0.22\columnwidth]{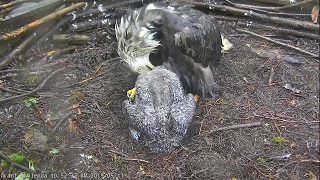}}
	\quad
	\subfigure[\textit{Friends}.]
	{\includegraphics[width = 0.22\columnwidth]{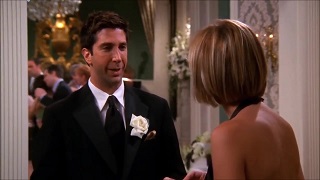}} 
	\quad
	\subfigure[\textit{Logo}.]
	{\includegraphics[width = 0.22\columnwidth]{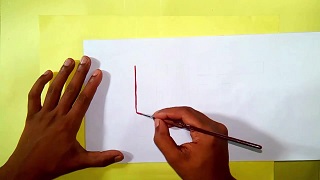}} 
	
	\vspace{-10px}
	\caption{Example image frames in the videos.}
	\label{fig:exampleImages}
\end{figure}

For the robust tensor completion,
we take RTDGC~\citep{gu2014robust} as the baseline,
which adopts  the $\ell_1$ loss and overlapped nuclear norm in \eqref{eq:robustpro}
(i.e., $\kappa_{\ell}(x)=x$ and $\phi$ is the nuclear norm).
As this is non-smooth and non-differentiable, 
RTDGC uses ADMM~\citep{boyd2011distributed} for the optimization,
which 
handles the robust loss and low-rank regularizer separately. 
As discussed in Section~\ref{sec:robustc},
we use the smoothing NORT (Algorithm~\ref{alg:snort}, with
$\delta_0=0.9$)
to optimize \eqref{eq:smoothpro}, the smoothed version of
\eqref{eq:robustpro}.
Table~\ref{tab:robust} shows the RMSE results. 
As can be seen, NORT obtains better denoising performance than RTDGC. 
This again validates the efficacy of nonconvex low-rank learning.
Figure~\ref{fig:robust} shows convergence of the testing RMSE.  
As shown, 
NORT leads to a lower RMSE
and converges much faster as folding/unfolding are avoided.

\begin{table}[ht]
	\centering
	\caption{Testing RMSEs on the videos.}
	\small
	\begin{tabular}{cc|c|c|c}
		\toprule
		\multicolumn{2}{c|}{} &       \textit{Eagle}        &      \textit{Friends}       &        \textit{Logo}        \\ \midrule
		 convex   &   RTDGC    &     0.122$\pm$0.007      &     0.128$\pm$0.005      &     0.112$\pm$0.008      \\ \midrule
		nonconvex &   NORT    & \textbf{0.090$\pm$0.003} & \textbf{0.075$\pm$0.002} & \textbf{0.088$\pm$0.004} \\ \bottomrule
	\end{tabular}
	\label{tab:robust}
\end{table}

\begin{figure}[ht]
	\centering
	\subfigure[\textit{Eagle}.]
	{\includegraphics[width=0.32\textwidth]{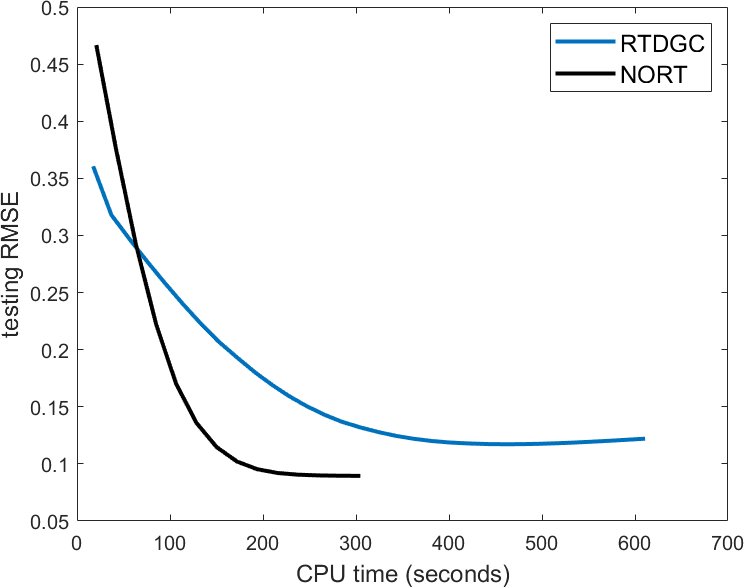}}
	\subfigure[\textit{Friends}.]
	{\includegraphics[width=0.32\textwidth]{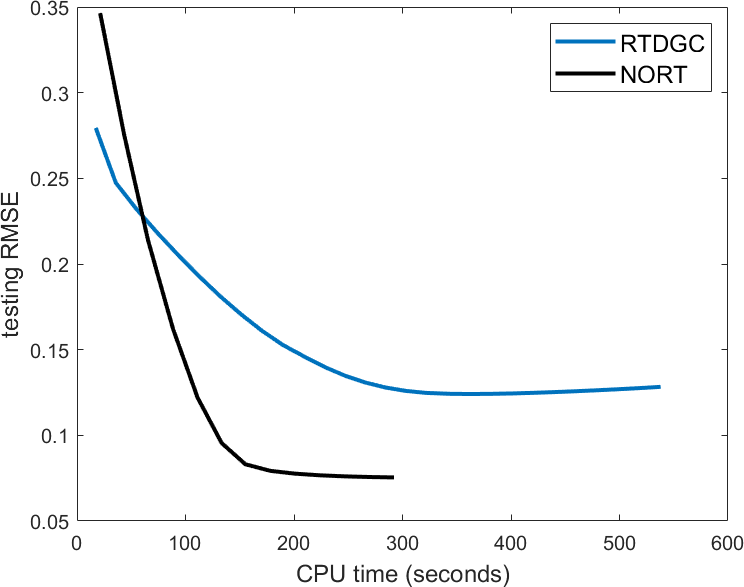}}
	\subfigure[\textit{Logo}.]
	{\includegraphics[width=0.32\textwidth]{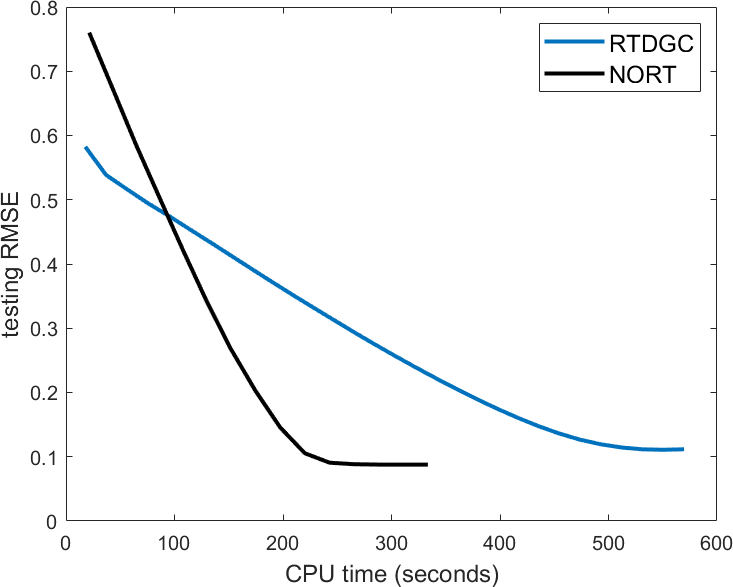}}
	\vspace{-10px}
	\caption{Testing RMSE versus CPU time on the videos.}
	\label{fig:robust}
\end{figure}

\subsection{Spatial-temporal Data}
\label{sec:exp:sptemp}

In this experiment,
we predict climate observations for locations that do not have any records. 
This is formulated as a regularized tensor completion problem in \eqref{eq:pro_reg}.
We use the square loss with 
a graph Laplacian regularizer constructed as in \eqref{eq:pro_reg}.  

We use	the \textit{CCDS} 
and \textit{USHCN} data sets from
\citep{bahadori2014fast}. \textit{CCDS}\footnote{\url{https://viterbi-web.usc.edu/~liu32/data/NA-1990-2002-Monthly.csv}} contains
monthly observations of 17 variables (such as carbon dioxide and temperature) in 125 stations from
January 1990 to December 2001.
\textit{USHCN}\footnote{\url{http://www.ncdc.noaa.gov/oa/climate/research/ushcn}}
contains monthly observations of 4 variables (minimum, maximum, average
temperature and total precipitation) in 1218 stations from from January 1919 to
November 2019. 
As discussed in Section~\ref{sec:rtc}, these records are  collectively represented by a 3-order tensor $\ten{O}\in\R^{I^1\times I^2 \times I^3}$, 
where $I^1$ is the number of locations, 
$I^2$ is the number of recorded time stamps, 
and $I^3$ is the number of variables corresponding to climate observations. 
Consequently, \textit{CCDS} is represented as a $125\times156\times17$ tensor and 
\textit{USHCN} is represented as a $1218\times1211\times4$ tensor. 
The affinity matrix is denoted $\bm{A}$, with
$\bm{A}_{ij}$ 
being the similarity 
$s(i,j)=\exp(-2b_{ij})$
between locations $i$ and $j$ 
($b_{ij}$ is the Haversine distance between $i$ and $j$). 
Following~\citep{bahadori2014fast}, we normalize the data to zero mean and unit variance, 
then randomly sample 10\% of the locations for training, 
another 10\% for validation,
and the rest
for testing.

Algorithms FaLRTC, FFW, TR-MM, RP and TMac cannot be directly used for this graph Laplacian regularized tensor completion problem, while  PA-APG, ADMM, Tucker and CP-WOPT can be adapted by modifying the
gradient calculation. 
Hence we adapt and implement PA-APG, ADMM, Tucker and CP-WOPT as baselines in this section.  
In addition, we
compare with a greedy algorithm 
(denoted ``Greedy")\footnote{This method is denoted ``ORTHOGONAL" in~\citep{bahadori2014fast} and obtains the best results  there.}
from~\citep{bahadori2014fast},
which successively adds a rank-1 matrix to
approximate the mode-$n$ unfolding with the rank constraint.
For the factorization-based algorithms Tucker and CP-WOPT,
the graph Laplacian regularizer $h$ takes 
the corresponding factor matrix rather than $\ten{X}_{\ip{1}}$ as the input. 
Specifically, recall that Tucker factorizes $\ten{X}$ into $[\ten{G};
\vect{B}^1,\vect{B}^2, \vect{B}^3]$, where $\ten{G}\in\R^{k^1\times k^2 \times k^3}$,  $\vect{B}^i\in\R^{I^i\times k^i}$, $i=1,2,3$, 
and $k^i$'s are hyperparameters.
When $k^1=k^2=k^3$ and $\ten{G}$ is superdiagonal, this reduces to the CP-WOPT decomposition. 
The graph Laplacian regularizer is then constructed as $h(\vect{B}^1)$ to leverage location proximity. 
As an additional baseline, we also experiment with a NORT variant that does not
use the Laplacian regularizer (denoted ``NORT-no-Lap''). 

\begin{table}[ht]
	\centering
	\caption{Testing RMSEs on \textit{CCDS} and \textit{USHCN} data sets.}
	\small
	\begin{tabular}{c | c | c | c}
		\midrule
		\multicolumn{2}{c|}{}	& \textit{CCDS}            & \textit{USHCN}           \\ \midrule
		\multirow{2}{*}{convex}     & ADMM    & 0.890$\pm$0.016          & 0.691$\pm$0.005          \\
		\cmidrule(r){2-4}        & PA-APG  & 0.866$\pm$0.014          & 0.680$\pm$0.009          \\ \midrule
		\multirow{2}{*}{factorization} & Tucker  & 0.856$\pm$0.026          & 0.647$\pm$0.006          \\
		\cmidrule(r){2-4}        & CP-WOPT & 0.887$\pm$0.018          & 0.688$\pm$0.009          \\\midrule
		\multirow{1}{*}{rank constraint}        & Greedy  & 0.871$\pm$0.008          & 0.658$\pm$0.012          \\ \midrule
		\multirow{2}{*}{nonconvex}     & NORT-no-Lap  &0.997$\pm$0.001  &1.391$\pm$0.001  \\
		\cmidrule{2-4}
		& NORT    & \textbf{0.793$\pm$0.002} & \textbf{0.583$\pm$0.012} \\ \bottomrule
	\end{tabular}
	\label{tab:climate}
\end{table}

Table~\ref{tab:climate} shows the RMSE
results.  Again,
NORT obtains the lowest testing RMSEs.
Moreover, when
the Laplacian regularizer
is not used,
the testing RMSE is much higher,
demonstrating that the missing slices cannot be reliably completed.
Figure~\ref{fig:climate} shows the convergence.
As can be seen, NORT is orders of magnitude faster than the other algorithms.
The gaps on the performance and speed
between NORT and the other baselines are more obvious on the larger \textit{USHCN} data set. 
Further, note  from Figures~\ref{climate-obj1} and \ref{climate-obj2} that though NORT-no-Lap 
has converged, 
it cannot decrease the testing RMSE during learning 
(Figures~\ref{climate-err1} and \ref{climate-err2}).
This validates the efficacy of the graph Laplacian regularizer.

\begin{figure}[ht]
	\centering
	
	\subfigure[\textit{CCDS}. \label{climate-obj1}]
	{\includegraphics[width=0.335\textwidth]{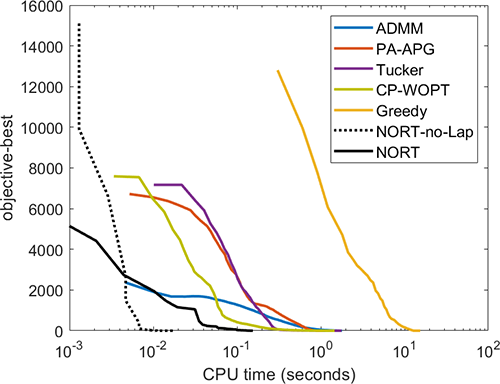}}
	\qquad
	\subfigure[\textit{USHCN}. \label{climate-obj2}]
	{\includegraphics[width=0.315\textwidth]{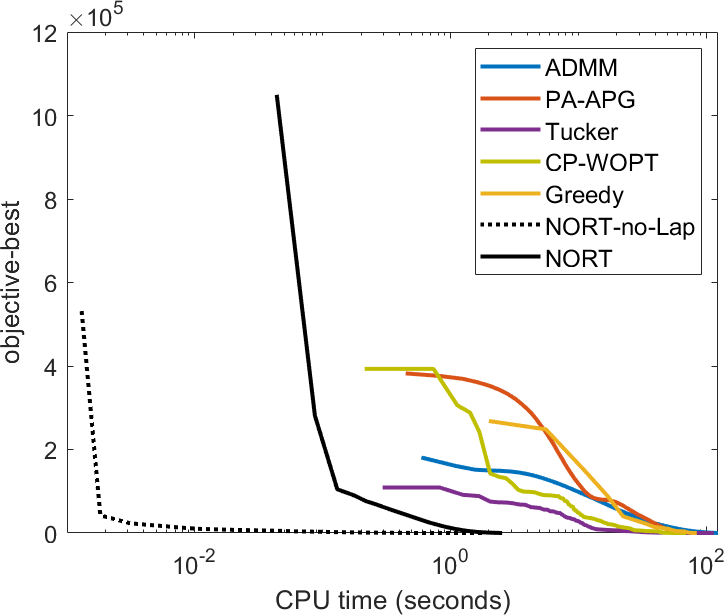}}
	
	\subfigure[\textit{CCDS}.\label{climate-err1}]
	{\includegraphics[width=0.325\textwidth]{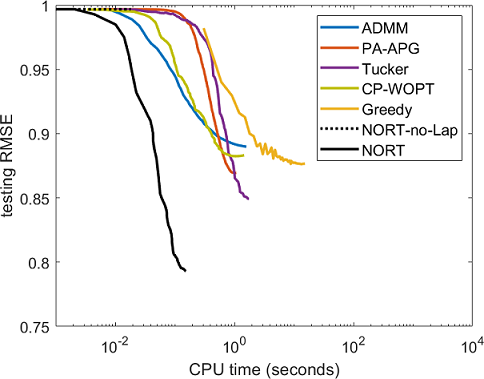}}
	\qquad
	\subfigure[\textit{USHCN}. \label{climate-err2}]
	{\includegraphics[width=0.32\textwidth]{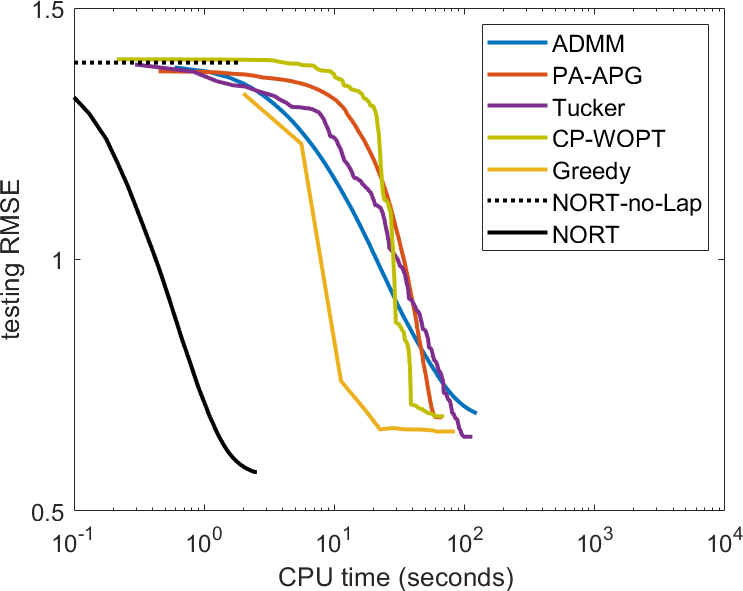}}
	
	\vspace{-10px}
	\caption{Convergence of the 
	training objective (below) 
	and 
	testing RMSE (top) 
	versus 
	CPU time on the spatial-temporal data.}
	\label{fig:climate}
\end{figure}

\section{Conclusion}
\label{sec:concl}

In this paper,
we propose a low-rank tensor completion model with nonconvex regularization.
An efficient
nonconvex proximal average algorithm
is developed,
which maintains the ``sparse plus
low-rank" structure throughout the iterations and  
incorporates adaptive momentum.
Convergence to critical points is guaranteed,
and the obtained critical points can have small statistical errors.
The algorithm is also extended for nonsmooth losses and additional
regularization, demonstrating
broad applicability of the proposed algorithm.
Experiments on a variety of synthetic and real data sets are performed.
Results show that the proposed  algorithm is more efficient
and more accurate
than existing state-of-the-art.

In the future,
we will extend the proposed algorithm
to simultaneous completion of multiple tensors,
e.g., collaborative tensor completion~\citep{Zheng2013CollaborativeMF} and coupled tensor completion~\citep{Wimalawarne2018ConvexCM}.
Besides, it is also interesting to study how the proposed algorithm can be
efficiently parallelized on GPUs and distributed computing environments~\citep{Phipps2019SoftwareFS}. 

\section*{Acknowledgement}

BH was supported by the RGC Early Career Scheme 
No.22200720 and NSFC Young Scientists Fund No. 62006202.

\appendix

\section*{Appendix}


\section{Comparison with Incoherence Condition}
\label{app:incoh}

The matrix incoherence condition~\citep{candes2009exact,candes2011robust,negahban2012restricted}
is in form of the singular value decomposition
$\vect{X} = \bm{U} \bm{\Sigma} \bm{V}^{\top} \in \R^{m \times n}$,
where
$\bm{U} \in \R^{m \times r}$ (resp. $\bm{V} \in \R^{n \times r}$) 
contains the left (resp. right) singular vectors
and $\bm{\Sigma} \in \R^{r \times r}$ is the diagonal matrix containing singular values.
The purpose of this condition
is to enforce
that the left and right 
singular vectors should not be aligned
with the standard basis
(i.e., vector $\vect{e}_i$'s with the $i$th dimension being $1$ and others being $0$).
Typically,
this condition is stated as
\begin{align}
	\max_{j = 1, \dots, m}
	\left|
	[ \bm{U} \bm{U}^{\top} ]_{jj}  
	\right| 
	\le
	\mu_0
	\frac{r}{m},
	\quad
	\text{and}
	\quad
	\max_{j = 1, \dots, n}
	\left|
	[ \bm{V} \bm{V}^{\top} ]_{jj}  
	\right| 
	\le
	\mu_0
	\frac{r}{n},
	\label{eq:matinco}
\end{align}
for some constant $\mu_0 > 0$.
Note 
that \eqref{eq:matinco}
does not depend on the singular values of $\vect{X}$.
However,
this condition can be restrictive in realistic settings,
where the underlying matrix is contaminated
by noise.
In this case,
the observed matrix can have small singular values.
Therefore,
we need to impose conditions
related to the singular values,
and \eqref{eq:spiky} shows such a dependency.
An example matrix satisfying the
matrix RSC condition but not the incoherence condition is in Section~3.4.2 of~\citep{negahban2012restricted}. 
As a result,
the RSC condition,
which involves singular values, 
is less restrictive than the incoherence condition,
and can better describe ``spikiness''.

\section{Proofs}

\subsection{Proposition~\ref{pr:mulv}}

\begin{proof}
For simplicity,
we consider the case where $\bm{U} \in \R^{I^j \times k}$ 
(resp. $\bm{V} \in \R^{(\frac{I^{\pi}}{I^j}) \times k}$) has only one single column 
$\bm{u} \in \R^{I^j}$
(resp. $\bm{v} \in \R^{\frac{I^{\pi}}{I^j}}$).	
We need to fold
$\bm{u} \bm{v}^{\top}$ along with the $j$th
mode
and then unfold it along its $i$th mode.
Let us consider the structure of $\ten{X} = (\bm{u} \bm{v}^{\top})^{\ip{j}}$,
we can express it as
\begin{align*}
\ten{X}_{\ip{j}} =
\big[ 
\mathbf{u} \bm{v}_1^{\top},
...,
\mathbf{u} \bm{v}_{\frac{I^{\pi}}{(I^i I^j)}}^{\top}
\big]
\in \R^{I^j \times \frac{I^{\pi}}{I^j}},
\end{align*}
where $\bm{v} = [ \bm{v}_1; \dots; \bm{v}_{\frac{I^{\pi}}{(I^i I^j)}} ]$ with each $\bm{v}_p \in \R^{I^i}$.
When unfolding $\ten{X}$ with the $i$th mode,
the unfolding matrix is 
\begin{align}
\big[ 
\bm{v}_1 \bm{u}^{\top}, \dots, 
\bm{v}_{\frac{I^{\pi}}{(I^i I^j)}} \bm{u}^{\top} 
\big] 
\in \R^{I^i \times \frac{I^{\pi}}{I^i}}.
\label{eq:temp3}
\end{align}
Thus,
\begin{eqnarray}
\bm{a}^{\top} \big[ \bm{v}_1 \bm{u}^{\top}, \dots, \bm{v}_{I^3} \bm{u}^{\top} \big]  
& = & 
\big[
(\bm{a}^{\top} \bm{v}_1) \bm{u}^{\top}, ...,
(\bm{a}^{\top} \bm{v}_{I^3}) \bm{u}^{\top} \big],
\nonumber
\\
& = &
\left(  \bm{a}^{\top} \text{mat}\left( \bm{v}_p; I^i, \bar{I}^{ij} \right) \right) 
\otimes
\bm{u}^{\top}.
\label{eq:temp4}
\end{eqnarray}
Similarly,
let $\bm{b} = \big[ \bm{b}_1; \dots; \bm{b}_{\frac{I^{\pi}}{I^i I^j}} \big]$, where each $\bm{b}_p \in \R^{I^j}$.
From \eqref{eq:temp3},
we have
\begin{eqnarray}
\big[ 
\bm{v}_1 \bm{u}^{\top}, \dots, 
\bm{v}_{\frac{I^{\pi}}{(I^i I^j)}} \bm{u}^{\top} 
\big] 
	\bm{b}
& = & \sum\nolimits_{j = 1}^{\frac{I^{\pi}}{I^i I^j}}
\bm{v}_i (\bm{u}^{\top} \bm{b}_i),
\nonumber \\
& = & \big[ \bm{v}_1; \dots; \bm{v}_{\frac{I^{\pi}}{I^i I^j}} \big]
\begin{bmatrix}
\bm{u}^{\top} \bm{b}_1
\\
\vdots
\\
\bm{u}^{\top} \bm{b}_{ \frac{I^{\pi}}{I^i I^j} }
\end{bmatrix},
\nonumber \\
& = & \big[ \bm{v}_1; \dots; \bm{v}_{ \frac{I^{\pi}}{I^i I^j} } \big]
\big[ \bm{b}_1; \dots; \bm{b}_{ \frac{I^{\pi}}{I^i I^j} } \big]^{\top}
\bm{u},
\nonumber \\
& = & 
\text{mat}\left( \bm{v}; I^i, \bar{I}^{ij} \right) 
\text{mat}\left( \bm{b}; \bar{I}^{ij}, I^j \right) 
\bm{u}.
\label{eq:temp5}
\end{eqnarray}
When $\bm{U}$ (resp. $\bm{V}$) has $k$ columns,
combining with the fact that
$\bm{U} \bm{V}^{\top} = \sum_{p = 1}^k \bm{u}_p \bm{v}_p^{\top}$
with \eqref{eq:temp4} and \eqref{eq:temp5},
we obtain \eqref{eq:mulv1} and \eqref{eq:mulv2}.
\end{proof}


\subsection{Proposition~\ref{pr:reg}}

\begin{proof}
Define $\bar{\lambda}_d = \lambda_d / \tau$,
then
\begin{align}
& \sum\nolimits_{d = 1}^D 
\min\nolimits_{ \vect{X}_d }
\frac{1}{2} \NM{\vect{X}_d - \ten{Z}_{\ip{d}}}{F}^2 
+ \bar{\lambda}_d \, \phi(\vect{X}_d),
\notag
\\
& \!\!\! = \min\nolimits_{\{ \vect{X}_d \}}
\frac{D}{2} \NM{\ten{Z}}{F}^2
- \big< \ten{Z}, \sum\nolimits_{d = 1}^D  \vect{X}_d^{\ip{d}} \big> 
+ \frac{D}{2} \sum\nolimits_{d = 1}^D  \NM{\vect{X}_d}{F}^2
+ \sum\nolimits_{d = 1}^D \bar{\lambda}_d \phi(\vect{X}_d),
\notag
\\
& \!\!\! = \min\nolimits_{\{ \vect{X}_d \}} 
\frac{D}{2}
\NM{\ten{Z} \! - \! \sum\nolimits_{d = 1}^D  \vect{X}_{d}^{\ip{d}} }{F}^2 
\!\!\! - \frac{D}{2}\NM{ \sum\nolimits_{d = 1}^D \vect{X}_{d}^{\ip{d}} }{F}^2 
\!\!\! + \sum\nolimits_{d = 1}^D 
\big[ \frac{1}{2} \NM{\vect{X}_d}{F}^2 
\! + \! \bar{\lambda}_d \phi (\vect{X}_d) \big].
\label{eq:app11}
\end{align}
Next, 
we
introduce an extra
parameter as 
$\ten{X} = \sum\nolimits_{d = 1}^D \vect{X}_{d}^{\ip{d}}$,
and express \eqref{eq:app11} as
\begin{align}
& \min\nolimits_{ \{\vect{X}_{d} \} :
\ten{X} = \sum\nolimits_{d = 1}^D \vect{X}_{d}^{\ip{d}}} 
\frac{D}{2}
\NM{\ten{Z} - \ten{X} }{F}^2 
- \frac{D}{2}\NM{ \sum\nolimits_{d = 1}^D \vect{X}_{d}^{\ip{d}} }{F}^2 
+  \sum\nolimits_{d = 1}^D 
\left[ \frac{1}{2} \NM{\vect{X}_d}{F}^2 
+  \bar{\lambda}_d \phi (\vect{X}_d) \right],
\notag
\\
& = \!
\min_{\ten{X}}  \left\{
\frac{D}{2}
\NM{\ten{Z} - \ten{X} }{F}^2 
+
\min_{ \{\vect{X}_{d} \} :
\sum\nolimits_{d = 1}^D \vect{X}_{d}^{\ip{d}} = \ten{X} } 
\sum\nolimits_{d = 1}^D 
\left[ \frac{1}{2} \NM{\vect{X}_d}{F}^2 
\! + \! \bar{\lambda}_d \phi (\vect{X}_d) \right]
\! - \! \frac{D}{2}\NM{ \ten{X} }{F}^2 \right\}. 
\label{eq:app16}
\end{align}
We transform the above equation as 
\begin{align*}
\min\nolimits_{ \ten{X} }
\frac{1}{2} \NM{\ten{Z}  - \ten{X} }{F}^2
+ \frac{1}{\tau} \bar{g}_{\tau}(\ten{X})
= \Px{\frac{\bar{g}_{\tau}}{\tau}}{\ten{X}},
\end{align*}
where $\bar{g}_{\tau}(\ten{X})$ is defined as
\begin{align}
\bar{g}_{\tau}(\ten{X})
& = 
\tau
\left[ 
\min\nolimits_{ \{ \vect{X}_{d} \} } 
\sum\nolimits_{d = 1}^D 
\big( \frac{1}{2} \NM{\vect{X}_d}{F}^2 
+  \bar{\lambda}_d \phi (\vect{X}_d) \big)
- \frac{D}{2}\NM{ \ten{X} }{F}^2
\right],
\label{app:barg}
\\
&
\text{\;s.t.\;} 
\sum\nolimits_{d = 1}^D \vect{X}_{d}^{\ip{d}} = \ten{X}.
\notag
\end{align}
Thus, there exists $\bar{g}_{\tau}$ such that
$\Px{\frac{\bar{g}_{\tau}}{\tau}}{\ten{Z}}
 = \sum\nolimits_{i = 1}^D \big[ \Px{\bar{\lambda}_d \phi}
{ \left[ \ten{Z} \right]_{\ip{i}} } \big]^{\ip{i}}$.
\end{proof}


\subsection{Proposition~\ref{pr:bnd}}

Let $g(\ten{X}) =  \sum_{d = 1}^{D} \lambda_i \phi(\ten{X}_{\ip{d}})$.
Before proving Proposition~\ref{pr:bnd},
we first extend Proposition 2 in~\citep{zhong2014gradient} in the following auxiliary Lemma.

\subsubsection{Auxiliary Lemma}

\begin{lemma}\label{app:lem1}
$0 \le g(\ten{X}) - \bar{g}_{\tau}(\ten{X}) \le \frac{\kappa_0^2}{ 2 \tau } \sum\nolimits_{d = 1}^D \lambda_d^2$.
\end{lemma}

\begin{proof}
From the definition of $\bar{g}_{\tau}$ in \eqref{app:barg},
if $\ten{X}  = \vect{X}^{\ip{1}}_1 = ... = \vect{X}^{\ip{D}}_D$,
we have
\begin{align*}
\bar{g}_{\tau}(\ten{X})
& \le \tau 
\big( \sum\nolimits_{d = 1}^D \big( \frac{1}{2} \NM{\vect{X}_d^{\ip{d}}}{F}^2 \! 
+ \bar{\lambda}_d \phi (\vect{X}_d) \big) 
\! - \! \frac{D}{2}\NM{ \ten{X} }{F}^2 \big),
\\
& = \sum\nolimits_{d = 1}^D \lambda_d \phi(\vect{X}_d)
 = \sum\nolimits_{d = 1}^D \lambda_d \phi(\ten{X}_{\ip{d}}) =  g(\ten{X}).
\end{align*}
Thus,
$ g(\ten{X}) - \bar{g}_{\tau}(\ten{X}) \ge 0$.
Next, we prove the ``$\le$'' part in the Lemma.
Note that
\begin{align}
& \sup\nolimits_{\vect{X}_d}
\lambda_d \phi(\vect{X}_d)
- \tau \min\nolimits_{\vect{Y}} 
\big( \frac{1}{2} \NM{\vect{Y} - \vect{X}_d }{F}^2
+ \bar{\lambda}_d \phi(\vect{Y}) \big),
\notag
\\
& = \sup\nolimits_{\vect{X}_d, \vect{Y}}
\lambda_d \phi(\vect{X}_d)
- \frac{\tau}{2} \NM{\vect{Y} - \vect{X}_d }{F}^2
- \lambda_d \phi(\vect{Y}).
\label{eq:app2}
\end{align}
Since $\phi$ is $\kappa_0$-Lipschitz continuous,
let $\alpha = \NM{\vect{Y} - \vect{X}_d }{F}$,
we have 
\begin{align}
\ten{\eqref{eq:app2}} 
& =
\sup\nolimits_{\vect{X}_d, \vect{Y}}
\lambda_d \left[ \phi(\vect{X}_d) - \phi(\vect{X}) \right] 
- \frac{\tau}{2} \NM{\vect{Y} - \vect{X}_d }{F}^2,
\notag
\\
& \le \sup\nolimits_{\vect{X}_d, \vect{Y}}
\lambda_d \kappa_0 \NM{\vect{Y} - \vect{X}_d }{F}
- \frac{\tau}{2} \NM{\vect{Y} - \vect{X}_d }{F}^2,
\notag
\\
& = \sup\nolimits_{\alpha} \big[ \lambda_d \kappa_0 \alpha - \frac{\tau}{2} \alpha^2 \big]
= \sup\nolimits_{\alpha} - \frac{1}{2}\big[ \alpha - \frac{\lambda_d \kappa_0}{\tau} \big]^2 + \frac{\lambda_d^2 \kappa_0^2}{2}
\le  \frac{\lambda_d^2 \kappa_0^2}{2 \tau}.
\label{eq:app3}
\end{align}
Next, we have
\begin{align}
g(\ten{X}) - \bar{g}_{\tau}(\ten{X})
& \le g(\ten{X})
- \tau 
\big(\min\nolimits_{\ten{Y}}
\frac{1}{2} \NM{\ten{X} - \ten{Y}}{F}^2 
+ \frac{1}{\tau} \bar{g}_{\tau}(\ten{Y})
\big),
\label{eq:app12}
\\
& = 
\sum\nolimits_{d = 1}^D
\lambda_d \phi( \ten{X}_{\ip{d}} )
- \tau \sum\nolimits_{d = 1}^D
\big(\min\nolimits_{\{ \vect{Y}_d \}}
\frac{1}{2} \NM{\ten{X}_{\ip{d}} - \vect{Y}_{d}}{F}^2 
+ \frac{\lambda_d}{\tau} \phi(\vect{Y}_{d})
\big),
\label{eq:app13}
\\
& \le \sup\nolimits_{\ten{X}}
\sum\nolimits_{d = 1}^D
\lambda_d \phi( \ten{X}_{\ip{d}} )
- \tau \sum\nolimits_{d = 1}^D
\big(\min\nolimits_{\{ \vect{Y}_d \}}
\frac{1}{2} \NM{\ten{X}_{\ip{d}} - \vect{Y}_{d}}{F}^2 
+ \frac{\lambda_d}{\tau} \phi(\vect{Y}_{d})
\big),
\notag
\\
& \le  \sum\nolimits_{d = 1}^D \frac{\lambda_d^2 \kappa_0^2}{2 \tau}.
\label{eq:app15}
\end{align}
Note that \eqref{eq:app12} comes from the fact that
\begin{align*}
\min\nolimits_{\ten{Y}}
\frac{1}{2} \NM{\ten{X} - \ten{Y}}{F}^2 
+ \frac{1}{\tau} \bar{g}_{\tau}(\ten{Y})
\le 
\frac{1}{2} \NM{\ten{X} - \ten{X}}{F}^2 
+ \frac{1}{\tau} \bar{g}_{\tau}(\ten{X})
= \frac{1}{\tau} \bar{g}_{\tau}(\ten{X}),
\end{align*}
then \eqref{eq:app13} is from the definition of $\bar{g}_{\tau}$ in Proposition~\ref{pr:reg},
and \eqref{eq:app15} is from \eqref{eq:app3}.
\end{proof}

\subsubsection{Proof of Proposition~\ref{pr:bnd}}

\begin{proof}
From Lemma~\ref{app:lem1},
we have
\begin{align*}
\min_{\ten{X}} F(\ten{X}) - \min_{\ten{X}} F_{\tau} (\ten{X})
\ge \min_{\ten{X}} F (\ten{X}) - F_{\tau} (\ten{X})
= g(\ten{X}) - \bar{g}_{\tau}(\ten{X})
\ge 0.
\end{align*}
Let $\ten{X}_1 = \arg\min\nolimits_{\ten{X}} F(\ten{X})$
and $\ten{X}_{\tau} = \arg\min\nolimits_{\ten{X}} F_{\tau}(\ten{X})$.
Then,
we have
\begin{align*}
\min_{\ten{X}} F(\ten{X}) 
\! - \! \min_{\ten{X}} F_{\tau} (\ten{X}) 
= \! F (\ten{X}_1) \! - \! F_{\tau} (\ten{X}_{\tau})
\le \! F (\ten{X}_{\tau}) \! - \! F_{\tau} (\ten{X}_{\tau})
\! = \! g(\ten{X}_{\tau}) 
\! - \! \bar{g}_{\tau}(\ten{X}_{\tau})
\! \le \! \frac{\kappa_0^2}{ 2 \tau} 
\sum\nolimits_{d = 1}^D \lambda_d^2.
\end{align*}
Thus,
$0 \le \min F - \min F_{\tau} \le \frac{\kappa_0^2}{ 2 \tau} \sum\nolimits_{d = 1}^D \lambda_d^2$.
\end{proof}


\subsection{Proposition~\ref{pr:criti}}
\label{app:pr:criti}

\begin{proof}
The proof of this proposition can also be found in~\citep{zhong2014gradient},
we add one here for the completeness.
Recall that
\begin{align*}
\Px{ \frac{ \bar{g}_{\tau} }{ \tau } }{ \tilde{\ten{X}} - \nabla f( \tilde{\ten{X}} ) / \tau }
= \arg\min_{\ten{X}}
\frac{1}{2}
\NM{ \ten{X} - \left( \tilde{\ten{X}} - \frac{1}{ \tau } \nabla f(\tilde{\ten{X}}) \right) }{F}^2
+ \frac{1}{\tau} \bar{g}_{\tau}(\ten{X}).
\end{align*}
Let $\ten{Z} = \Px{ \frac{ \bar{g}_{\tau} }{ \tau } }{ \tilde{\ten{X}} - \nabla f( \tilde{\ten{X}} ) / \tau }$.
Thus,
\begin{align*}
0
\in 
\ten{Z} - 
\left( \tilde{\ten{X}} - \frac{1}{\tau} \nabla f( \tilde{\ten{X}} ) \right) 
+ \frac{1}{\tau} \partial \bar{g}_{\tau}(\ten{X}).
\end{align*}
When $\ten{Z} = \tilde{\ten{X}}$,
we have 
$0 \in \nabla f( \tilde{\ten{X}} ) + \partial \bar{g}_{\tau}(\ten{X})$.
Thus,
$\tilde{\ten{X}}$ is a critical point of $F_{\tau}$.
\end{proof}


\subsection{Theorem~\ref{thm:conv}}
\label{app:proof:thm:conv}

First,
we introduce the following Lemmas,
which are basic properties for the proximal step.

\begin{lemma}\citep{parikh2013proximal} 
\label{lem:prox}
Let $\tau > \rho + D \kappa_0$ and $\eta = \tau - \rho + D \kappa_0$.
Then,
$F_{\tau}(\Px{\frac{\bar{g}_{\tau}}{\tau}}{\ten{X}})	
\le F_{\tau}(\ten{X}) - \frac{\eta}{2} \big\| \ten{X} - \Px{\frac{\bar{g}_{\tau}}{\tau}}{\ten{X}} \big\|_F^2$.
\end{lemma}

\begin{lemma}\citep{parikh2013proximal} 
\label{lem:crti}
If $\ten{X} \! = \! \Px{\frac{\bar{g}_{\tau}}{\tau}}{\ten{X} \! - \! \frac{1}{\tau} \nabla f(\ten{X})}$,
then $\ten{X}$ is a critical point of $F_{\tau}$.
\end{lemma}

\begin{lemma}\citep{hare2009computing}
\label{lem:cont}
The proximal map $\Px{\frac{\bar{g}_{\tau}}{\tau}}{\ten{X}}$ is continuous. 
\end{lemma}


\begin{proof}(\textit{of Theorem~\ref{thm:conv}})
Recall that $\Px{\frac{\bar{g}_{\tau}}{\tau}}{\ten{X}} = \sum\nolimits_{i = 1}^D \Px{\frac{\lambda_i \phi}{\tau}}{\ten{X}_{\ip{i}}}$.
From Lemma~\ref{lem:prox},
\begin{itemize}[leftmargin=*]
\item If step~7 is performed, we have
\begin{align}
F_{\tau}(\ten{X}_{t + 1})
\le F_{\tau}(\ten{V}_t)
- \frac{\eta}{2}\NM{\ten{X}_{t + 1} - \ten{V}_t}{F}^2
 \le F_{\tau}(\ten{X}_t)
- \frac{\eta}{2}\NM{\ten{X}_{t + 1} - \ten{X}_t}{F}^2.
\label{eq:temp6}
\end{align}

\item If step~5 is performed,
\begin{align}
F_{\tau}(\ten{X}_{t + 1})
 \le  F_{\tau}(\ten{V}_t)  -  \frac{\eta}{2}\NM{\ten{X}_{t + 1}  -  \ten{V}_t}{F}^2
& 
\le  F_{\tau}(\bar{\ten{X}}_t)  -  \frac{\eta}{2}\NM{\ten{X}_{t + 1}  -  \bar{\ten{X}}_t}{F}^2,
\notag
\\
& 
\le  F_{\tau}(\ten{X}_t)  -  \frac{\eta}{2}\NM{\ten{X}_{t + 1}  -  \bar{\ten{X}}_t}{F}^2
.
\label{eq:temp7}
\end{align}
\end{itemize}
Combining \eqref{eq:temp6} and \eqref{eq:temp7},
we have
\begin{align}
\frac{2}{\eta} ( F_{\tau}(\ten{X}_1) - F_{\tau}(\ten{X}_{T + 1}) )
\ge 
\sum\nolimits_{j \in \chi_1(T)} 
\NM{\ten{X}_{t + 1} - \bar{\ten{X}}_t}{F}^2
+  \sum\nolimits_{j \in \chi_2(T)} \NM{\ten{X}_{t + 1} - \ten{X}_t}{F}^2,
\label{eq:temp8}
\end{align}
where $\chi_1(T)$ and $\chi_2(T)$ are a partition of $\{ 1, ..., T \}$
such that when $j \in \chi_1(T)$ step~5 is performed,
and when $j \in \chi_2(T)$ step~7 is performed.
As $F_{\tau}$ is bounded from below
and $\lim_{ \NM{\ten{X}}{F} \rightarrow \infty } F_{\tau}(\ten{X}) = \infty$,
taking $T = \infty$ in \eqref{eq:temp8},
we have
\begin{align*}
\sum\nolimits_{j \in \chi_1(\infty)} \NM{\ten{X}_{t + 1} - \ten{Y}_t}{F}^2
+ \sum\nolimits_{j \in \chi_2(\infty)} \NM{\ten{X}_{t + 1} - \ten{X}_t}{F}^2 = c,
\end{align*}
where
$c \le \frac{2}{\eta} \left[ F_{\tau}(\ten{X}_1) - F_{\tau}^{\text{min}} \right]$
is a positive constant.
Thus,
the sequence $\{ \ten{X}_t \}$ is bounded,
and it must have limit points.
Besides,
one of the following three cases must hold.
\begin{enumerate}[leftmargin=*]
\item $\chi_1(\infty)$ is finite, $\chi_2(\infty)$ is infinite.
Let $\tilde{\ten{X}}$ be a limit point of $\{ \ten{X}_t \}$,
and $\{ \ten{X}_{j_t} \}$ be a subsequence that converges to $\tilde{\ten{X}}$.
In this case,
on using Lemma~\ref{lem:cont},
we have
\begin{align*}
\lim\limits_{j_t \rightarrow \infty} \NM{\ten{X}_{j_t + 1} - \ten{X}_{j_t}}{F}^2
& = \lim\limits_{j_t \rightarrow \infty} \NM{\Px{\frac{\bar{g}_{\tau}}{\tau}}{\ten{X}_{j_t} - \frac{1}{\tau} \nabla f(\ten{X}_{j_t})} - \ten{X}_{j_t}}{F}^2,
\\
& =
\NM{\Px{\frac{\bar{g}_{\tau}}{\tau}}{\tilde{\ten{X}} - \frac{1}{\tau} \nabla f(\tilde{\ten{X}})} - \tilde{\ten{X}}}{F}^2 = 0.
\end{align*}
Thus,
$\tilde{\ten{X}} = \Px{\frac{\bar{g}_{\tau}}{\tau}}{\tilde{\ten{X}} - \frac{1}{\tau} \nabla f(\tilde{\ten{X}})}$,
and $\tilde{\ten{X}}$ is a critical point of $F_{\tau}$ from Lemma~\ref{lem:crti}.

\item $\chi_1(\infty)$ is infinite, $\chi_2(\infty)$ is finite.
Let $\tilde{\ten{X}}$ be a limit point of $\{ \ten{X}_t \}$,
and $\{ \ten{X}_{j_t} \}$ be a subsequence that converges to $\tilde{\ten{X}}$.
In this case,
we have
\begin{align*}
\lim\limits_{j_t \rightarrow \infty} \NM{\ten{X}_{j_t + 1} - \ten{Y}_{j_t}}{F}^2
& = \lim\limits_{j_t \rightarrow \infty} \NM{\Px{\frac{\bar{g}_{\tau}}{\tau}}{\ten{X}_{j_t} - \frac{1}{\tau} \nabla f(\ten{X}_{j_t})} - \ten{Y}_{j_t}}{F}^2,
\\
& =
\NM{\Px{\frac{\bar{g}_{\tau}}{\tau}}{\tilde{\ten{X}} - \frac{1}{\tau} \nabla f(\tilde{\ten{X}})} - \tilde{\ten{X}}}{F}^2 = 0.
\end{align*}
Thus,
$\tilde{\ten{X}} = \Px{\frac{\bar{g}_{\tau}}{\tau}}{\tilde{\ten{X}} - \frac{1}{\tau} \nabla f(\tilde{\ten{X}})}$,
and $\tilde{\ten{X}}$ is a critical point of $F_{\tau}$ from Lemma~\ref{lem:crti}.

\item Both $\chi_1(\infty)$ and $\chi_2(\infty)$ are infinite.
From the above cases,
we can see that either $\chi_1(\infty)$ or $\chi_2(\infty)$ is infinite,
and
limit points are also the critical points of $F_{\tau}$.
\end{enumerate}
Thus, all limit points of $\{ \ten{X}_t \}$ are critical points of $F_{\tau}$.
\end{proof}

\subsection{Corollary~\ref{cor:rate}}

This corollary can be easily derived from 
the proof of Theorem~\ref{thm:conv}.

\begin{proof}
Since $\ten{X}_{t + 1} = \Px{\frac{\bar{g}_{\tau}}{\tau}}{\ten{V}_t - \frac{1}{\tau} \nabla f(\ten{V}_t) }$,
conclusion (i) directly follows from Lemma~\ref{lem:crti}.
From~\eqref{eq:temp8},
we have
\begin{align*}
\min\nolimits_{1,\dots,T}
\NM{\ten{X}_{t + 1}  -  \ten{V}_t}{F}^2
& \le 
\frac{1}{T}
\sum\nolimits_{t = 1 \dots T} 
\NM{\ten{X}_{t + 1}  -  \ten{V}_t}{F}^2,
\\
& \le 
\frac{2}{\eta T} ( F_{\tau}(\ten{X}_1)  -  F_{\tau}(\ten{X}_{T + 1}) )
\le  
\frac{2}{\eta T} ( F_{\tau}(\ten{X}_1)  -  F_{\tau}^{\min} ).
\end{align*}
Thus,
we obtain Conclusion (ii).
\end{proof}

\subsection{Theorem~\ref{thm:klrate}}
\label{app:thm:klrate}

We first bound $\partial F_{\tau}$ in Lemma~\ref{lem:subnorm},
then prove Theorem~\ref{thm:klrate}.

\begin{lemma}\label{lem:subnorm}
	For iterations in Algorithm~\ref{alg:nort}, we have
	$\min\nolimits_{\ten{U}_t \in \partial F_{\tau}(\ten{X}_t)  }
	\NM{\ten{U}_t}{F}
	\! \le \! (\tau + \rho) \NM{\ten{X}_{t + 1} \! - \! \ten{V}_t}{F}$.
\end{lemma}

\begin{proof}
	Since $\ten{X}_{t + 1}$ is generated from the proximal step,
	i.e.,
	$\ten{X}_{t + 1} = \Px{\frac{\bar{g}_{\tau}}{\tau}}{ \ten{V}_t - \frac{1}{\tau} \nabla f(\ten{V}) }$,
	from its optimality condition,
	we have 
	\begin{align*}
	\ten{X}_{t + 1}
	- \big( \ten{V}_t - \frac{1}{\tau} \nabla f(\ten{V}_t) \big) + 
	\frac{1}{\tau} \partial \bar{g}_{\tau}(\ten{X}_{t + 1})
	\ni \bm{0}.
	\end{align*}
	Let
	$	\ten{U}_t = 
	\tau \left[ \ten{X}_{t + 1} - \ten{V}_t \right] 
	- \left[ \nabla f(\ten{V}_t) - \nabla f(\ten{X}_{t + 1}) \right]$.
	We have
	\begin{align*}
	\partial F_{\tau}(\ten{X}_{t + 1}) = 
	\left[ \nabla f(\ten{X}_{t + 1}) + \partial \bar{g}_{\tau}(\ten{X}_{t + 1})
	\right] 
	\in  \ten{U}_t.
	\end{align*}
	Thus,
	$\NM{\ten{U}_t}{F}
	\le \tau \NM{\ten{X}_{t + 1} \! - \! \ten{V}_t}{F}
	\! + \! \NM{\nabla f(\ten{V}_t) \! - \! \nabla f(\ten{X}_{t + 1})}{F}
	\le \left( \tau + \rho \right) \NM{\ten{X}_{t + 1} - \ten{V}_t}{F}$.
\end{proof}
%

\begin{proof}(\textbf{of Theorem~\ref{thm:klrate}}).
From Theorem~\ref{thm:conv},
we have
$\lim\limits_{T \rightarrow \infty}
F_{\tau}(\ten{X}_t) = F_{\tau}^{\min}$.
Then, 
from Lemma~\ref{lem:subnorm},
we have
\begin{align*}
\lim\limits_{t \rightarrow \infty}
\min\nolimits_{\ten{U}_t \in \partial F_{\tau}(\ten{X}_t)  } \!
\NM{\ten{U}_t}{F} 
\! \le \! 
\lim\limits_{t \rightarrow \infty}
(\tau \! + \! \rho) \NM{\ten{X}_{t + 1} \! - \! \ten{V}_t}{F}
\! = \! 0.
\end{align*}
Thus,
for any $\epsilon$, $c > 0$ and $t > t_0$ where $t_0$ is a sufficiently large positive integer,
we have
\begin{align*}
\ten{X}_t
\in 
\left\lbrace 
\ten{X} \,|
\min\nolimits_{\ten{U} \in \partial F_{\tau}(\ten{X}) }
\NM{\ten{U}}{F} \le \epsilon, 
F^{\min}_{\tau} 
< F_{\tau}(\ten{X}) 
< F_{\tau}^{\min} + c
\right\rbrace .
\end{align*}
Then,
the uniformized KL property implies for all $t \ge t_0$,
\begin{align}
1
& \le 
\psi'\left( F_{\tau} (\ten{X}_{t + 1})  -  F_{\tau}^{\min} \right) 
\min\nolimits_{\ten{U}_t \in \partial F_{\tau}(\ten{X}_t) }
\NM{\ten{U}_t}{F},
\notag
\\
& =  \psi'\left( F_{\tau} (\ten{X}_{t + 1}) - F_{\tau}^{\min} \right) 
(\tau  +  \rho) \NM{\ten{X}_{t + 1}  -  \ten{V}_t}{F}.
\label{eq:app4}
\end{align}
Moreover,
from Lemma~\ref{lem:prox}, we have
\begin{align}
\NM{\ten{X}_{t + 1} \! - \! \ten{V}_t}{F}^2
\le \frac{2}{\eta} \left[ F_{\tau}(\ten{V}_t) - F_{\tau}(\ten{X}_{t + 1}) \right].
\label{eq:app5}
\end{align}
Let $r_t = F_{\tau} (\ten{X}_t) - F_{\tau}^{\min}$,
we have
\begin{align}
r_t - r_{t + 1}
& = F_{\tau} (\ten{X}_t) - F_{\tau}^{\min}
- \left[ F_{\tau} (\ten{X}_{t + 1}) - F_{\tau}^{\min} \right],
\notag
\\ 
& \ge F_{\tau} (\ten{V}_t) - F_{\tau}^{\min}
- \left[ F_{\tau} (\ten{X}_{t + 1}) - F_{\tau}^{\min} \right]
= F_{\tau} (\ten{V}_t) - F_{\tau} (\ten{X}_{t + 1}).
\label{eq:app6}
\end{align}
Combine \eqref{eq:app4}, \eqref{eq:app5} and \eqref{eq:app6},
we have
\begin{align}
1 
& \le \left[ \psi'( r_t ) \right]^2
(\tau + \rho)^2 \NM{\ten{X}_{t + 1} - \ten{V}_t}{F}^2,
\notag
\\
& \le 
\frac{2 (\tau + \rho)^2}{\eta} 
\left[ \psi'( r_t ) \right]^2
\left[ F_{\tau}(\ten{V}_t) - F_{\tau}(\ten{X}_{t + 1}) \right]
\le 
\frac{2 (\tau + \rho)^2}{\eta}  \left[ \psi'( r_{t + 1} ) \right]^2 (r_t - r_{t + 1}).
\label{eq:app7}
\end{align}
Since 
$\phi(\alpha) = \frac{C}{x} \alpha^{x}$, then $\phi'(\alpha) = C \alpha^{x - 1}$,
\eqref{eq:app7} becomes
$1 \le d_1 C^2 r_{t + 1}^{2 x - 2} (r_{t} - r_{t + 1})$,
where $d_1 = \frac{2 (\tau + \rho)^2}{\eta}$.
Finally,
it is shown in~\citep{bolte2014proximal,li2015accelerated,Li2017ada} that
the sequence $\{ r_t \}$ satisfying the above inequality, convergence to zero with
different rates stated in the Theorem. 
\end{proof}

%

%

\subsection{Lemma~\ref{lem:rsc:hold}}

First, 
we introduce the following Lemma.

\begin{lemma}[Theorem~1 in~\citep{negahban2012restricted}] \label{lem:rschold}
	Consider a matrix $\vect{X} \in \R^{m \times n}$.
	Let $d = \frac{1}{2} \left(m + n \right)$ and $m_\text{rank}
	( \vect{X} )  = \frac{ \NM{ \vect{X} }{*} }{ \NM{ \vect{X} }{ F } }$.
	Define a constraint set $\mathcal{C}$ 
	(with parameters 
	$c_0, n$)
	as
	\begin{align*}
		\mathcal{C}
		(n, c_0)
		= 
		\left\lbrace 
		\vect{X} \in \R^{m \times n},
		\vect{X} \not= 0
		\;|\;
		m_\text{spike}( \vect{X} )
		\cdot 
		m_\text{rank} ( \vect{X} ) 
		\le
		\frac{1}{c_0 L}
		\sqrt{\frac{ \NM{ \vect{\Omega} }{1} }{ d \log d }}
		\right\rbrace,
	\end{align*}
	where $L$ is a constant.
	There are 
	constants $(c_0, c_1, c_2, c_3)$ such that when $\NM{\vect{\Omega}}{1} > c_3 \max (d \log d)$,
	we have
	\begin{align*}
		\frac{ \NM{ \SO{\vect{X}} }{F} }{ \NM{\vect{\Omega}}{1} }
		\ge
		\frac{1}{8}
		\NM{ \vect{X} }{F}
		\left\lbrace 
		1 - \frac{ 128 L \cdot m_\text{spike}( \vect{X} ) }{\sqrt{ \NM{ \vect{\Omega} }{1} }}
		\right\rbrace,
		\quad
		\forall
		\vect{X} \in 
		\mathcal{C} (\NM{ \vect{\Omega} }{1}, c_0),
	\end{align*}
	with a high probability greater at least of
	$1 -  c_1 \exp( - c_2 d \log d)$.
\end{lemma}

\begin{proof} (of
Lemma~\ref{lem:rsc:hold})
For a $M$th-order tensor $\Delta$,
using Lemma~\ref{lem:rschold} 
on each unfolded matrix $\Delta_{\ip{i}}$ ($i = 1, \dots, M$),
we have
\begin{align}
\frac{ \NM{ \SO{ \Delta_{\ip{i}} } }{F} }{ \NM{\vect{\Omega}}{1} }
\ge
\frac{1}{8}
\NM{ \Delta }{F}
\left\lbrace 
1 - \frac{ 128 L \cdot m_\text{spike}( \Delta_{\ip{i}} ) }{\sqrt{ \NM{ \vect{\Omega} }{1} }}
\right\rbrace,
\label{eq:app14}
\end{align}
for all
$\Delta_{\ip{i}} \in 
\mathcal{C}^i (\NM{ \vect{\Omega} }{1}, c_0)$.
Note that
the L.H.S. of \eqref{eq:app14} is the same for 
all $i = 1, ..., M$.
Thus,
to ensure \eqref{eq:app14} holds
for all $\Delta_{\ip{i}}$,
we need to take the intersection of all $\Delta_{\ip{i}}$,
which leads to
\begin{align}
\left\lbrace 
\ten{X} \in \R^{I_1 \times ... \times I_M},
\ten{X} \not= 0
\;|\;
m_\text{spike}(\ten{X})
\cdot 
m_\text{rank} ( \ten{X}_{\ip{i}} ) 
\le
\frac{1}{c_0 L}
\sqrt{\frac{ \NM{ \vect{\Omega} }{1} }{ d_i \log d_i }}
\right\rbrace.
\label{eq:app17}
\end{align}
Recall that 
$m_\text{rank}(\ten{X})
= \frac{1}{M} \sum\nolimits_{i = 1}^M m_\text{rank} ( \ten{X}_{\ip{i}} ) $
as defined in \eqref{eq:mrank}.
Thus,
$\tilde{\mathcal{C}}
(n, c_0)$ is a subset of \eqref{eq:app17}.
As a result, 
\eqref{eq:rschold} holds.
\end{proof}

\subsection{Theorem~\ref{thm:stat}}

Here,
we first introduce some
auxiliary lemmas in Appendix~\ref{app:auxlem},
which will be used to prove Theorem~\ref{thm:stat}
in Appendix~\ref{app:thm:stat}.

\subsubsection{Auxiliary Lemmas} 
\label{app:auxlem}

\begin{lemma}[Lemma 4 in~\citep{loh2015regularized}]\label{lem:app1}
	For $\kappa$ in Assumption~\ref{ass:kappa},
	we have
	\begin{itemize}[leftmargin = 25px]
		\item[(i).]
		The function $\alpha \rightarrow \frac{\kappa(\alpha)}{\alpha}$ is nonincreasing on $\alpha > 0$;
		
		\item[(ii).] The derivative of $\kappa$ is upper bounded by $\kappa_0$;
		
		\item[(iii).]
		The function $\alpha \rightarrow \kappa(\alpha) + \frac{\alpha^2 c}{2}$ is convex
		only when $c \ge \kappa_0$;
		
		\item[(iv).]
		$\lambda |\alpha| \le \lambda \kappa( |\alpha| ) + \frac{\alpha^2 \kappa_0}{2}$.
	\end{itemize}
\end{lemma}

\begin{lemma}\label{lem:app4}
	$
	\left\langle \ten{X}, \ten{Y} \right\rangle
	\le 
	\min\nolimits_{i = 1, ..., K}
	\NM{ \ten{X}_{\ip{i}} }{\infty}
	\NM{ \ten{Y}_{\ip{i}} }{*}
	$.
\end{lemma}

\begin{proof}
	First,
	we have
	$\left\langle \ten{X}, \ten{Y} \right\rangle
	= \left\langle \ten{X}_{\ip{i}}, \ten{Y}_{\ip{i}} \right\rangle$
	for all
	$i \in \{ 1, \dots, M \}$.
	Then,
	since $\NM{\cdot}{\infty}$ and $\NM{\cdot}{*}$ are dual norm with each other,
	$\left\langle \ten{X}_{\ip{i}}, \ten{Y}_{\ip{i}} \right\rangle
	\le \NM{ \ten{X}_{\ip{i}} }{\infty}
	\NM{ \ten{Y}_{\ip{i}} }{*}$. 
	Thus, we have 
	$
	\left\langle \ten{X}, \ten{Y} \right\rangle
	\le 
	\min\nolimits_{i = 1, ..., K}
	\NM{ \ten{X}_{\ip{i}} }{\infty}
	\NM{ \ten{Y}_{\ip{i}} }{*}
	$.
\end{proof}

\begin{lemma}\label{lem:app6}
	For all $i \! \in \! \{1, ... m\}$, we have
	$\NM{\ten{X}}{F} \! \le \! \NM{\ten{X}_{\ip{i}}}{*}$
	and
	$\NM{\ten{X}_{\ip{i}}}{*} \! \le \! \sqrt{\min(I^i, \frac{I^{\pi}}{I^i})} \NM{ \ten{X} }{F}$.
\end{lemma}

\begin{proof}
	Note that
	$\NM{\ten{X}}{F} = \NM{ \ten{X}_{\ip{i}} }{F}$
	and $\NM{ \ten{X}_{\ip{i}} }{F} \le \NM{ \ten{X}_{\ip{i}} }{*}$,
	thus $\NM{\ten{X}}{F} \le \NM{\ten{X}_{\ip{i}}}{*}$.
	Then,
	since $\NM{\vect{X}}{*} \le \sqrt{\min(p,q)} \NM{\vect{X}}{F}$ for 
	a matrix $\vect{X}$ of size $p \times q$,
	we have 
	$\NM{\ten{X}_{\ip{i}}}{*} 
	\le \sqrt{\min(I^i, I^{\pi} / I^i)} \NM{ \ten{X}_{\ip{i}} }{F}$
	$= \sqrt{\min(I^i, I^{\pi} / I^i)} \NM{ \ten{X} }{F}$.
\end{proof}

\begin{lemma}\label{lem:app5}
	Define $h_i(\ten{X}) = \phi( \ten{X}_{\ip{i}} )$.
	Let $\varPhi_k(\bm{A})$ produce the best rank $k$ approximation to matrix $\bm{A}$
	and $\varPsi_k(\bm{A}) = \bm{A} - \varPhi_k(\bm{A})$.
	Suppose $\varepsilon_i > 0$ for $i \in \{ 1, ..., M \}$
	are constants such that
	$\varepsilon_i  h_i(\varPhi_{k_i}( \ten{A}_{\ip{i}} ))
	- h_i( \varPsi_{k_i}( \ten{A}_{\ip{i}} ) ) \ge 0$.
	Then,
	\begin{align}
	\varepsilon_i  h_i(\varPhi_{k_i}( \ten{A}_{\ip{i}} )) 
	-  h_i( \varPsi_{k_i}( \ten{A}_{\ip{i}} ) )
	\le \kappa_0
	( 
	\varepsilon_i \NM{\varPhi_{k_i}( \ten{A}_{\ip{i}} )}{*}
	- \NM{\varPsi_{k_i}( \ten{A}_{\ip{i}} )}{*}
	).
	\label{eq:vareps1}
	\end{align}
	Moreover,
	if $\ten{X}^*_{\ip{i}}$ is of rank $k_i$,
	for any tensor $\ten{X}$ satisfying
	$\varepsilon_i h_i( \ten{X}^*_{\ip{i}} )
	- h_i( \ten{X}_{\ip{i}} ) \ge 0$ and $\varepsilon_i > 1$,
	we have
	\begin{align}
	\varepsilon_i h_i( \ten{X}^*_{\ip{i}} ) 
	\! - \! h_i( \ten{X}_{\ip{i}} )
	\le  \kappa_0
	( 
	\varepsilon_i \NM{\varPhi_{k_i}( \ten{V}_{\ip{i}} )}{*}
	\! - \! \NM{\varPsi_{k_i}( \ten{V}_{\ip{i}} )}{*}
	),
	\label{eq:vareps2}
	\end{align}
	where $\ten{V} = \ten{X}^* - \ten{X}$.
\end{lemma}

\begin{proof}
	\textbf{We first prove \eqref{eq:vareps1}}.
	Let $h(\alpha) = \frac{\alpha}{\kappa(\alpha)}$ on $\alpha > 0$.
	From Lemma~\ref{lem:app1}, 
	we know $h(\alpha)$ is a non-decreasing function.
	Therefore,
	\begin{align}
	\NM{\varPsi_{k_i}( \ten{A}_{\ip{i}} )}{*}
	&
	= \sum\nolimits_{j = k_i + 1}
	\kappa\left( \sigma_j \left( \ten{A}_{\ip{i}} \right)  \right) 
	h \left( \sigma_j \left( \ten{A}_{\ip{i}} \right)  \right),
	\notag
	\\
	&
	\le 
	h \left( \sigma_1 \left( \ten{A}_{\ip{i}} \right)  \right) 
	\sum\nolimits_{j = k_i + 1}
	\kappa\left( \sigma_j \left( \ten{A}_{\ip{i}} \right)  \right)
	= 
	h \left( \sigma_1 \left( \ten{A}_{\ip{i}} \right)  \right) 
	\cdot h_i \left(  \varPsi_{k_i}( \ten{A}_{\ip{i}} )  \right) .
	\label{eq:app8}
	\end{align}
	Again,
	using non-decreasing property of $h$,
	we have
	\begin{align}
	h_i \left(  \varPhi_{k_i}( \ten{A}_{\ip{i}} )  \right)
	h \left( \sigma_{k_i + 1} \left( \ten{A}_{\ip{i}} \right)  \right)
	& =
	h \left( \sigma_{k_i + 1} \left( \ten{A}_{\ip{i}} \right)  \right) 
	\sum\nolimits_{j = 1}^{k_i} \kappa\left( \sigma_j \left( \ten{A}_{\ip{i}} \right)  \right),
	\notag
	\\
	&
	\le 
	\sum\nolimits_{j = 1}^{k_i} \kappa\left( \sigma_j \left( \ten{A}_{\ip{i}} \right)  \right)  
	h \left( \sigma_{j} \left( \ten{A}_{\ip{i}} \right)  \right)
	= \NM{\varPhi_{k_i}( \ten{A}_{\ip{i}} )}{*}.
	\label{eq:app9}
	\end{align}
	Note that $h(\alpha) \ge 1/\kappa_0$ from Lemma~\ref{lem:app1},
	and combining \eqref{eq:app8} and \eqref{eq:app9},
	we have
	\begin{align*}
	0 \le \varepsilon_i \cdot h_i(\varPhi_{k_i}( \ten{A}_{\ip{i}} ))
	- h_i( \varPsi_{k_i}( \ten{A}_{\ip{i}} ) ) 
	& \le
	\big( 
	\varepsilon_i \NM{\varPhi_{k_i}( \ten{A}_{\ip{i}} )}{*}
	- \NM{\varPsi_{k_i}( \ten{A}_{\ip{i}} )}{*}
	\big) 
	/ h \left( \sigma_1 ( \ten{A}_{\ip{i}} )  \right)
	\\
	& \le
	\kappa_0
	\big( 
	\varepsilon_i \NM{\varPhi_{k_i}( \ten{A}_{\ip{i}} )}{*}
	- \NM{\varPsi_{k_i}( \ten{A}_{\ip{i}} )}{*}
	\big).
	\end{align*}
	Thus, \eqref{eq:vareps1} is obtained.
	\textbf{Next, we prove \eqref{eq:vareps2}}.
	The triangle inequality and subadditivity of $h_i$ 
	(see Lemma~5 in~\citep{loh2015regularized}) imply that
	\begin{align*}
	0 
	\le \varepsilon_i \cdot h_i( \ten{X}^*_{\ip{i}} )
	- h_i( \ten{X}_{\ip{i}} )
	& = \varepsilon_i \cdot h_i( \varPhi_{m_i} (\ten{X}^*_{\ip{i}}) )
	- h_i( \varPhi_{m_i}(\ten{X}_{\ip{i}}) )
	- h_i( \varPsi_{m_i}(\ten{X}_{\ip{i}}) ),
	\\
	& \le
	\varepsilon_i \cdot 
	h_i\left( \varPhi_{m_i} \left( \ten{V}_{\ip{i}} \right) \right) - h_i\left( \varPsi_{m_i}\left( \ten{V}_{\ip{i}} \right) \right),
	\\
	& \le
	\kappa_0
	\big( 
	\varepsilon_i \NM{\varPhi_{k_i}( \ten{V}_{\ip{i}} )}{*}
	- \NM{\varPsi_{k_i}( \ten{V}_{\ip{i}} )}{*}
	\big). 
	\end{align*}
	Thus, \eqref{eq:vareps2} is obtained.
\end{proof}

\begin{lemma}\label{lem:app2}
	$\NM{\partial \phi ( \vect{X} )}{\infty} \le \kappa_0$
	where $\phi$ is defined in \eqref{eq:lrphi}.
\end{lemma}

\begin{proof}
	Let $\vect{X}$ be of size $m \times n$ with $m \le n$,
	and SVD of $\vect{X}$ be $\bm{U} \bm{\Sigma} \bm{V}^{\top}$
	where $\bm{\Sigma} = \Diag{\sigma_1, \dots, \sigma_m}$.
	From Theorem~3.7 in~\citep{lewis2005nonsmooth},
	we have
	\begin{align*}
	\partial \phi ( \vect{X} )
	= \bm{U} \Diag{\kappa'(\sigma_1), ..., \kappa'(\sigma_m)} \bm{V}^{\top}.
	\end{align*}
	From Lemma~\ref{lem:app1},
	we have 
	$	\kappa'(\sigma_1) 
	\le \kappa'(\sigma_2)
	\le ... 
	\le \kappa_0$.
	Since $\NM{\vect{X}}{\infty}$ returns
	the maximum singular value of $\vect{X}$,
	we have $\NM{\partial \phi ( \vect{X} )}{\infty} \le \kappa'(\sigma_m) \le \kappa_0$.
\end{proof}

\begin{lemma}\label{lem:app3}
	$\phi(\vect{X}) + \frac{\kappa_0}{2} \NM{ \vect{X} }{F}^2$ is convex.
\end{lemma}

\begin{proof}
	Using the definition of $\phi$ in \eqref{eq:lrphi}
	and the fact $\NM{\vect{X}}{F}^2 = \sum_i \sigma_i(\vect{X})$,
	we have 
	\begin{align*}
	\gamma(\vect{X}) =
	\phi(\vect{X}) + \kappa_0 / 2 \NM{ \vect{X} }{F}^2
	= \sum\nolimits_{i} \psi ( \sigma_i(\vect{X}) ),
	\end{align*}
	where $\psi(\alpha) = \kappa(\alpha) + \kappa_0 \alpha^2 / 2$.
	Since $\psi(\alpha)$ is convex (Lemma~\ref{lem:app1}),
	$\gamma(\vect{X})$ is convex (using Proposition~6.1 in~\citep{lewis2005nonsmooth}).
\end{proof}

\subsubsection{Proof of Theorem~\ref{thm:stat}}
\label{app:thm:stat}

\begin{proof}
\noindent
\textbf{Part 1).}
Let $\tilde{\ten{V}} = \tilde{\ten{X}} - \ten{X}^*$,
we begin by proving $\| \tilde{\ten{V}} \|_F \le 1$.
If not, then the second condition in \eqref{eq:stat1} holds,
i.e.,
\begin{align}
\left\langle
\nabla f ( \tilde{\ten{X}} ) 
-
\nabla f ( \ten{X}^* ),
\tilde{\ten{V}}
\right\rangle 
\ge
\alpha_2
\| \tilde{\ten{V}} \|_F^2 
-
\tau_2 \sqrt{ \log I^{\pi} / \| \vect{\Omega} \|_1 } \sum\nolimits_{i = 1}^M \NM{\Delta_{\ip{i}}}{*}.
\label{eq:stat3}
\end{align}
Since $\tilde{\ten{X}}$ is a first-order critical point,
then
\begin{align}
\left\langle
\nabla f ( \tilde{\ten{X}} ) 
+  \partial r( \tilde{\ten{X}} ),
\ten{X} - \tilde{\ten{X}}
\right\rangle 
\ge 0,
\label{eq:stat2}
\\
\nabla f ( \tilde{\ten{X}} ) 
+  \partial r( \tilde{\ten{X}} ) \ni \bm{0}.
\label{eq:stat13}
\end{align}
Taking $\ten{X} = \ten{X}^*$, from \eqref{eq:stat2}, we have
\begin{align}
\left\langle
\nabla f ( \tilde{\ten{X}} ) 
+  \partial r(\tilde{\ten{X}}),
- \tilde{\ten{V}}
\right\rangle 
\ge 0.
\label{eq:stat4}
\end{align}
Combining \eqref{eq:stat3} and \eqref{eq:stat4},
we have
\begin{align}
\left\langle
- \partial r(\tilde{\ten{X}})
- \nabla f \left( \ten{X}^* \right),
\tilde{\ten{V}}
\right\rangle 
\ge
\alpha_2
\| \tilde{\ten{V}} \|_F^2 
-
\tau_2 \sqrt{ \log I^{\pi} / \| \vect{\Omega} \|_1 } \sum\nolimits_{i = 1}^M \NM{\Delta_{\ip{i}}}{*}.
\label{eq:stat5}
\end{align}
Let $\tilde{v}_i = \| \tilde{\ten{V}}_{\ip{i}} \|_*$ and
$\tilde{v} = \sum\nolimits_{i = 1}^M \tilde{v}_i$.
For the L.H.S of \eqref{eq:stat5},
\begin{align}
\left\langle
\partial r(\tilde{\ten{X}}) + \nabla f \left( \ten{X}^* \right),
\tilde{\ten{V}}
\right\rangle 
&
=
\left\langle
\nabla f \left( \ten{X}^* \right),
\tilde{\ten{V}}
\right\rangle 
+
\lambda \sum\nolimits_{i = 1}^M 
\left\langle
\partial \phi( \ten{X}_{\ip{i}} ),
\tilde{\ten{V}}_{\ip{i}}
\right\rangle 
\\
& 
\le \max_i
\left\| [\nabla f \left( \ten{X}^* \right)]_{\ip{i}} \right\|_{\infty} \tilde{v}_i
+
\lambda \sum\nolimits_{i = 1}^M 
\left\| \partial \phi( \ten{X}_{\ip{i}} ) \right\|_{\infty} \tilde{v}_i,
\label{eq:temp2}
\end{align}
Next, note that
the following inequalities hold. 
\begin{itemize}[leftmargin=*]
\item From the left part of \eqref{eq:lambda} in Theorem~\ref{thm:stat},
we have
$\max_i \left\| [\nabla f \left( \ten{X}^* \right)]_{\ip{i}} \right\|_{\infty} \le \frac{\lambda \kappa_0}{4}$.

\item From Lemma~\ref{lem:app2},
we have $\NM{\partial \phi ( \vect{X} )}{\infty} \le \kappa_0$.
\end{itemize}
Combining with \eqref{eq:temp2},
we have
\begin{align}
\left\langle
\partial r(\tilde{\ten{X}}) + \nabla f \left( \ten{X}^* \right),
\tilde{\ten{V}}
\right\rangle 
\le 
\frac{\lambda \kappa_0}{4} 
+ 3 \lambda \kappa_0
= \frac{13 \lambda \kappa_0}{4}.
\label{eq:stat6}
\end{align}
Combining \eqref{eq:stat5} and
\eqref{eq:stat6},
then rearranging terms,
we have
\begin{align*}
	\| \tilde{\ten{V}} \|_F
	\le 
	\frac{1}{\alpha_2}
	\left( 
	\tau_2 \sqrt{ \log I^{\pi} / \| \vect{\Omega} \|_1 }
	+ \lambda \kappa_0  
	\right) 
	\tilde{v}
	\le 
	\frac{1}{\alpha_2}
	\left( 
	\tau_2 \sqrt{ \log I^{\pi} / \| \vect{\Omega} \|_1 }
	+ \frac{13 \lambda \kappa_0}{4}
	\right) 
	R.
\end{align*}
Finally,
using assumptions on $\| \vect{\Omega} \|_1$ and $\lambda$,
we have $\| \tilde{\ten{V}} \|_F \le 1$,
which is in the contradiction with our assumption at the beginning of Part 1).
Thus,
$\| \tilde{\ten{V}} \|_F \le 1$ must hold.

\vspace{5px}
\noindent
\textbf{Part 2).}
Let $h_i(\ten{X}) \! = \! \phi( \ten{X}_{\ip{i}} )$.
Since the function $h_i(\ten{X}) \! + \! \mu / 2 \NM{\ten{X}}{F}^2$ is convex (Lemma~\ref{lem:app3}),
we have
\begin{align}
	\left\langle 
	\partial h_i(\tilde{\ten{X}}),
	\ten{X}^* - \tilde{\ten{X}}
	\right\rangle 
	\le
	\tilde{h}_i(\tilde{\ten{X}}).
	\label{eq:stat7}
\end{align} 
where
$
\tilde{h}_i(\tilde{\ten{X}}) = 
h_i(\ten{X}^*)
- h_i(\tilde{\ten{X}}) 
+ \frac{L}{2} \| \tilde{\ten{X}} - \ten{X}^* \|_F^2
$.
From the first condition in \eqref{eq:stat1},
we have
\begin{align}
	\langle
	\nabla f ( \tilde{\ten{V}} ) 
	- 
	\nabla f \left( \ten{X}^* \right),
	- \tilde{\ten{X}}
	\rangle   
	\ge
	\alpha_1 
	\| \tilde{\ten{V}} \|_F^2
	- 
	\tau_1 \frac{\log I^{\pi}}{ \| \vect{\Omega} \|_1 }  \tilde{v}^2.
	\label{eq:stat8}
\end{align}
Combining \eqref{eq:stat2} and \eqref{eq:stat8},
we have
\begin{align*}
\alpha_1
\| \tilde{\ten{V}} \|_F^2
 - 
\tau_1 \frac{\log I^{\pi}}{ \| \vect{\Omega} \|_1 } 
\tilde{v}^2
& \le 
\left\langle \partial r(\tilde{\ten{X}}), \tilde{\ten{V}} \right\rangle 
 -  \left\langle
\nabla f(\ten{X}^*),
\tilde{\ten{V}}
\right\rangle,
\\
& =  
\lambda 
\sum\nolimits_{i = 1}^M  
\left\langle  \partial h_i (\tilde{\ten{X}}), \tilde{\ten{V}} \right\rangle 
 -  
\left\langle
\nabla f(\ten{X}^*),
\tilde{\ten{V}}
\right\rangle.
\end{align*}
Together with \eqref{eq:stat7},
we have
\begin{align*}
\alpha_1
\| \tilde{\ten{V}} \|_F^2
- \tau_1 \frac{\log I^{\pi}}{ \| \vect{\Omega} \|_1 } 
\tilde{v}^2
& \le
\lambda \sum\nolimits_{i = 1}^M 
\tilde{h}_i(\tilde{\ten{X}})
- \left\langle
\nabla f(\ten{X}^*),
\tilde{\ten{V}}
\right\rangle,
\\
& \le
\lambda \sum\nolimits_{i = 1}^M 
\tilde{h}_i(\tilde{\ten{X}}) 
+
\max_{i}
\NM{
	[\nabla f(\ten{X}^*)]_{\ip{i}} 
}{\infty}
\tilde{v}_i
\\
& \le
\lambda \sum\nolimits_{i = 1}^M 
\tilde{h}_i(\tilde{\ten{X}}) 
+
\max_{i}
\NM{
	[\nabla f(\ten{X}^*)]_{\ip{i}} 
}{\infty}
\tilde{v},
\end{align*}
where the second inequality is from Lemma~\ref{lem:app4}.
Rearranging items in the above inequality,
we have
\begin{align}
\big(
\alpha_1 - \frac{\mu M}{2}
\big) 
\| \tilde{\ten{V}} \|_F^2
\! \le \! \lambda \sum\nolimits_{i = 1}^M 
\big( 
h_i(\ten{X}^*)
\! - \! h_i(\tilde{\ten{X}}) 
\big) 
\! + \!
\left( 
\max_i \NM{ [\nabla f(\ten{X}^*)]_{\ip{i}} }{\infty}
\! + \! \tau_1 \frac{\log I^{\pi}}{ \| \vect{\Omega} \|_1 } \tilde{v}
\right) 
\tilde{v}.
\label{eq:temp1}
\end{align}
Note that from the Assumption in Theorem~\ref{thm:stat},
we have the following inequalities.
\begin{itemize}[leftmargin=*]
\item 
$\max_i 
\big\| \left[ \nabla f(\ten{X}^*) \right]_{\ip{i}} \big\|_{\infty}
\le
\kappa_0 \lambda / 4$.

\item Since $\| \vect{\Omega} \|_1 \ge 16 R^2 \max\left( \tau_1^2, \tau_2^2 \right) \log (I^{\pi} ) / \alpha_2^2$
and 
$\alpha_2 \sqrt{ \log I^{\pi} / \| \vect{\Omega} \|_1 } \le \kappa_0 \lambda / 4$, then
\begin{align*}
	\frac{\tau_1 \log I^{\pi}}{ \| \vect{\Omega} \|_1 } \tilde{v}
	\le 
	\frac{\tau_1}{ \alpha_2 } 
	\sqrt{\frac{\log I^{\pi}}{\| \vect{\Omega} \|_1}}\tilde{v}
	\cdot
	\alpha_2
	\sqrt{\frac{\log I^{\pi}}{\| \vect{\Omega} \|_1}}
	\le 
	\frac{\tau_1}{ \alpha_2 } 
	\sqrt{\frac{\alpha_2^2 \log I^{\pi}}{16 \tilde{v}^2 \tau_1^2}}\tilde{v}
	\cdot
	\alpha_2
	\sqrt{\frac{\log I^{\pi}}{\| \vect{\Omega} \|_1}}
	\le \frac{\lambda \kappa_0}{4}.
\end{align*}
\end{itemize}
Combing above inequalities into \eqref{eq:temp1},
we further have
\begin{align}
\big(
\alpha_1 - \frac{\mu M}{2}
\big) 
\| \tilde{\ten{V}} \|_F^2
&  \le  
\lambda  \sum\nolimits_{i = 1}^M 
\big( 
h_i(\ten{X}^*)
-  h_i(\tilde{\ten{X}}) 
\big) + 
\frac{\lambda \kappa_0}{2} \tilde{v}.
\label{eq:stat9}
\end{align}


\vspace{5px}
\noindent
\textbf{Part 3).}
Combining \eqref{eq:stat9} and Lemma~\ref{lem:app1},
as well as the subadditivity of $h_i$, 
we have
\begin{align}
\big(
\alpha_1 - \frac{L M}{2}
\big) 
\| \tilde{\ten{V}} \|_F^2
& \le 
\lambda \sum\nolimits_{i = 1}^M
\big( 
h_i(\ten{X}^*)
- h_i(\tilde{\ten{X}}) 
\big)
+ 
\frac{\lambda \kappa_0}{2}
\big(  
\frac{\sum\nolimits_{i = 1}^M h_i(\tilde{\ten{V}})}{L M}
+ \frac{ L M  }{2 \lambda \kappa_0} \| \tilde{\ten{V}} \|_F^2 
\big),
\notag
\\
&
\le \! \lambda \! \sum\nolimits_{i = 1}^M \!
\big( 
h_i(\ten{X}^*) \! - \! h_i(\tilde{\ten{X}}) 
\big) 
+ \frac{\lambda\sum\nolimits_{i = 1}^M h_i(\ten{X}^*) \! + \! h_i(\tilde{\ten{X}}) }{2D}
+ \frac{L M}{4} \| \tilde{\ten{V}} \|_F^2.
\label{eq:app10}
\end{align}
Next,
define 
\begin{align*}
	a_v
	= \alpha_1 - \frac{3 M}{4} \kappa_0,
	\quad
	b_v
	= 1 + \frac{1}{2 M},
	\quad
	c_v
	= 1 - \frac{1}{2 M}.
\end{align*}
Rearranging terms in \eqref{eq:app10},
we have
\begin{align}
	a_v \| \tilde{\ten{V}} \|_F^2
	\le \lambda
	\sum\nolimits_{i = 1}^M
	b_v h_i(\ten{X}^*) 
	- c_v h_i(\tilde{\ten{X}}).
	\label{eq:stat10}
\end{align}
From Lemma~\ref{lem:app5},
we have
\begin{align}
	b_v  h_i(\ten{X}^*) 
	- c_v  h_i(\tilde{\ten{X}})
	\le 
	L
	\big(
	b_v 
	\NM{ \varPhi_{k_i} (\tilde{\ten{V}}_{\ip{i}}) }{*}
	- c_v \NM{ \varPsi_{k_i}( \tilde{\ten{V}}_{\ip{i}} ) }{*}
	\big) .
	\label{eq:stat11}
\end{align}
Besides,
we have the cone condition
\begin{align}
	\NM{\varPhi_{k_i} (\tilde{\ten{V}}_{\ip{i}})}{*} 
	\le \frac{c_v}{b_v} \NM{\varPsi_{k_i}( \tilde{\ten{V}}_{\ip{i}} )}{*}.
	\label{eq:stat12}
\end{align}
Combining \eqref{eq:stat10},
\eqref{eq:stat11}
and \eqref{eq:stat12},
we have
\begin{align*}
	a_v
	\| \tilde{\ten{V}} \|_F^2
	& \le 
	\lambda \kappa_0
	\sum\nolimits_{i = 1}^M
	\big(
	b_v 
	\NM{\varPhi_{k_i} (\tilde{\ten{V}}_{\ip{i}})}{*}
	- c_v \NM{\varPsi_{k_i}( \tilde{\ten{V}}_{\ip{i}} )}{*}
	\big),
	\\ 
	& \le 
	\lambda \kappa_0
	\sum\nolimits_{i = 1}^M b_v \NM{\varPhi_{k_i} (\tilde{\ten{V}}_{\ip{i}})}{*}
	\le 
	\lambda \kappa_0
	\sum\nolimits_{i = 1}^M c_v \sqrt{k_i} \| \tilde{\ten{V}} \|_F.
\end{align*}
where the last inequality comes from Lemma~\ref{lem:app6}.
Since $a_v > 0$ as assumed,
we conclude that
\begin{align*}
	\| \tilde{\ten{V}} \|_F
	\le 
	\frac{ \lambda \kappa_0 c_v }{a_v}
	\sum\nolimits_{i = 1}^M  \sqrt{k_i}.
\end{align*}
which proves the theorem.
\end{proof}


\subsection{Corollary~\ref{cor:noisy}}

\begin{proof}
When noisy level is sufficiently small,
\eqref{eq:lambda} reduces to
\begin{align} 
\sqrt{ \log I^{\pi} / \NM{\vect{\Omega}}{1} }
\le
\lambda
\le
\frac{1}{ 4 R \kappa_0}.
\label{eq:app18}
\end{align}
Let $\lambda = b_1 \max_i \| \left[ \SO{\ten{E}} \right]_{\ip{i}} \|_{\infty}$
where
$b_1 \in 
\left[ 
\frac{ 4 }{ \kappa_0 }, 
\frac{ \alpha_2 }{ 4 R \kappa_0 \max_i \|  \left[ \SO{\ten{E}} \right]_{\ip{i}} \|_{\infty} } 
\right]$.
It is easy to check \eqref{eq:app18} holds.
Then,
from Theorem~\ref{thm:stat},
we will have
\begin{align}
\big\| \ten{X}^* - \tilde{\ten{X}} \big\|_F
\le 
b_1 \max_i \| \left[ \SO{\ten{E}} \right]_{\ip{i}} \|_{\infty}
\cdot
\frac{ \kappa_0 c_v }{a_v} \sum\nolimits_{i = 1}^M  \sqrt{k_i}.
\label{eq:app20}
\end{align}
Next, 
note that 
\begin{align}
\mathbb{E}
\left[ 
\| \left[ \SO{\ten{E}} \right]_{\ip{i}} \|_{\infty} 
\right] 
\le 
\mathbb{E}
\left[ 
\| \left[ \SO{\ten{E}} \right]_{\ip{i}} \|_{F}
\right] 
= \mathbb{E}
\left[ 
\| \xi \cdot \vect{\Omega} \|_{F}
\right] 
= \sigma \| \vect{\Omega} \|_{F}
\le 
\sigma \sqrt{I^{\pi}}.
\label{eq:app19}
\end{align}
Combining \eqref{eq:app20} and \eqref{eq:app19},
we then have
\begin{align*}
\mathbb{E}
\left[ 
\big\| \ten{X}^* - \tilde{\ten{X}} \big\|_F
\right] 
\le 
\sigma
\frac{ \kappa_0 c_v \sqrt{I^{\pi}}}{a_v} 
\sum\nolimits_{i = 1}^M  \sqrt{k_i}.
\end{align*}
\end{proof}


\subsection{Corollary~\ref{cor:number}}

\begin{proof}
	When $\NM{\vect{\Omega}}{1}$ is sufficiently larger,
	\eqref{eq:lambda} reduces to
	\begin{align} 
		4
\sqrt{ \log I^{\pi} / \NM{\vect{\Omega}}{1} }
		\le
		\lambda
		\le
		\frac{1}{ 4 R }.
		\label{eq:app21}
	\end{align}
	Let $\lambda = b_3 \sqrt{ \frac{ \log I^{\pi} }{ \NM{\vect{\Omega}}{1} } }$
	where
	$b_3 \in 
	\left[ 
	4, 
	\frac{ 1 }{ 4 R  \sqrt{ \log I^{\pi} / \NM{\vect{\Omega}}{1} } } 
	\right]$.
	It is easy to check \eqref{eq:app21} holds.
	Then,
	from Theorem~\ref{thm:stat},
	we will have
	\begin{align*}
		\big\| \ten{X}^* - \tilde{\ten{X}} \big\|_F
		\le 
		b_3 \sqrt{ \frac{ \log I^{\pi} }{ \NM{\vect{\Omega}}{1} } }
		\cdot
		\frac{ \kappa_0 c_v }{a_v} \sum\nolimits_{i = 1}^M  \sqrt{k_i}.
	\end{align*}
\end{proof}


\subsection{Proposition~\ref{pr:smtl1}}

\begin{proof}
First, 
from Proposition 1 in~\citep{yao2018efficient},
we know that the function $\kappa(|a|) - \kappa_0 \cdot |a|$ is smooth.
Since $\tilde{\ell}$ is also smooth,
thus $\tilde{\kappa}_{\ell}$ is differentiable.
Finally,
note that $\lim\nolimits_{\delta \rightarrow 0} \ell(a; \delta) = |a|$.
Then,
we have
\begin{align*}
	\lim\nolimits_{\delta \rightarrow 0} 
	\tilde{\kappa}_{\ell}(|a|; \delta)
	& =  
	\lim\nolimits_{\delta \rightarrow 0}
	\big[ 
	\kappa_0 \cdot \tilde{\ell}(|a|; \delta) + 
	\left(
	\kappa_{\ell}(|a|) - \kappa_0 \cdot |a|
	\right)
	\big],
	\\
	& = 
	\kappa_0  |a| + 
	\left(
	\kappa_{\ell}(|a|) - \kappa_0  |a|
	\right) 
	= \kappa_{\ell}(|a|).
\end{align*}
Thus,
the Proposition holds.
\end{proof}


\subsection{Theorem~\ref{thm:snort}}

\begin{proof}
	First,
	by the definition of $\tilde{\kappa}_{\ell}$ in \eqref{eq:smtl1},
	when $|a| \le \delta$,
	we have
	\begin{align*}
		\lim\limits_{\delta \rightarrow 0} \partial \tilde{\kappa}_{\ell}(a; \delta)
		= \frac{a}{\delta} \kappa_0
		\in
		\begin{cases}
			[0, \kappa_0) & \text{if}\; a \ge 0
			\\
			(-\kappa_0, 0) & \text{otherwise}
		\end{cases}.
	\end{align*}
	Thus,
	\begin{align}
		\lim\nolimits_{\delta \rightarrow 0} \partial \tilde{\kappa}_{\ell}(a; \delta)  = \partial \kappa_{\ell}(|a|).
		\label{eq:app1}
	\end{align}
	Define
	$\tilde{F}_{\tau}(\ten{X}; \delta) =
	\sum\nolimits_{(i_1 ... i_M) \in \mathbf{\Omega}}
	\tilde{\ell}\left( \ten{X}_{i_1 ... i_M} - \ten{O}_{i_1... i_M} ; \delta \right) 
	+ \sum\nolimits_{i = 1}^D \lambda_i \phi( \ten{X}_{\ip{i}} )$.
	Since $\ten{X}_s$ is obtained from solving \eqref{eq:smoothpro} at step~4 of Algorithm~\ref{alg:snort},
	we have
	$\ten{X}_s \in \partial \tilde{F}_{\tau}(\ten{X}; (\delta_0)^s)$.
	Take $s \rightarrow \infty$ and use \eqref{eq:app1},
	we have
	$		\lim\nolimits_{s \rightarrow \infty} \ten{X}_s 
	\in \lim\nolimits_{s \rightarrow \infty} \partial \tilde{F}_{\tau}(\ten{X}; (\delta_0)^s)
	= \lim\nolimits_{\delta \rightarrow 0} \partial \tilde{F}_{\tau}(\ten{X}; \delta)
	= \partial \tilde{F}_{\tau}(\ten{X})$.
	Thus,
	Theorem~\ref{thm:snort} holds.
\end{proof}

\bibliographystyle{apalike}
\bibliography{bib}

\clearpage

\end{document}